\documentclass[11pt, a4paper, oneside, reqno]{amsart}

\usepackage[dvipsnames]{xcolor}
\definecolor{darkblue}{rgb}{0.0, 0.0, 0.45}
\definecolor{darkblue2}{RGB}{0  40 159}
\definecolor{lightblue}{RGB}{240,248,255}
\definecolor{lightblue2}{rgb}{0.68, 0.85, 0.9}
\definecolor{lightcyan}{rgb}{0.88, 1.0, 1.0}
\definecolor{palepink}{rgb}{0.98, 0.85, 0.87}

\usepackage[colorlinks	= true,
raiselinks	= true,
linkcolor	= darkblue,
citecolor	= Mahogany,
urlcolor	= ForestGreen,
pdfauthor	= {Sasan Vakili},
pdftitle	= {},
pdfkeywords	= {},
pdfsubject	= {},
plainpages	= false]{hyperref}

\usepackage{amsfonts,dsfont,mathtools, mathrsfs,amsthm, amssymb,amsmath,mathdots} 
\usepackage{dsfont}
\usepackage{mathabx}
\usepackage{bbm}
\usepackage{accents}
\usepackage{adjustbox}
\usepackage{siunitx}

\usepackage{setspace}

\usepackage{fancyhdr,mdframed}

\usepackage[T1]{fontenc}

\usepackage{algcompatible}

\usepackage{algorithm}
\usepackage{algorithmicx,multirow, titlecaps}

\usepackage{multirow}
\usepackage{bigstrut}
\usepackage{wrapfig}
\usepackage[font=small, hypcap=false]{caption} 
\usepackage{nicefrac}

\usepackage{epsfig}
\usepackage{graphicx}
\usepackage{float}
\usepackage{placeins} 

\usepackage{stackengine}
\usepackage{booktabs}

\usepackage{subcaption}
\makeatletter

\makeatletter

\usepackage{algpseudocode}

\usepackage{enumitem}
\setlist{noitemsep, topsep=0cm}

\usepackage{tikz, pgfplots}
\pgfplotsset{compat=1.18}
\usetikzlibrary{mindmap}
\usepgfplotslibrary{groupplots}
\newlength\figureheight
\newlength\figurewidth

\usepackage[utf8]{inputenc}
\usepackage{lscape}

\usepackage{colortbl}
\usepackage{svg}

\usepackage{tcolorbox}
\tcbuselibrary{theorems, skins, breakable}

\usepackage{lipsum}

\allowdisplaybreaks
\date{\today}



\makeatletter
\def\@settitle{\begin{center}%
		\baselineskip14\p@\relax
		\normalfont\LARGE\scshape\bfseries
		\@title
	\end{center}%
}

\def\@setauthors{%
  \begingroup
  \def\thanks{\protect\thanks@warning}%
  \trivlist
  \centering\small \@topsep30\p@\relax
  \advance\@topsep by -\baselineskip
  \item\relax
  \author@andify\authors
  \def\\{\protect\linebreak}%
  \authors%
  \ifx\@empty\contribs
  \else
    ,\penalty-3 \space \@setcontribs
    \@closetoccontribs
  \fi
  \endtrivlist
  \endgroup
}

\makeatother
\makeatletter

\def\subsection{\@startsection{subsection}{2}%
	\z@{.5\linespacing\@plus.7\linespacing}{.5\linespacing}%
	{\normalfont\large\bfseries}}

\def\subsubsection{\@startsection{subsubsection}{3}%
	\z@{.5\linespacing\@plus.7\linespacing}{.5\linespacing}%
	{\normalfont\itshape}}

\algnewcommand\algorithmicvariant{\textbf{Variant:}}
\algnewcommand\Variant{\item[\algorithmicvariant]}

\newtheorem{theorem}{Theorem}[section]
\newtheorem*{theorem*}{Theorem}
\newtheorem{definition}{Definition}
\newtheorem*{definition*}{Definition}
\newtheorem{assumption}{Assumption}

\newtheorem{lemma}[theorem]{Lemma}
\newtheorem{remark}[theorem]{Remark}

\newtheorem*{example*}{Example}

\newtheorem*{problem*}{Problem}

\newenvironment{examplepar}[1][]
{%
  \refstepcounter{example}%
  \paragraph{\textbf{Example~\theexample%
    \if\relax\detokenize{#1}\relax\else\ #1\fi}}%
}%
{}

\newcommand{\cmdb}[1]{{\color{darkblue2}{#1}}}

\newcommand*{\sumOp}{\operatornamewithlimits{\sum}\limits}

\newcommand{\binarySum}[2]{\!\!\!\!\!\!\!\!\!\!\!{\sumOp_{\;\;\;\;\;\;\; {#1}} \!\!\!\!\!\!\!\!{#2}}}
\newcommand*{\prodOp}{\operatornamewithlimits{\prod}\limits}

\newcommand*{\minOp}{\operatornamewithlimits{min}\limits}

\newcommand{\drm}{\mathrm{d}}           

\newcommand{\nth}{{\text{\tiny{th}}}}
\newcommand{\trace}{\mathrm{tr}} 
\newcommand{\tr}{\!^{\scalebox{0.65}{$\mathsf{T}$}}}
\newcommand{\inv}{\!\!^{\scalebox{0.65}{$-1$}}}
\newcommand{\gradient}{\nabla}

\newcommand{\Expectation}[1]{\mathbb{E}[#1]}
\newcommand{\bigExpectation}[1]{\mathbb{E}\big[#1\big]}

\newcommand{\BigExpectation}[1]{\mathbb{E}\Big[#1\Big]}
\newcommand{\CondExpectation}[2]{\mathbb{E}_{#1}\Big[#2\Big]}

\newcommand{\diag}{\mathrm{diag}}

\newcommand{\Real}{{\mathbb{R}}}
\newcommand{\Complex}{{\mathbb{C}}}

\newcommand{\NormalDist}{{\mathcal{N}}}
\newcommand{\UniformDist}{{\mathcal{U}}}
\newcommand{\eye}{\mathbb{I}}

\newcommand{\expS}[1]{\mathrm{exp}{(#1)}}
\newcommand{\norm}[1]{\left\lVert#1\right\rVert}

\newcommand{\innerS}[2]{{\langle {#1,#2} \rangle}}

\newcommand{\Pdist}{{\mathbb{P}}}
\newcommand{\zero}{0}
\newcommand{\zeromx}{0}

\newcommand{\matrixSqrt}[1]{{#1^{^{\frac{1}{2}}}}}

\newcommand{\fracComma}{\raisebox{0.5ex}{,}}
\newcommand{\fracDot}{\raisebox{0.45ex}{.}}
\newcommand{\vectorize}{\mathrm{vec}}

\renewcommand{\hat}[1]{\widehat{#1}}

\newcommand{\trueFunc}{{h}}
\newcommand{\estimatorFunc}{\mathscr{F}_{\!\vecTheta}}
\newcommand{\estimatorAffine}[1]{\mathscr{A}_{#1}}

\newcommand{\operatorGaussianDBS}{\matrixPhi_{\!\scalebox{0.6}{$\#$}}}
\newcommand{\operatorDBS}{\mathcal{M}}
\newcommand{\operatorParam}{\mathcal{S}}

\newcommand{\vecTheta}{{\mathrm{\theta}}}
\newcommand{\timedTheta}[1]{\vecTheta^{#1}}

\newcommand{\Ntheta}{\mathrm{N}}
\newcommand{\Nparticle}{N_{P}}
\newcommand{\vecY}{{\mathrm{y}}}

\newcommand{\vecV}{{\mathrm{v}}}
\newcommand{\matrixPhi}{\mathrm{\Phi}}
\newcommand{\matrixPhibar}{\mu_{\matrixPhi}}
\newcommand{\matrixPhibarG}{\mu_{\matrixPhi \scriptscriptstyle( \nu \scriptscriptstyle)}}

\newcommand{\MuTheta}[1]{\mu_{\vecTheta_{#1}}}
\newcommand{\timedMuTheta}[1]{\mu_{\vecTheta}^{#1}}
\newcommand{\SigmaTheta}[1]{\Sigma_{\vecTheta_{#1}}}

\newcommand{\SigmaV}[1]{\Sigma_{\vecV_{#1}}}

\newcommand{\sigmaTheta}[1]{\sigma_{\vecTheta_{#1}}}

\newcommand{\Trajectory}{T}
\newcommand{\vecX}{{\mathrm{x}}}

\newcommand{\nX}{n_{\scalebox{0.65}{$\mathrm{x}$}}}
\newcommand{\nY}{n_{\scalebox{0.65}{$\mathrm{y}$}}}
\newcommand{\nZ}[1]{n_{#1}}
\newcommand{\SigmaX}[1]{\Sigma_{\vecX_{#1}}}

\newcommand{\MuX}[1]{\mu_{\vecX_{#1}}}
\newcommand{\vecA}{{\mathrm{a}}}
\newcommand{\A}{{\mathrm{A}}}
\newcommand{\B}{{\mathrm{B}}}
\newcommand{\vecU}{{\mathrm{u}}}
\newcommand{\nU}{n_{\scalebox{0.65}{$\mathrm{u}$}}}
\newcommand{\vecW}{{\mathrm{w}}}
\newcommand{\SigmaW}[1]{\Sigma_{\vecW_{#1}}}

\newcommand{\Abar}{\A}
\newcommand{\Bbar}{\B}
\newcommand{\vecXbar}{\vecX}
\newcommand{\vecUbar}{\vecU}
\newcommand{\vecWbar}{\vecW}
\newcommand{\SigmaWbar}[1]{\Sigma_{\vecWbar_{#1}}}
\newcommand{\vecYbar}{\vecY}
\newcommand{\vecVbar}{\vecV}

\newcommand{\SigmaVbar}[1]{\Sigma_{\vecVbar_{#1}}}

\newcommand{\sigmaV}[1]{\sigma_{\vecV_{#1}}}

\newcommand{\matrixS}{{\mathrm{S}}}

\newcommand{\deltaPhi}{{\Delta\phi}}

\newcommand{\DeltaTheta}{{\Delta\vecTheta}}

\newcommand{\Cost}{{\mathcal{J}}}

\newcommand{\matrixG}{{\mathrm{G}}}

\newcommand{\funcG}{{\mathrm{g}}}

\newcommand{\Jacobian}{{\mathrm{J}}}
\newcommand{\matrixC}{{\mathrm{C}}}

\newcommand{\hatSigmaXbar}[1]{\Sigma_{\hat{\vecXbar}_{#1}}}

\newcommand{\hatSigmaTheta}[1]{\Sigma_{\hat{\vecTheta}_{#1}}}
\newcommand{\hatSigmaPhi}[1]{\Sigma_{\hat{\matrixPhi}_{#1}}}

\newcommand{\matrixQ}{{\mathrm{Q}}}

\newcommand{\phiBar}{\Bar{\phi}}

\newcommand{\matrixP}{\mathrm{P}}

\newcommand{\SigmaPhi}{\Sigma_{\matrixPhi}}
\newcommand{\SigmaPhiG}{\Sigma_{\matrixPhi \scriptscriptstyle( \nu \scriptscriptstyle)}}
\newcommand{\TimedSigmaPhi}[1]{\Sigma_{\phi}^{#1}}
\newcommand{\MuPhi}{\mu_{\phi}}

\newcommand{\Cbar}{\matrixC}

\newcommand{\PsiPhi}{\Psi_{\!\phi}}
\newcommand{\psiPhi}{\psi_{\phi}}

\newcommand{\PsiX}{\Psi_{\!\vecXbar}}
\newcommand{\psiX}{\psi_{\vecXbar}}

\newcommand{\PsiTheta}{\Psi_{\!\vecTheta}}
\newcommand{\psiTheta}{\psi_{\vecTheta}}

\title[Nonlinear Bayesian Estimator for Parameter Learning]{\Large Nonlinear Bayesian Estimator for Parameter Learning: \\ A Fixed-Point Characterization}

\author[S. Vakili]{Sasan Vakili$^{1}$}
\author[Dani\"{e}l Woonings]{Dani\"{e}l Woonings$^{1}$}
\author[P. Paruchuri]{Pradyumna Paruchuri$^{1}$}
\author[P. {Mohajerin Esfahani}]{Peyman {Mohajerin Esfahani$^{1,2}$}
\\
\\
$^1$Delft University of Technology, The Netherlands\\
$^2$University of Toronto, Canada}

\thanks{This work was supported by the Marie Sk\l{}odowska-Curie Grant under Agreement 956200, the ERC Starting Grant TRUST-949796, and the NSERC Discovery grant RGPIN-2025-06544.}

\begin{document}
\begin{abstract}
This paper presents a nonlinear parameter estimator for Wiener-type state-space models obtained as a fixed-point architecture that couples two affine minimum mean-squared error (MMSE) estimators: one for the unknown parameters and one for latent variables. The architecture retains the functional structure of the optimal affine MMSE parameter estimator while incorporating Dynamic Basis Statistics (DBS) estimates that summarize nonlinear basis-function evaluations. Two DBS construction strategies are developed, leading to two nonlinear estimator frameworks. The dual basis-parameter estimator combines an affine basis estimator with the affine parameter estimator, whereas the dual state-parameter estimator first computes affine state estimates and their covariances, then maps these state-estimate statistics through a Gaussian DBS operator to obtain DBS estimates. Both dual estimators admit fixed-point characterizations that alternate between estimating each component using the updated prior of the other, obtained from that component's plug-in estimate statistics from the previous iteration. The efficacy of the proposed methods is examined via extensive Monte Carlo experiments, showing that the dual basis-parameter estimator attains parameter mean-squared errors comparable to those of the purely affine parameter estimator, while the dual state-parameter estimator achieves the lowest parameter mean-squared error, outperforming both the dual basis-parameter and purely affine parameter estimators, as well as sequential Monte Carlo variants of classical Particle Gibbs and Expectation-Maximization schemes.
\end{abstract}

\maketitle

\section{Introduction} \label{sec:Introduction}
The identification of unknown parameters in mathematical models is a foundational problem that has been studied extensively using a wide range of techniques, including subspace methods in system identification~\cite{Ljung1999SysId} and statistical inference~\cite{bishop_pattern_2006}. Among the two main paradigms of parameter inference, namely the frequentist and Bayesian approaches, the Bayesian approach has received increasing attention in recent years due to its ability to quantify parameter uncertainty and incorporate prior information within a coherent probabilistic framework~\cite{gelman2013bayesian}. Identifying a parametrized unknown output function becomes particularly challenging when the inputs are corrupted by correlated noise arising from underlying state dynamics, as in state-space models, because these additional sources of randomness further complicate the inference process~\cite{sarkka2013bayesian}. Consider a \emph{known} discrete-time linear time-varying dynamical system
\begin{equation}\label{eq:Dynamic_Observation_Models}
\vecX_{t+1} = \A_{t}\vecX_{t} + \B_{t}\vecU_t + \vecW_{t+1}, \qquad
\vecY_{t} = \trueFunc_{\vecTheta}(\vecX_{t}) + \vecV_{t},
\end{equation}
where $t$ is the time index, $\vecX_{t} \in \Real^{\nX}$ denotes the system state vector, $\A_{t} \in \Real^{\nX \times \nX}$ and $\B_{t} \in \Real^{\nX \times \nU}$ are known system matrices, and $\vecU_{t} \in \Real^{\nU}$ represents the sequence of known or controlled inputs. The process noise $\vecW_{t+1} \in \Real^{\nX}$ is modelled as a zero-mean random vector with known covariance $\SigmaW{t+1}$, that is, $\vecW_{t+1} \sim \Pdist(\zero, \SigmaW{t+1})$. The observation $\vecY_t \in \Real^{\nY}$ is generated via an \emph{unknown} parameterized function $\trueFunc_{\vecTheta} \colon \Real^{\nX} \!\! \to \! \Real^{\nY}$ corrupted by measurement noise $\vecV_t \in \Real^{\nY}$ and $\vecV_t \sim \Pdist(0, \SigmaV{t})$. Throughout this work, and without loss of generality, we focus on the single-output case $\nY \!=\! 1$, since the extension to multi-output measurements ($\vecY_t \in \Real^{\nY}$) is straightforward and essentially reduces to a set of independent single-output problems~\cite[Rem.~2.1]{vakili2025optimal}. The objective is to estimate the unknown parameter vector $\vecTheta$ characterizing the observation model $\trueFunc_{\vecTheta}(\cdot)$, given only the noisy observations ${\vecY_t}$ and known inputs ${\vecU_t}$. We further parameterize the output function $\trueFunc_{\vecTheta} \colon \Real^{\nX} \!\! \to \! \Real$ as a linear combination of known basis functions $\phi_{n} \colon \Real^{\nX} \!\! \to \! \Complex$, i.e.,
\begin{equation} \label{eq:Function_Space}
\trueFunc_{\vecTheta}(\vecX_{t}) = \sumOp_{n=0}^{\Ntheta} \vecTheta_{n} \phi_{n}(\vecX_{t}) = \innerS{\phi(\vecX_{t})}{\vecTheta},
\end{equation}
where the number of basis functions is $\Ntheta+1$, the coefficient vector $\vecTheta = [\vecTheta_{0}, \ldots, \vecTheta_{\Ntheta}]^{\tr}$ contains the unknown parameters, and $\phi(\vecX_{t}) = [ \phi_{0}(\vecX_{t}), \ldots, \phi_{\Ntheta}(\vecX_{t}) ]^{\tr}$ is the vector of basis functions evaluated at $\vecX_{t}$. Additionally, we incorporate prior information about the unknown parameters through a probability distribution characterized by mean $\MuTheta{}$ and covariance $\SigmaTheta{}$, i.e., $\vecTheta \sim \Pdist(\MuTheta{}, \SigmaTheta{})$. This formulation encompasses many nonlinear system identification problems.

Such systems are important in many application domains that are modelled as block-oriented structures consisting of a linear dynamical process followed by a static nonlinear observation model, commonly known as Wiener models~\cite{schoukens2017identification}.
In this setting, parameter estimation is inherently challenging because the system states, $\vecX_t$, are stochastic due to the process noise, $\vecW_{t+1}$, which propagates uncertainty through the linear dynamics. Given this stochastic nature of the system states appearing as inputs to the nonlinear function $\trueFunc_{\vecTheta}(\cdot)$, the probability density function of the observations $\vecY_t$ is, in general, not available in closed-form~\cite{doucet2000sequential, sarkka2013bayesian}. This intractability fundamentally complicates Bayesian parameter estimation and necessitates the use of approximate inference techniques~\cite{bishop_pattern_2006}. In particular, the optimal Bayesian minimum mean squared error (MMSE) estimator requires integrating over the joint distribution of the unobserved states and unknown parameters, which is computationally prohibitive. As a result, research in this area has mainly focused on fully Bayesian inference techniques~\cite{doucet2001introduction, tokdar2010importance} as well as more tractable alternatives such as maximum a posteriori (MAP) estimation~\cite{cappe2005inference}.

\paragraph{\bf Related literature.}
Fully Bayesian methods aim to characterize the entire posterior distribution, typically only approximately, using techniques such as Markov Chain Monte Carlo (MCMC)~\cite[Ch.~6]{robert_monte_2004}, sequential Monte Carlo (SMC) ~\cite{andrieu_particle_2010, chopin_introduction_2020, naesseth2019elements}, and variational inference~\cite{blei2017variational, courts2023variational}. MCMC algorithms construct a Markov chain that converges in distribution to the posterior~\cite{roberts_general_2004}. Classic examples include the Metropolis-Hastings (MH) algorithm~\cite{metropolis1953equation, hastings1970monte} and the Gibbs sampler~\cite{gelfand_sampling-based_1990}. However, in nonlinear, non-Gaussian state-space models, such conventional MCMC algorithms cannot be applied directly, since MH method requires evaluating the likelihood up to a normalizing constant, and Gibbs sampler relies on sampling from exact conditional distributions, both of which are generally infeasible in this setting.

To address these challenges, several approximate variants have been developed, such as MCMC methods that incorporate the unscented Kalman filter (UKF-MCMC)~\cite{ukf_smoothing_mcmc}, and combinations of MCMC and SMC techniques~\cite{pitt_properties_2012}, collectively referred to as particle MCMC (PMCMC) algorithms. By exploiting different theoretical properties of the SMC framework~\cite{chopin_introduction_2020}, the MH and Gibbs sampler algorithms can be extended to the particle marginal Metropolis-Hastings (PMMH) method~\cite{schon_system_2011} and the particle Gibbs (PG) algorithm~\cite{chopin_particle_2015}, respectively. Although these methods provide reliable uncertainty quantification and improved performance, they are typically computationally demanding \cite{chopin_complexity} and scale poorly to high-dimensional problems~\cite{bardenet2017markov, rajaratnam_mcmc-based_2015}. In this work, the particle Gibbs with ancestor sampling (PGAS) algorithm~\cite{lindsten2014particle}, a prominent particle sampler for Bayesian learning of state-space models, is included as one of the benchmark methods for numerical evaluation in Section~\ref{sec:Numerical_Experiments}.

In contrast to fully Bayesian approaches, MAP estimation seeks point estimates by identifying the mode of the posterior. The expectation maximization (EM) algorithm is widely used for maximum likelihood (ML) and MAP estimation in the presence of latent variable models. EM alternates between estimating the distribution over latent variables (E-step) and updating model parameters (M-step). In dynamical systems with unobserved latent states~\cite{dempster_maximum_1977, schon_system_2011}, the E-step entails a smoothing problem, where the latent states distribution given all observations is estimated, while the M-step maximizes the expected complete-data log-likelihood, typically incorporating prior information.

When both steps have closed-form solutions, as in linear Gaussian models~\cite{ghahramani1996parameter}, EM guarantees a monotonic increase in the likelihood~\cite[Ch.~13]{gelman2013bayesian}. However, for nonlinear or non-Gaussian models, these quantities are generally intractable, so approximate methods are used in the E-step, such as extended Kalman smoothing~\cite[Ch.~6]{haykin2004kalman}, stochastic methods~\cite{delyon_convergence_1999}, and variational approaches~\cite{lindley_variational_2003}. Among approximate E-step methods, particle filters and kernels such as PGAS sample the smoothing distribution to approximate the likelihood~\cite{lindsten_efficient_2013, schon_system_2011}. These approximations, while often effective, do not guarantee monotonic likelihood improvement and can lead to suboptimal modes or slow convergence~\cite{balakrishnan2017statistical}. Nevertheless, we use the EM algorithm~\cite{schon_system_2011} as a benchmark in our experiments of Section~\ref{sec:Numerical_Experiments}.

A recent perspective restricts attention to the class of affine estimators rather than approximating full posterior distributions~\cite{vakili2025optimal}. Specifically, this approach leverages the fact that the optimal Bayesian MMSE estimator is affine when the joint distribution of the observations and parameters is elliptical. The resulting closed-form optimal affine estimator for unknown parameters relies on computing the so-called Dynamic Basis Statistics (DBS), which capture the influence of stochastic system states on the estimation process. In other words, when $\trueFunc_{\vecTheta}(\cdot)$ is expressed as a linear combination of known basis functions as in~\eqref{eq:Function_Space}, the DBS capture the statistics of these basis functions evaluated at different time steps along the stochastic state trajectory.

To illustrate how state uncertainty affects the estimation of the parameter vector $\vecTheta$, we consider the following simple example with a linear basis function $\phi(\vecX_{t}) = \vecX_{t}$ in~\eqref{eq:Dynamic_Observation_Models}, where $\vecTheta$ is a scalar parameter to be inferred from a single measurement. Specifically, we focus on a one-dimensional dynamical system with $\nX = 1$ and only an initial-state prior $\vecX_{0} \sim \Pdist(\MuX{0}, \SigmaX{0})$. This example demonstrates that the presence of a latent state variable can give rise to a nonlinear parameter estimator. In this simple setting, the empirical MMSE estimate is tractable and is taken to represent the optimal Bayesian MMSE estimator. The estimator obtained from the optimal Bayesian MMSE estimate of $\vecTheta$ exhibits pronounced nonlinearity and deviates significantly from the affine estimator.

\begin{tcolorbox}[colback=darkblue2!7, colframe=darkblue2!35, boxrule=0.5pt, arc=2pt, breakable]
\begin{examplepar}[(1-D linear state-space model)] 
\label{ex:1d_simple}
Let the output function $\trueFunc_{\vecTheta}$ in~\eqref{eq:Dynamic_Observation_Models} be defined using a single linear basis function, $\phi_n(\vecX) = \vecX$, so that
\begin{equation}
\label{eq:linear_measurement}
\vecY = \vecTheta \vecX + \vecV,
\end{equation}
where $\vecX \in \Real$ is the initial state with $\vecX \sim \NormalDist(\MuX{}, \sigma_{\vecX}^{2})$, $\vecTheta \in \Real$ is an unknown scalar parameter, and $\vecV \sim \NormalDist(0, \sigmaV{}^{2})$ is measurement noise with $\sigmaV{}^{2} = 0.16$. In this simulation, the prior on $\vecTheta$ is taken to be the uniform distribution on $[-1, 5]$ with mean $\MuTheta{} = 2$ and variance $\sigmaTheta{}^{2} = 3$. Figure~\ref{Fig:Compare_MMSE_Affine} shows the optimal affine and empirical MMSE estimators for $\vecTheta$ as functions of the measurement $\vecY$, i.e., $\hat{\vecTheta}(\vecY)$, for two choices of initial-state statistics: $(\MuX, \sigma_{\vecX}^{2}) = (0.5, 0.05)$ (left plot) and $(\MuX, \sigma_{\vecX}^{2}) = (0.5, 4)$ (right plot). The empirical MMSE estimator (red) is computed using MCMC methods, while the affine MMSE estimator (blue) is obtained from the analytical expression in~\cite{vakili2025optimal}. This example illustrates that, depending on the prior on $\vecX$ and the measurement range, the true MMSE estimator may be close to the optimal affine MMSE estimator in some regions of $\vecY$, while exhibiting nonlinear deviations in others.

\medskip
\begin{center}
\includegraphics[width=0.45\textwidth]{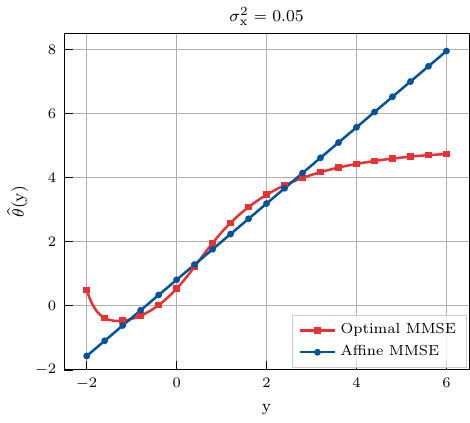}
\quad
\includegraphics[width=0.4585\textwidth]{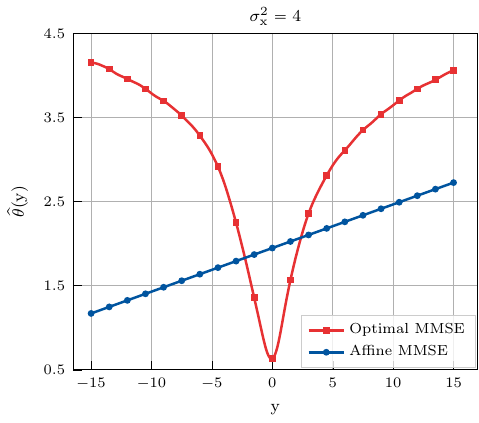}
\vspace{-0.5\baselineskip}
\captionof{figure}{\small Optimal MMSE in~\eqref{eq:Exact_Bayesian_MMSE} vs.\ affine MMSE in~\eqref{eq:Parameters_Affine_Estimation} for estimating $\vecTheta$ given $\vecY$}
\label{Fig:Compare_MMSE_Affine}
\end{center}
\end{examplepar}
\end{tcolorbox}

In this work, we propose an analytical approach to construct a nonlinear estimator that improves upon the affine class. We revisit Example~\ref{ex:1d_simple} throughout the paper to illustrate the behaviour of the proposed estimation algorithms.
\paragraph{\bf Contributions}
Building on this direction, we propose a nonlinear parameter estimator that retains the functional structure of the affine MMSE estimator for the unknown parameters while using new DBS estimates to reduce the mean-squared estimation error, as illustrated in Figure~\ref{Fig:Nonlinear_Parameter_Estimator_Architecture}. As shown in the block diagram, the proposed estimator consists of two coupled components: block~(I), described in Section~\ref{sec:Bayesian_Estimation_for_Wiener_Model}, implements the affine MMSE parameter estimator, which takes as input the entire measurement sequence $\{\vecY_{0}, \ldots, \vecY_{\Trajectory}\}$, the prior $(\MuTheta{}, \SigmaTheta{})$ on the unknown parameters, and DBS estimates $(\hat{\matrixPhi}, \hatSigmaPhi{})$ summarizing the basis-function evaluations. This block further produces the parameter estimates and their estimation-error covariance $(\hat{\vecTheta}, \hatSigmaTheta{})$, which are used as input to the DBS estimator in block~(II). Block~(II), discussed in Section~\ref{sec:Nonlinear_Fixedpoint_Estimation}, implements the DBS estimator, which uses the parameter estimates $(\hat{\vecTheta}, \hatSigmaTheta{})$, the same measurement sequence, and information about the state-trajectory distribution $p(\vecX_{0}, \ldots, \vecX_{\Trajectory})$ to construct updated DBS statistics. This nonlinear estimator architecture induces an interdependence between the two blocks: the parameter estimator requires DBS statistics, and the DBS estimator, in turn, depends on the parameter estimates. We resolve this interdependence via a fixed-point characterization that iteratively refines both sets of statistics. Leveraging the closed-form structure of the affine MMSE estimator, we investigate two DBS estimation methods for the DBS estimator in block~(II) of Figure~\ref{Fig:Nonlinear_Parameter_Estimator_Architecture}, as detailed in Figure~\ref{Fig:DBS_Estimator_Architecture}:

\begin{figure}[t]
\centering 
\includegraphics[width=0.8\linewidth]{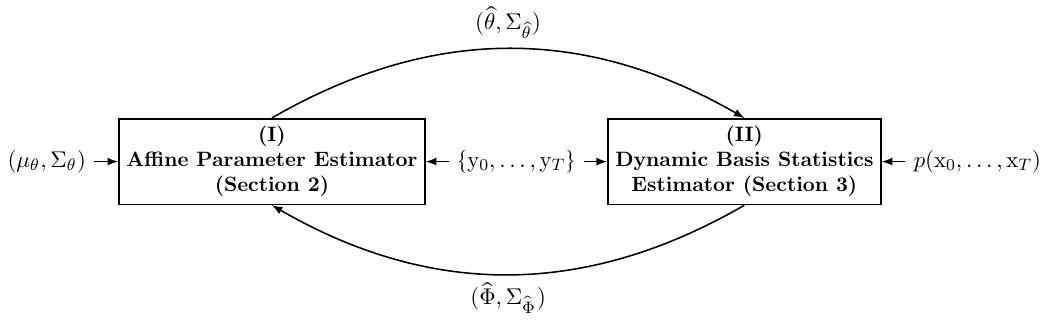} 
\caption{Architecture of the proposed nonlinear parameter estimator}
\label{Fig:Nonlinear_Parameter_Estimator_Architecture}
\end{figure}

\begin{itemize}[itemsep = 2mm, topsep = 1mm, leftmargin = 7.5mm]
\item \textbf{Basis-estimation approach (\subref{Fig:DBS_Estimator_Architecture_v1}):} This approach directly estimates basis-function evaluations over the entire trajectory. Within a Bayesian setting, we derive a closed-form expression for the optimal affine MMSE estimator of these unknown basis evaluations (Lemma~\ref{lemma:MMSE_Affine_Basis_Estimator}), which depends on the prior mean and covariance of the unknown parameters. As shown in subfigure~\subref{Fig:DBS_Estimator_Architecture_v1}, the DBS estimator uses information about the state-trajectory distribution $p(\vecX_{0}, \ldots, \vecX_{\Trajectory})$ to construct prior DBS statistics $(\matrixPhibar, \SigmaPhi)$ for the basis evaluations via Definition~\ref{definition:Dynamic_Basis_Statistics_Collection}. This prior, together with the measurement sequence and the parameter-estimate statistics $(\hat{\vecTheta}, \hatSigmaTheta{})$, is then processed by the affine basis estimator to obtain DBS estimates $(\hat{\matrixPhi}, \hatSigmaPhi{})$. Using this DBS estimator in block~(II) and combining it with the affine parameter estimator in block~(I) yields the dual basis-parameter (\ref{eq:Dual_basis_parameter_estimator}) estimator within the architecture of Figure~\ref{Fig:Nonlinear_Parameter_Estimator_Architecture}. However, numerical experiments show that this approach can underperform the purely affine MMSE parameter estimator in certain regimes (cf.~Figure~\ref{Fig:basisParam_Error_Setup1}), highlighting the limitations discussed in Remark~\ref{remark:Limitations_dual_basis_parameter} and the importance of explicitly estimating the latent state dynamics.

\item \textbf{State-estimation approach (\subref{Fig:DBS_Estimator_Architecture_v2}):} In contrast to the basis-estimation approach, this method first estimates the latent state trajectory and then constructs the DBS from the resulting state estimates. Assuming Gaussian process noise, $\vecW_{t+1} \sim \NormalDist(\zero, \SigmaW{t+1})$, we derive a closed-form expression for the optimal affine MMSE state estimator (Lemma~\ref{lemma:MMSE_Affine_State_Estimator}). As shown in subfigure~\subref{Fig:DBS_Estimator_Architecture_v2}, the affine state estimator uses information about the state-trajectory distribution $p(\vecX_{0}, \ldots, \vecX_{\Trajectory})$, together with the measurement sequence and the parameter-estimate statistics $(\hat{\vecTheta}, \hatSigmaTheta{})$ in place of its prior, to compute state estimates and their estimation-error covariance. These state-estimate statistics $(\hat{\vecXbar}, \hatSigmaXbar{})$ are then interpreted as the mean and covariance of a Gaussian prior over the system states and passed to the Gaussian DBS operator in Definition~\ref{definition:Gaussian_DBS_operator} to obtain DBS estimates $(\hat{\matrixPhi}, \hatSigmaPhi{})$. Using this DBS estimator in block~(II) and pairing it with the affine parameter estimator in block~(I) yields the dual state-parameter (\ref{eq:Dual_state_parameter_estimator}) estimator within the architecture of Figure~\ref{Fig:Nonlinear_Parameter_Estimator_Architecture}. Extensive numerical experiments indicate that this approach achieves a lower average parameter estimation error than both the dual basis-parameter (\ref{eq:Dual_basis_parameter_estimator}) estimator and a purely affine parameter estimator in two different configurations (cf.~Figures~\ref{Fig:basisParam_Error_Setup1}, \ref{Fig:param_Error_distributions}, and \ref{Fig:param_Error_differentMeasurements}).
\end{itemize}

\begin{figure}[t]
\centering
\begin{subfigure}[t]{0.49\linewidth}
  \centering
  \includegraphics[width=\linewidth]{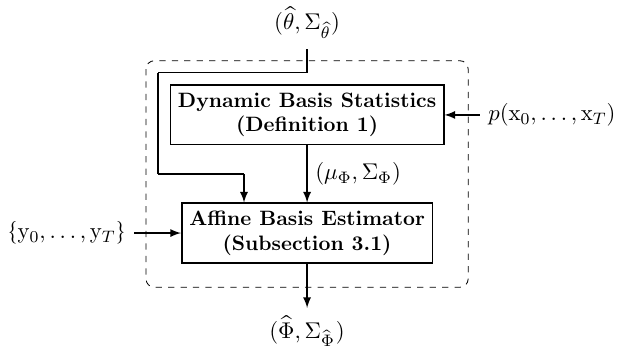}
  \caption{Basis-estimation approach}
  \label{Fig:DBS_Estimator_Architecture_v1}
\end{subfigure}
\hfill
\begin{subfigure}[t]{0.49\linewidth}
  \centering
  \includegraphics[width=\linewidth]{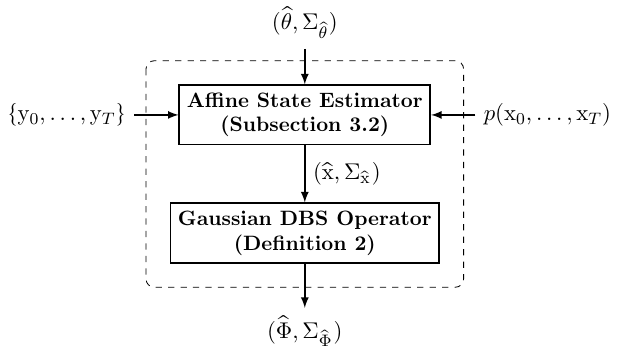}
  \caption{State-estimation approach}
  \label{Fig:DBS_Estimator_Architecture_v2}
\end{subfigure}
\caption{Architectures of the Dynamic Basis Statistics estimator}
\label{Fig:DBS_Estimator_Architecture}
\end{figure}

Combining either of the proposed DBS estimation methods with the affine parameter estimator induces an interdependence between the parameters and the DBS statistics, resolved by the algebraic equations defining \ref{eq:Dual_basis_parameter_estimator} and \ref{eq:Dual_state_parameter_estimator} estimators. These estimators naturally lead to fixed-point characterizations, implemented by Algorithm~\ref{algo:Dual_Estimators}, that alternate between estimating each set of unknowns using the updated prior of the other set derived from that set's estimate statistics from the previous iteration (see the plug-in interpretation in Remark~\ref{remark:Plugin_interpretation} for the prior-update concept). This mechanism is enabled by the structure of the affine MMSE estimators, whose coefficients depend only on the means and covariances of the unknowns rather than their full distributions. All affine estimators in \ref{eq:Dual_basis_parameter_estimator} and \ref{eq:Dual_state_parameter_estimator} admit explicit expressions for Fourier basis functions, yielding fully explicit DBS formulas in this case~\cite{vakili2025optimal}. To facilitate reproducibility, an open-source MATLAB library implementing all proposed methods for Fourier basis functions is provided at \href{https://github.com/sasanvakili/nonlinearBayesian4Wiener}{https://github.com/sasanvakili/nonlinearBayesian4Wiener}.

Similar ideas of alternating between state and parameter estimates have been explored in the literature, most notably in the Dual Kalman Smoother (DKS) of~\cite{wan1996dual}, the iterative smoothing scheme of the dual estimation method~\cite[Ch.~5]{haykin2004kalman}, and in SMC-based approaches such as PGAS~\cite{lindsten2014particle} and SMC-EM (particle EM)~\cite{schon_system_2011}. Both DKS and PGAS target the joint smoothing distribution $p(\vecTheta, \{\vecX_{t}\}_{t=0}^{\Trajectory} \!\mid\! \{\vecY_{t}\}_{t=0}^{\Trajectory})$ and conceptually treat the two unknown sets as separate blocks, updating one while treating the other as fixed at its current estimate. DKS approximates the conditional posteriors of the latent trajectory and the parameters by Gaussian distributions and alternates between two conditional smoother blocks: it iteratively estimates the system states using the latest parameter estimates via an extended Kalman smoother, and then estimates the parameters from the currently smoothed states until convergence. By contrast, PGAS does not make this Gaussian approximation and instead alternates between sampling the latent trajectory and the parameters from their full conditional posteriors. EM-based methods, on the other hand, treat the parameters as fixed in the E-step to approximate the state-trajectory posterior $p(\{\vecX_{t}\}_{t=0}^{\Trajectory} \!\mid\! \vecTheta, \{\vecY_{t}\}_{t=0}^{\Trajectory})$, and then update $\vecTheta$ in the M-step by maximizing an objective function obtained from integrating over this posterior. Our approach differs from these schemes in that it propagates both first- and second-order statistics between the coupled estimators. Extensive numerical experiments show that the proposed dual state-parameter (\ref{eq:Dual_state_parameter_estimator}) estimator achieves a lower average parameter estimation error than both PGAS and the EM-based method (cf.~Figure~\ref{Fig:stateParam_Error_differentMethods}).
\paragraph{\bf Organization}
Section~\ref{sec:Bayesian_Estimation_for_Wiener_Model} formulates the MMSE estimation problem for the parameterized output function~\eqref{eq:Function_Space}, reviews the affine parameter estimator of~\cite{vakili2025optimal}, and motivates the nonlinear parameter estimator structure. Section~\ref{sec:Nonlinear_Fixedpoint_Estimation} develops two approaches for constructing DBS estimates and introduces their fixed-point characterizations when combined with the parameter estimator, resulting in the dual basis-parameter and dual state-parameter estimators. Section~\ref{sec:Numerical_Experiments} presents numerical experiments comparing the proposed dual estimators with the affine parameter estimator and SMC-based methods. Detailed proofs of the mathematical statements are provided in Appendix~\ref{appendix:Technical_proofs}, titled ``Technical Proofs.''
\paragraph{\bf Notation}
Throughout this paper, $\Real$ and $\Real^{n \times m}$ denote the set of real numbers and the set of $n \times m$ real matrices, respectively. The symbol $\eye$ refers to the identity matrix. Subscripts denote elements of a vector (e.g., $\vecX_{t,i}$ represents the $i^{\nth}$ element of the vector $\vecX_{t}$), while superscripts represent the time index of vector or matrix elements (e.g., $\MuPhi^{t}$ for a vector and $\SigmaPhi^{tt'}$ for a matrix). The trace operator is denoted by $\trace(\cdot)$, the transpose of a matrix $\A$ is denoted by $\A^{\tr}\!$, and $\diag(\A_1,\ldots,\A_k)$ represents a block-diagonal matrix with diagonal entries $\A_1,\ldots,\A_k$. Given $\A\in\Real^{m \times n}$, a matrix with columns $\vecA_{1},\ldots,\vecA_{n}\in\Real^{m}$, $\vectorize(\A)$ is the vector $[\vecA_{1}^{\tr},\ldots,\vecA_{n}^{\tr}]^{\tr}\in\Real^{mn}$. The inner product of two vectors $\vecX$ and $\vecY$ is given by $\innerS{\vecX}{\vecY} = \vecX^{\tr}\vecY$, and the corresponding 2-norm is $\norm{\vecX} = \sqrt{\innerS{\vecX}{\vecX}}$. The inner product of two matrices $\A, \B \in \Real^{m \times n}$ is defined as $\innerS{\A}{\B} = \trace(\A^{\!\tr}\B)$. The conic inequality $\A \preceq \B$ means that the matrix difference $\B - \A$ is positive semidefinite, i.e., $\B - \A \succeq 0$. The notation $\Pdist(\mu,\Sigma)$ refers to an arbitrary distribution with mean $\mu$ and covariance matrix $\Sigma$, while a multivariate normal (Gaussian) distribution is denoted by $\NormalDist(\mu,\Sigma)$ and a uniform distribution over $[a,b]$ is denoted by $\UniformDist(a,b)$. The symbol $\sim$ stands for \emph{``distributed according to''}. For a random vector $\vecTheta \in \Real^{n}$, an estimate and its corresponding estimation-error covariance are denoted by $\hat{\vecTheta}$ and $\hatSigmaTheta{}$, respectively, and the indices of these quantities indicate the type of estimator: $(\hat{\vecTheta}_{\mathrm{opt}}, \hatSigmaTheta{\mathrm{opt}})$ are obtained from the optimal Bayesian MMSE estimator, $(\hat{\vecTheta}_{\mathrm{af}}, \hatSigmaTheta{\mathrm{af}})$  from the affine estimator, and $(\hat{\vecTheta}_{\mathrm{nl}}, \hatSigmaTheta{\mathrm{nl}})$ from a nonlinear estimator.
\section{Bayesian estimation for the Wiener model} \label{sec:Bayesian_Estimation_for_Wiener_Model}
In this section, we formulate the Bayesian MMSE estimation problem for the Wiener model and review key results on the Bayesian affine MMSE estimator from~\cite{vakili2025optimal}. We then discuss the relevance of these results and their connections to the proposed nonlinear estimator.
\subsection{Problem description} \label{subsec:Problem-Description}
Let the process noise $\vecW_t$, the measurement noise $\vecV_t$, the initial state $\vecX_0$, and the vector of parameters $\vecTheta$ be independent of one another at all times. Given the parametrized observation model~\eqref{eq:Function_Space} and state representation in~\eqref{eq:Dynamic_Observation_Models}, together with the sequence of measurements $\vecYbar = [\vecY_{0}, \ldots, \vecY_{\Trajectory}]^{\tr}\!$, our objective is to design an estimator $\hat{\vecTheta}(\cdot)$ that minimizes the following mean squared error (MSE):
\begin{equation} \label{problem:Problem}
\min_{\hat{\vecTheta}(\cdot) \in \estimatorFunc} \, \bigExpectation{\big\|\vecTheta - \hat{\vecTheta}(\vecYbar)\big\|^2},
\end{equation}
where $\estimatorFunc \coloneqq \left\{ \hat{\vecTheta} \colon \Real^{\Trajectory+1} \!\! \to \! \Theta \right\}$ denotes the class of all estimators, $\Trajectory+1$ is the length of the data sequence, and $\Theta$ is the parameter space. By definition, among all $\hat{\vecTheta}$ in $\estimatorFunc$, the optimal Bayesian MMSE estimator $\hat{\vecTheta}_{\mathrm{opt}}$ for~\eqref{problem:Problem} coincides with the posterior mean of $\vecTheta$ given the measurements~\cite[Ch.~4]{levy2008principles}. Accordingly, the parameter estimate $\hat{\vecTheta}_{\mathrm{opt}}$ and its error covariance $\hatSigmaTheta{\mathrm{opt}}$ are
\begin{equation} \label{eq:Exact_Bayesian_MMSE}
\hat{\vecTheta}_{\mathrm{opt}} = \Expectation{\vecTheta|\vecYbar} = \int_{\Theta}\! \vecTheta p(\vecTheta|\vecYbar) \drm \vecTheta, \quad \hatSigmaTheta{\mathrm{opt}} = \Expectation{( \vecTheta - \hat{\vecTheta}_{\mathrm{opt}})(\vecTheta - \hat{\vecTheta}_{\mathrm{opt}})^{\!\tr}}.
\end{equation} 
However, this estimator is generally computationally intractable due to the complexity of the posterior distribution $p(\vecTheta|\vecYbar)$ and the lack of an explicit solution for the associated integral. In the following subsection, we briefly review the affine MMSE parameter estimator, which admits a closed-form solution~\cite{vakili2025optimal} and serves as the foundation for the new contributions presented in this work.
\subsection{From affine to nonlinear estimators} \label{subsec:Affine_Nonlinear_Estimators}
To obtain a tractable alternative to the Bayesian MMSE estimator in~\eqref{eq:Exact_Bayesian_MMSE}, we first restrict attention to the class of affine estimators, characterize the corresponding MMSE solution, and subsequently use this construction as a stepping stone towards more general nonlinear estimators. The following generic result characterizes the optimal affine MMSE estimator in terms of the mean and covariance of a pair of random vectors, serving as the basis for our subsequent specialization to the parameter estimation.

\begin{lemma}[\textbf{Affine MMSE estimator}] \label{lemma:Generic_affine_MMSE_estimator}
Consider a random~$\xi = (\xi_{1}, \xi_{2})$ with components $\xi_{1} \in \Real^{\nZ{1}}$ and $\xi_{2} \in \Real^{\nZ{2}}$, whose joint distribution is
$$
\begin{bmatrix} \xi_{1} \\ \xi_{2} \end{bmatrix}
\sim \Pdist\!\left(
\begin{bmatrix} \mu_{1} \\ \mu_{2} \end{bmatrix},
\begin{bmatrix}
\Sigma_{11} & \Sigma_{12} \\
\Sigma_{21} & \Sigma_{22}
\end{bmatrix}
\right),
$$
where $\mu_{i} \coloneqq \Expectation{\xi_{i}}$ and $\Sigma_{ij} \coloneqq \Expectation{(\xi_{i}-\mu_{i})(\xi_{j}-\mu_{j})^{\tr}\!}$ for $i,j \in \{1,2\}$ denote the means and covariances, respectively. Suppose that $\Sigma_{22}$ is invertible and restrict the estimation of $\xi_{1}$ from $\xi_{2}$ to affine estimators,
\begin{equation} \label{eq:Affine_Class}
\estimatorAffine{\zeta} \!
\coloneqq \left\{\, \hat{\xi}_{1}(\xi_{2}) = \Psi \xi_{2} + \psi \,\middle|\, \Psi \in \Real^{\nZ{1} \times \nZ{2}},\, \psi \in \Real^{\nZ{1}} \right\}.
\end{equation}
\begin{subequations} \label{eq:Affine_MMSE_estimate_Cov}
For the Bayesian MMSE problem $\Cost^{\star} := \min_{\hat{\xi}_{1}(\cdot) \in \estimatorAffine{\zeta}} \, \Expectation{\|\xi_{1} - \hat{\xi}_{1}(\xi_{2})\|^2}$,
the solution comprises the affine MMSE estimate, its error covariance, and optimal cost:
\begin{equation} \label{eq:Affine_MMSE_closed_form}
\begin{cases}
\hat{\xi}_{1}(\xi_{2}) = \Psi^{\star} \xi_{2} + \psi^{\star}, \\[0.2em]
\Sigma_{\hat{\xi}_{1}} \coloneqq \bigExpectation{\big(\xi_{1} - \hat{\xi}_{1}(\xi_{2})\big) \big(\xi_{1} - \hat{\xi}_{1}(\xi_{2})\big)^{\!\!\tr}\!}
= \Sigma_{11} - \Sigma_{12}\Sigma_{22}^{\inv}\Sigma_{21}, \\[0.2em]
\Cost^{\star} = \trace(\Sigma_{\hat{\xi}_{1}}),
\end{cases}
\end{equation}
with the optimal coefficients given by
\begin{equation} \label{eq:Affine_optimal_coefficients}
\Psi^{\star} = \Sigma_{12}\Sigma_{22}^{\inv},
\qquad
\psi^{\star} = \mu_{1} - \Psi^{\star}\mu_{2}.
\end{equation}
\end{subequations}
\end{lemma}

The proof of this lemma is provided in Appendix~\ref{proof:Generic_affine_MMSE_estimator_proof}. Finding the affine MMSE estimator for estimating $\vecTheta$ from $\vecYbar$ using Lemma~\ref{lemma:Generic_affine_MMSE_estimator} requires the observation covariance $\Sigma_{\vecYbar}$ and the cross-covariance between the unknown parameters and the observation vector, $\Sigma_{\vecTheta\vecYbar}$. This task is nontrivial because the observation model~\eqref{eq:Function_Space} depends on the inner product of two unknown quantities, namely the basis-function evaluations of the unknown system states $\phi(\vecX)$ and the unknown parameters $\vecTheta$. A central component in characterizing the affine estimator for the MMSE problem~\eqref{problem:Problem}, as detailed in~\cite{vakili2025optimal}, is the propagation of the state trajectory through the output basis functions. The resulting quantities, known as Dynamic Basis Statistics (DBS), are formalized as follows.

\begin{definition} [\textbf{Dynamic Basis Statistics}]
\label{definition:Dynamic_Basis_Statistics_Collection}
Let the state sequence $\vecXbar = [\vecX_0^{\tr}, \ldots, \vecX_{\Trajectory}^{\tr}]^{\tr}\!$ be a trajectory of~\eqref{eq:Dynamic_Observation_Models} for $t \in \{0, \ldots, \Trajectory\}$, and let $\matrixPhi(\vecXbar) = [\phi(\vecX_{0}), \ldots, \phi(\vecX_{\Trajectory})]$ be the matrix whose columns are the basis function vectors evaluated at each $\vecX_t$ as in~\eqref{eq:Function_Space}. The Dynamic Basis Statistics (DBS) collection with respect to the joint distribution of the state trajectory is defined as
\begin{equation} \label{eq:DBS_Collection}
\matrixPhibar = [\MuPhi^{0}, \ldots, \MuPhi^{\Trajectory}], \qquad \SigmaPhi = \begin{bmatrix} \TimedSigmaPhi{^{00}} & \ldots & \TimedSigmaPhi{^{0\Trajectory}} \\ 
\vdots & \ddots &  \vdots \\ 
\TimedSigmaPhi{^{\Trajectory0}} & \ldots & \TimedSigmaPhi{^{\Trajectory\Trajectory}}
\end{bmatrix},
\end{equation}
where $\MuPhi^{t}$ is the mean of $\phi(\vecX_{t})$ and $\TimedSigmaPhi{^{tt'}}$ is the covariance between $\phi(\vecX_{t})$ and $\phi(\vecX_{t'})$, given by
\begin{equation} \label{eq:Random_Phi}
\MuPhi^{t} \coloneqq \Expectation{\phi(\vecX_{t})}, \qquad \TimedSigmaPhi{^{tt'}} \coloneqq \Expectation{\phi(\vecX_{t})\phi^{\tr}\!\!(\vecX_{t'})} - 
\Expectation{\phi(\vecX_{t})}\Expectation{\phi(\vecX_{t'})}^{\tr}\!.
\end{equation}
\end{definition}

This DBS collection enables marginalization over the latent system states, which in turn allows derivation of the closed-form solution for the Bayesian affine MMSE estimator.

\begin{theorem}[\textbf{Affine MMSE parameter estimator}{~\cite[Thm.~3.2]{vakili2025optimal}}] \label{theorem:MMSE_Affine_Parameter_Estimator}
The Bayesian MMSE estimator of Problem~\ref{problem:Problem} within the affine class, analogous to~\eqref{eq:Affine_Class}, and its corresponding estimation-error covariance with the optimal cost, as functions of $\matrixPhibar$ and $\SigmaPhi$, are given by 
\begin{subequations}\label{eq:opt_Bayes_Parameter}
\begin{equation} \label{eq:Parameters_Affine_Estimation}
\begin{cases}
\hat{\vecTheta}_{\mathrm{af}}(\vecYbar ; \matrixPhibar, \SigmaPhi) = \PsiTheta^{\star}(\matrixPhibar, \SigmaPhi) \vecYbar + \psiTheta^{\star}(\matrixPhibar, \SigmaPhi), \\[0.2em]
\hatSigmaTheta{\mathrm{af}}\!(\matrixPhibar, \SigmaPhi) = \SigmaTheta{} - \SigmaTheta{}\matrixPhibar \big( \matrixPhibar^{\tr} \SigmaTheta{} \matrixPhibar + \operatorDBS(\SigmaPhi) + \SigmaVbar{} \big)^{\inv}\!\matrixPhibar^{\tr} \SigmaTheta{}, \\[0.2em]
\Cost^{\star}_{\vecTheta}(\matrixPhibar, \SigmaPhi) 
= \trace(\hatSigmaTheta{\mathrm{af}}\!(\matrixPhibar, \SigmaPhi)),
\end{cases}
\end{equation}
where the coefficient matrices are defined as
\begin{equation} \label{eq:MMSE_LinearEstimator}
\PsiTheta^{\star}(\matrixPhibar, \SigmaPhi) = \SigmaTheta{}\matrixPhibar \big( \matrixPhibar^{\tr} \SigmaTheta{}\matrixPhibar + \operatorDBS(\SigmaPhi) + \SigmaVbar{} \big)^{\inv}\!, \quad 
\psiTheta^{\star}(\matrixPhibar, \SigmaPhi) = \MuTheta{} - \PsiTheta^{\star}(\matrixPhibar, \SigmaPhi) \matrixPhibar^{\tr} \MuTheta{}, 
\end{equation}
$\operatorDBS \colon \Real^{((\Ntheta+1)(\Trajectory+1))^{2}} \!\! \to \! \Real^{(\Trajectory+1)^2}$ is a linear operator whose $(t,t')^{\nth}$ component is 
\begin{equation} \label{eq:Covariance_Operator}
\operatorDBS^{tt'}\!(\SigmaPhi) = \innerS{\TimedSigmaPhi{^{tt'}}}{(\SigmaTheta{} + \MuTheta{}\MuTheta{}^{\tr})},
\end{equation}
and $\SigmaVbar{}$ is the covariance matrix of the measurement noise sequence $\vecVbar = [\vecV_{0}, \ldots, \vecV_{\Trajectory}]^{\tr}\!$, which is the diagonal matrix~$\SigmaVbar{} = \diag(\sigmaV{0}^{2}, \ldots, \sigmaV{\Trajectory}^{2})$, when they are mutually independent.
\end{subequations}
\end{theorem}

\begin{remark}[\textbf{Functional dependence}]
In principle, the estimator $\hat{\vecTheta}_{\mathrm{af}}(\cdot)$ is a function of three arguments. However, the notation used in~\eqref{eq:Parameters_Affine_Estimation}, i.e., $\hat{\vecTheta}_{\mathrm{af}}(\vecYbar ; \matrixPhibar, \SigmaPhi)$, emphasizes the dependence of the estimator on the DBS collection $(\matrixPhibar,\SigmaPhi)$ in the last two arguments. This estimator depends linearly on the measurements $\vecYbar$ and nonlinearly on $(\matrixPhibar,\SigmaPhi)$. Note that the DBS collection itself is generated by a nonlinear transformation of the inputs and system dynamics and does not involve the measurements.
\end{remark}

Given Definition~\ref{definition:Dynamic_Basis_Statistics_Collection}, the proof of Theorem~\ref{theorem:MMSE_Affine_Parameter_Estimator} follows directly from the generic affine MMSE estimator of Lemma~\ref{lemma:Generic_affine_MMSE_estimator} by setting $\xi_{1}=\vecTheta$ and $\xi_{2}=\vecYbar$ (see~\cite[Thm.~3.2]{vakili2025optimal} for a complete proof). While $\hat{\vecTheta}_{\mathrm{af}}(\cdot)$ is provably optimal within the affine class, the proximity of its MSE to that of the Bayesian MMSE estimator, $\hat{\vecTheta}_{\mathrm{opt}}$, remains an open question. When perfect state knowledge is available, the measurements contain uncertainty only through the measurement noise. In this case, the affine MMSE estimator returns $\hat{\vecTheta}_{\mathrm{opt}}$ if the prior on the unknown parameters, $\Pdist(\MuTheta{}, \SigmaTheta{})$, and the distribution of the measurement noise sequence, $\Pdist(0, \SigmaVbar{})$, belong to an elliptical family~\cite[Rem.~3.1]{vakili2025optimal}. As observed in the simple case of Example~\ref{ex:1d_simple}, $\hat{\vecTheta}_{\mathrm{af}}(\cdot)$ can track $\hat{\vecTheta}_{\mathrm{opt}}$ well for certain priors on $\vecX$ and ranges of $\vecY$, particularly when the state uncertainty is small. Let us revisit Example~\ref{ex:1d_simple} to see how Theorem~\ref{theorem:MMSE_Affine_Parameter_Estimator} specializes to an explicit expression for the affine estimator in that setting.

\begin{tcolorbox}[colback=darkblue2!7, colframe=darkblue2!35, boxrule=0.5pt, arc=2pt, breakable]
\paragraph{\textbf{Example~\ref{ex:1d_simple} (continued)}}
In the setting of~\eqref{eq:linear_measurement}, the dynamical system is one-dimensional, and the observation model involves the linear basis function $\phi(\vecX) = \vecX$. In this case, the DBS from Definition~\ref{definition:Dynamic_Basis_Statistics_Collection} simplify to
$$
\matrixPhibar = \MuX{}, \qquad \SigmaPhi = \sigma_{\vecX}^{2},
$$
and the affine MMSE estimator and its estimation-error covariance from Theorem~\ref{theorem:MMSE_Affine_Parameter_Estimator} are
\begin{equation}
\label{eq:affine_param_1D}
\begin{aligned}
\hat{\vecTheta}_{\mathrm{af}}(\vecY ; \MuX{}, \sigma_{\vecX}^{2}) & = \MuTheta{} + \dfrac{\sigmaTheta{}^{2} \MuX{}}{\MuX{}^{2}\sigmaTheta{}^{2} + \sigma_{\vecX}^{2}\sigmaTheta{}^{2} + \sigma_{\vecX}^{2}\MuTheta{}^{2} + \sigmaV{}^{2}}\, (\vecY - \MuX{} \MuTheta{}) \fracComma \\
\hatSigmaTheta{\mathrm{af}}\!(\MuX{}, \sigma_{\vecX}^{2}) & = \sigmaTheta{}^{2} - \dfrac{\sigmaTheta{}^{4} \MuX{}^{2}}{\MuX{}^{2}\sigmaTheta{}^{2} + \sigma_{\vecX}^{2}\sigmaTheta{}^{2} + \sigma_{\vecX}^{2}\MuTheta{}^{2} + \sigmaV{}^{2}} \fracDot
\end{aligned}
\end{equation}
\end{tcolorbox}

However, as shown in Figure~\ref{Fig:Compare_MMSE_Affine}, the Bayesian MMSE estimator for Example~\ref{ex:1d_simple} can become highly nonlinear as the uncertainty in the system state increases. Consequently, the MSE of the affine estimator in~\eqref{eq:affine_param_1D} may be far from that of $\hat{\vecTheta}_{\mathrm{opt}}$. More generally, for hidden Markov processes (HMPs), the MSE decomposes into two terms~\cite{ephraim2002hidden, ephraim2002lower}: the MSE achieved when the system state is perfectly known, and an additional cross-error term for which no explicit expression is generally available. An analogous decomposition holds for the class of affine estimators marginalized over the latent system states. The following lemma demonstrates that the expected error of the affine MMSE estimator is bounded below by the fundamental limit corresponding to complete knowledge of the system state.

\begin{lemma}[\textbf{Lower bound on the MSE of the affine estimator}] \label{lemma:MMSE_Lowerbound}
Consider the affine MMSE estimator $\hat{\vecTheta}_{\mathrm{af}}(\vecYbar ; \matrixPhibar, \SigmaPhi)$ and its optimal error~$\Cost^{\star}_{\vecTheta}(\matrixPhibar, \SigmaPhi)$ as defined in~\eqref{eq:Parameters_Affine_Estimation}. , we have
\begin{equation}
\label{eq:MSE_lower_bound}
\Cost^{\star}_{\vecTheta}(\matrixPhibar, \SigmaPhi) \geq \CondExpectation{\vecXbar}{\Cost^{\star}_{\vecTheta}(\matrixPhi(\vecXbar), 0)} = \CondExpectation{\vecXbar}{\min_{\hat{\vecTheta}(\cdot) \in \estimatorAffine{\vecTheta}} \, \bigExpectation{\big\|\vecTheta - \hat{\vecTheta}(\vecYbar)\big\|^2 \,\big|\, \vecXbar}},
\end{equation}
where $\vecXbar$ is the state trajectory, $\matrixPhi(\vecXbar)$ is the DBS in Definition~\ref{definition:Dynamic_Basis_Statistics_Collection}, the set $\estimatorAffine{\vecTheta}$ denotes the class of affine estimators of $\vecTheta$ from $\vecYbar$, analogous to~\eqref{eq:Affine_Class}, and the inner expectation on the right-hand side is conditional on $\vecXbar$, while the outer expectation is taken with respect to the marginal distribution of $\vecXbar$.
\end{lemma}

The proof can be found in Appendix~\ref{proof:MMSE_Lowerbound_proof}. Notably, the term~$\Cost^{\star}_{\vecTheta}\big(\matrixPhi(\vecXbar), 0\big)$ in \eqref{eq:MSE_lower_bound} is indeed the optimal estimation error achievable by an affine MMSE estimator when the state trajectory~$\vecXbar$ is perfectly known. In this light, the lower bound in Lemma~\ref{lemma:MMSE_Lowerbound} suggests a fundamental performance limit attainable by affine estimators when the state trajectory is perfectly known. This gap motivates the development of nonlinear parameter estimators that more effectively exploit the DBS structure and the latent state information, with the aim of reducing the difference between the achievable MSE and the fundamental limit suggested by Lemma~\ref{lemma:MMSE_Lowerbound}. Consequently, this study proposes a nonlinear parameter estimator by composing the affine MMSE mapping~\eqref{eq:opt_Bayes_Parameter} with data-dependent DBS estimates. Let $\hat{\matrixPhi}(\vecYbar)$ and $\hatSigmaPhi{}(\vecYbar)$ denote estimates of the DBS components constructed from the measurements $\vecYbar$. Substituting these estimates into affine MMSE expressions $\hat{\vecTheta}_{\mathrm{af}}(\vecYbar ; \cdot, \cdot)$ and $\hatSigmaTheta{\mathrm{af}}\!(\cdot, \cdot)$ in~\eqref{eq:Parameters_Affine_Estimation} yields the nonlinear estimator
\begin{equation} \label{eq:Nonlinear_Parameter_Estimation}
\begin{cases}
\hat{\vecTheta}_{\mathrm{nl}}(\vecYbar) = \hat{\vecTheta}_{\mathrm{af}}\big(\vecYbar ; \hat{\matrixPhi}(\vecYbar), \hatSigmaPhi{}(\vecYbar)\big), \\[0.2em]
\hatSigmaTheta{\mathrm{nl}}\!(\vecYbar) = \hatSigmaTheta{\mathrm{af}}\!\big(\hat{\matrixPhi}(\vecYbar), \hatSigmaPhi{}(\vecYbar)\big).
\end{cases}
\end{equation}
Thus, $\hat{\vecTheta}_{\mathrm{nl}}(\vecYbar)$ has the same functional form as the affine MMSE estimator, but is nonlinear with respect to $\vecYbar$ because its DBS arguments are replaced by data-driven estimates. 

\begin{remark}[\textbf{Plug-in interpretation}] \label{remark:Plugin_interpretation}
The nonlinear estimator~\eqref{eq:Nonlinear_Parameter_Estimation} can be interpreted as a plug-in construction with respect to the DBS statistics. Whereas the affine MMSE estimator~\eqref{eq:Parameters_Affine_Estimation} treats $(\matrixPhibar,\SigmaPhi)$ as known hyperparameters, the nonlinear variant substitutes their data-dependent estimates $\big(\hat{\matrixPhi}(\vecYbar), \hatSigmaPhi{}(\vecYbar)\big)$ into the same affine formulas. This procedure is analogous to empirical Bayes methods, where hyperparameters are first estimated from data and subsequently treated as fixed when evaluating the conditional estimator~\cite{CarlinLouis2000}. Consequently, the independence assumptions underlying the affine derivation no longer hold exactly, since the effective DBS now depend on $\vecYbar$. The covariance $\hatSigmaTheta{\mathrm{nl}}\!(\vecYbar)$ is thus computed under the approximation that the plugged-in DBS are fixed rather than random. Nevertheless, this construction retains the affine functional form while incorporating DBS estimates.
\end{remark}

The subsequent section presents two approaches for constructing such DBS estimates, each resulting in a fixed-point characterization of the nonlinear parameter estimator.
\section{Fixed-Point Characterization of Nonlinear Estimation}
\label{sec:Nonlinear_Fixedpoint_Estimation}
The nonlinear parameter estimator~\eqref{eq:Nonlinear_Parameter_Estimation} depends critically on accurate DBS estimates $(\hat{\matrixPhi}(\vecYbar), \hatSigmaPhi{}(\vecYbar))$ that improve parameter estimation performance beyond the affine counterpart. Since $(\matrixPhibar,\SigmaPhi)$ are determined by the latent state sequence, we investigate two distinct strategies: (i) direct estimation of the basis-function evaluations $\phi(\vecX)$ in~\eqref{eq:Function_Space} over the entire trajectory, and (ii) estimation of the internal latent system states $\vecXbar$ followed by DBS construction from the resulting state estimates. Crucially, an interdependency arises in both strategies, as the DBS estimates require parameter estimates and vice versa, which are addressed by finding a fixed-point solution in the following subsections.
\subsection{Dual basis-parameter estimation} \label{subsec:Dual_basis_parameter_estimation}
The fundamental limit in Lemma~\ref{lemma:MMSE_Lowerbound} suggests that DBS estimates should satisfy $\hat{\matrixPhi}(\vecYbar) \! \to \! \matrixPhi(\vecXbar)$ and $\hatSigmaPhi{}(\vecYbar) \! \to \! 0$ to improve the parameter estimation performance of $\hat{\vecTheta}_{\mathrm{nl}}(\vecYbar)$ compared to $\hat{\vecTheta}_{\mathrm{af}}(\vecYbar ; \matrixPhibar, \SigmaPhi)$. This subsection implements the first strategy by estimating the basis-function evaluations $\phi(\vecX)$ in~\eqref{eq:Function_Space} via an affine MMSE estimator, and then forming $\hat{\matrixPhi}(\vecYbar)$ and $\hatSigmaPhi{}(\vecYbar)$ from the resulting basis estimates and their estimation-error covariance. The following lemma is analogous to Theorem~\ref{theorem:MMSE_Affine_Parameter_Estimator} and formalizes this construction as an immediate consequence of Lemma~\ref{lemma:Generic_affine_MMSE_estimator}.

\begin{lemma} [\textbf{DBS estimation via affine MMSE basis estimator}] \label{lemma:MMSE_Affine_Basis_Estimator}
Let $\matrixPhi(\vecXbar)$ be the unknown basis-function matrix over the trajectory, with prior DBS $(\matrixPhibar,\SigmaPhi)$ as in Definition~\ref{definition:Dynamic_Basis_Statistics_Collection}, and define the stacked basis-function vector $\phi \coloneqq \vectorize(\matrixPhi(\vecXbar)) = [\phi^{\tr}\!\!(\vecX_{0}),\ldots,\phi^{\tr}\!\!(\vecX_{\Trajectory})]^{\tr}\!$. Given the prior mean and covariance of the unknown parameters $(\MuTheta{},\SigmaTheta{})$, the Bayesian affine MMSE estimator of $\phi$ admits the closed form $\hat{\phi}_\mathrm{af}(\vecYbar ; \MuTheta{}, \SigmaTheta{}) = \PsiPhi^{\star}(\MuTheta{}, \SigmaTheta{}) \vecYbar + \psiPhi^{\star}(\MuTheta{}, \SigmaTheta{})$, where
\begin{subequations}\label{eq:opt_Bayes_Basis}
\begin{equation} \label{eq:MMSE_BasisEstimator}
\PsiPhi^{\star}(\MuTheta{}, \SigmaTheta{}) = \SigmaPhi \mu_{\Theta} \big( \mu_{\Theta}^{\tr} \SigmaPhi \mu_{\Theta} + \operatorParam(\SigmaTheta{}) + \SigmaVbar{} \big)^{\inv}\!, \quad \psiPhi^{\star}(\MuTheta{}, \SigmaTheta{}) = \phiBar - \PsiPhi^{\star}(\MuTheta{}, \SigmaTheta{}) \matrixPhibar^{\tr} \MuTheta{},
\end{equation}
with $\phiBar = \vectorize(\matrixPhibar)$, 
$\mu_{\Theta} = \diag(\timedMuTheta{0}, \ldots, \timedMuTheta{\Trajectory})$ where $\timedMuTheta{t} = \MuTheta{}$ for each $t \in \{0, \ldots, \Trajectory\}$, $\SigmaVbar{}$ identical to that in the affine MMSE parameter estimator~\eqref{eq:opt_Bayes_Parameter}, and the linear operator $\operatorParam \colon \Real^{(\Ntheta+1)^{2}} \!\! \to \! \Real^{(\Trajectory+1)^2}$ having $(t,t')^{\nth}$ component
\begin{equation} \label{eq:MMSE_Basis_Covariance_Operator}
\operatorParam^{tt'}\!(\SigmaTheta{}) = \innerS{\SigmaTheta{}}{( \TimedSigmaPhi{^{tt'}}\! + \MuPhi^{t}\MuPhi^{t'^{\tr}\!})}.
\end{equation}
Consequently, the resulting DBS estimates and the corresponding MMSE are
\begin{equation} \label{eq:BasisCollection_Affine_Estimation}
\begin{cases}
\hat{\matrixPhi}_{\mathrm{af}}(\vecYbar ; \MuTheta{}, \SigmaTheta{}) = [\hat{\phi}_\mathrm{af}^{\mkern1mu 0}(\vecYbar ; \MuTheta{}, \SigmaTheta{}), \ldots, \hat{\phi}_\mathrm{af}^{\mkern1mu \Trajectory}(\vecYbar ; \MuTheta{}, \SigmaTheta{})], \\[0.2em]
\hatSigmaPhi{\mathrm{af}}\!(\MuTheta{}, \SigmaTheta{}) = \SigmaPhi - \SigmaPhi \mu_{\Theta} \big( \mu_{\Theta}^{\tr} \SigmaPhi \mu_{\Theta} + \operatorParam(\SigmaTheta{}) + \SigmaVbar{} \big)^{\inv}\! \mu_{\Theta}^{\tr} \SigmaPhi, \\[0.2em]
\Cost^{\star}_{\phi}(\MuTheta{}, \SigmaTheta{}) = \trace(\hatSigmaPhi{\mathrm{af}}\!(\MuTheta{}, \SigmaTheta{})),
\end{cases}
\end{equation}
where $\hat{\phi}_\mathrm{af}^{\mkern1mu t}(\vecYbar ; \MuTheta{}, \SigmaTheta{})$ is the estimate of $\phi(\vecX_{t})$.
\end{subequations}
\end{lemma}

\begin{remark}[\textbf{Observation covariance forms}] \label{remark:Observations_Covariance_Equivalence}
The Bayesian affine MMSE estimators for the parameters in~\eqref{eq:opt_Bayes_Parameter} and for the basis collection in~\eqref{eq:opt_Bayes_Basis} both rely on the observation covariance
\begin{equation} \label{eq:Equivalent_Observation_Covariance}
\Sigma_{\vecYbar}
= \matrixPhibar^{\tr} \SigmaTheta{}\matrixPhibar + \operatorDBS(\SigmaPhi) + \SigmaVbar{}
= \mu_{\Theta}^{\tr} \SigmaPhi \mu_{\Theta} + \operatorParam(\SigmaTheta{}) + \SigmaVbar{},
\end{equation}
where $\operatorDBS(\SigmaPhi)$ and $\operatorParam(\SigmaTheta{})$ are defined in~\eqref{eq:Covariance_Operator} and~\eqref{eq:MMSE_Basis_Covariance_Operator}, respectively. The two expressions follow from the lifted observation model written either as $\vecYbar = \matrixPhi^{\tr}\!\!(\vecXbar)\vecTheta + \vecVbar$ in~\cite{vakili2025optimal} or $\vecYbar = \matrixG_{\theta}^{\tr} \phi + \vecVbar$ in~\eqref{eq:Equivalent_Lifted_ObservationModel}, where $\phi = \vectorize(\matrixPhi(\vecXbar))$ and $\matrixG_{\theta} = \diag(\timedTheta{0}, \ldots, \timedTheta{\Trajectory})$ with with $\timedTheta{t} = \vecTheta$ for all $t \in \{ 0, \ldots, \Trajectory \}$ (see also Appendix~\ref{proof:MMSE_Affine_Basis_Estimator_proof}). The former highlights the dependence on DBS and is used for affine parameter estimation, whereas the latter highlights the dependence on the parameter prior and is used for affine estimation of DBS.
\end{remark}

For the proof of Lemma~\ref{lemma:MMSE_Affine_Basis_Estimator}, see Appendix~\ref{proof:MMSE_Affine_Basis_Estimator_proof}. In the context of Example~\ref{ex:1d_simple}, the DBS estimation admits an explicit expression via Lemma~\ref{lemma:MMSE_Affine_Basis_Estimator}. In this setting, because the basis function is linear, basis evaluations and states coincide. Hence, the DBS estimates in~\eqref{eq:BasisCollection_Affine_Estimation} for a single measurement is essentially an optimal affine MMSE state estimator.

\begin{tcolorbox}[colback=darkblue2!7, colframe=darkblue2!35, boxrule=0.5pt, arc=2pt, breakable]
\paragraph{\textbf{Example~\ref{ex:1d_simple} (continued)}}
Given the linear basis function $\phi(\vecX) = \vecX$ in~\eqref{eq:linear_measurement}, the DBS estimator described in~\eqref{eq:BasisCollection_Affine_Estimation} is equal to the optimal affine MMSE state estimator
\begin{equation}
\label{eq:affine_basis_1D}
\begin{aligned}
\hat{\matrixPhi}_{\mathrm{af}}(\vecYbar ; \MuTheta{}, \sigmaTheta{}^{2}) & = \MuX{} + \dfrac{\sigma_{\vecX}^{2} \MuTheta{}}{\MuTheta{}^{2} \sigma_{\vecX}^{2} + \sigmaTheta{}^{2} \sigma_{\vecX}^{2} + \sigmaTheta{}^{2} \MuX{}^{2} + \sigmaV{}^{2}}\, (\vecY - \MuX{} \MuTheta{}) \fracComma \\
\hatSigmaPhi{\mathrm{af}}\!(\MuTheta{}, \sigmaTheta{}^{2}) & = \sigma_{\vecX}^{2} - \dfrac{\sigma_{\vecX}^{4} \MuTheta{}^{2}}{\MuTheta{}^{2} \sigma_{\vecX}^{2} + \sigmaTheta{}^{2} \sigma_{\vecX}^{2} + \sigmaTheta{}^{2} \MuX{}^{2} + \sigmaV{}^{2}} \fracDot \\
\end{aligned}
\end{equation}
Observe that the operators in~\eqref{eq:Covariance_Operator} and~\eqref{eq:MMSE_Basis_Covariance_Operator} evaluate to
$$
\operatorDBS(\sigma_{\vecX}^{2}) = \sigma_{\vecX}^{2} (\sigmaTheta{}^{2} + \MuTheta{}^{2}), \qquad \operatorParam(\sigmaTheta{}^{2}) = \sigmaTheta{}^{2}(\sigma_{\vecX}^{2} + \MuX{}^{2}),
$$
and that the denominators in~\eqref{eq:affine_param_1D} and~\eqref{eq:affine_basis_1D} are equal, representing the observation covariance for a scalar measurement.
\end{tcolorbox}

The DBS estimates in~\eqref{eq:BasisCollection_Affine_Estimation} depend on the prior parameter statistics $(\MuTheta{}, \SigmaTheta{})$. Leveraging the fact that these estimates would improve if we had perfect knowledge of $\vecTheta$, we apply a plug-in interpretation analogous to Remark~\ref{remark:Plugin_interpretation} by treating the parameter estimates $\big(\hat{\vecTheta}(\vecYbar), \hatSigmaTheta{}(\vecYbar)\big)$ as fixed prior moments,
\begin{equation} \label{eq:DBS_Nonlinear_Estimation}
\begin{cases}
\hat{\matrixPhi}_{\mathrm{nl}}(\vecYbar) = \hat{\matrixPhi}_{\mathrm{af}}(\vecYbar ; \hat{\vecTheta}(\vecYbar), \hatSigmaTheta{}(\vecYbar)), \\[0.2em]
\hatSigmaPhi{\mathrm{nl}}\!(\vecYbar) = \hatSigmaPhi{\mathrm{af}}\!(\hat{\vecTheta}(\vecYbar), \hatSigmaTheta{}(\vecYbar)).
\end{cases}
\end{equation}
The mutual substitution of~\eqref{eq:DBS_Nonlinear_Estimation} and~\eqref{eq:Nonlinear_Parameter_Estimation} creates a circular dependency: the DBS estimates require the parameter estimates, which in turn require the DBS estimates. This interdependency yields a nonlinear estimator, the main message of this study, characterized as a solution (a.k.a. fixed-point) to the following algebraic equations.
 
\begin{tcolorbox}[colback=gray!10, colframe=gray!50, boxrule=0.5pt, arc=2pt, breakable]
\textbf{Dual basis-parameter estimator:} For any measurements vector~$\vecYbar$, there exists a tuple of basis-parameter estimators $\big(\hat{\vecTheta}_{\mathrm{nl}}^{\star}, \hatSigmaTheta{\mathrm{nl}}^{\star}, \hat{\matrixPhi}_{\mathrm{nl}}^{\star}, \hatSigmaPhi{\mathrm{nl}}^{\star}\big)$ such that
\begin{equation*}
\label{eq:Dual_basis_parameter_estimator} \tag{DB-P}
\begin{cases}
\hat{\vecTheta}_{\mathrm{nl}}^{\star} = \hat{\vecTheta}_{\mathrm{af}}(\vecYbar ; \hat{\matrixPhi}_{\mathrm{nl}}^{\star}, \hatSigmaPhi{\mathrm{nl}}^{\star}), \\[0.2em]
\hatSigmaTheta{\mathrm{nl}}^{\star} = \hatSigmaTheta{\mathrm{af}}\!(\hat{\matrixPhi}_{\mathrm{nl}}^{\star}, \hatSigmaPhi{\mathrm{nl}}^{\star}),
\end{cases}
\quad
\begin{cases}
\hat{\matrixPhi}_{\mathrm{nl}}^{\star} = \hat{\matrixPhi}_{\mathrm{af}}(\vecYbar ; \hat{\vecTheta}_{\mathrm{nl}}^{\star}, \hatSigmaTheta{\mathrm{nl}}^{\star}), \\[0.2em]
\hatSigmaPhi{\mathrm{nl}}^{\star} = \hatSigmaPhi{\mathrm{af}}\!(\hat{\vecTheta}_{\mathrm{nl}}^{\star}, \hatSigmaTheta{\mathrm{nl}}^{\star}),
\end{cases}
\end{equation*}
where the functions~$\big(\hat{\vecTheta}_{\mathrm{af}}(\cdot),\hatSigmaTheta{\mathrm{af}}\!(\cdot)\big)$ and $\big(\hat{\matrixPhi}_{\mathrm{af}}(\cdot),\hatSigmaPhi{\mathrm{af}}\!(\cdot)\big)$ are the closed form solutions of the affine Bayesian MMSE defined in \eqref{eq:Parameters_Affine_Estimation} and \eqref{eq:BasisCollection_Affine_Estimation}, respectively.
\end{tcolorbox}

The dual basis-parameter estimation framework thus comprises two affine MMSE estimators, each relying on the estimates of the other. Using the closed-form expressions of the affine MMSE parameter and DBS estimates in~\eqref{eq:affine_param_1D} and~\eqref{eq:affine_basis_1D} for Example~\ref{ex:1d_simple}, we illustrate the dual basis-parameter~\eqref{eq:Dual_basis_parameter_estimator} estimator on the same data and compare it with the affine and empirical MMSE estimators.

\begin{tcolorbox}[colback=darkblue2!7, colframe=darkblue2!35, boxrule=0.5pt, arc=2pt, breakable]
\paragraph{\textbf{Example~\ref{ex:1d_simple} (continued)}} 
The dual basis-parameter~\eqref{eq:Dual_basis_parameter_estimator} estimator for $\vecTheta$ given a single measurement, constructed using the affine MMSE parameter and basis estimators in~\eqref{eq:affine_param_1D} and~\eqref{eq:affine_basis_1D}, is shown in green in Figure~\ref{Fig:Various_Estimators}. The algorithm used to construct the dual basis-parameter estimator (green curve) is described in Section~\ref{subsec:Algorithm_Description}. As can be seen, the resulting nonlinearity closely follows the one obtained from the empirical MMSE estimator (red curve).

\medskip
\begin{center}
\includegraphics[width=0.45\textwidth]{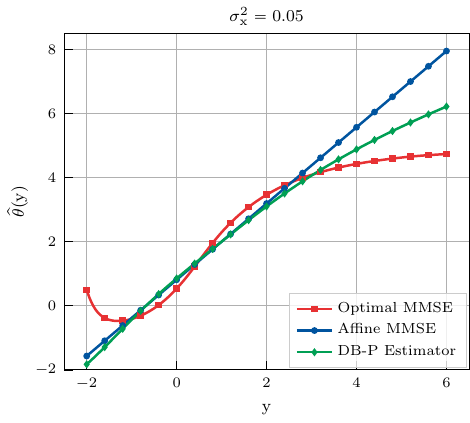}
\quad
\includegraphics[width=0.4585\textwidth]{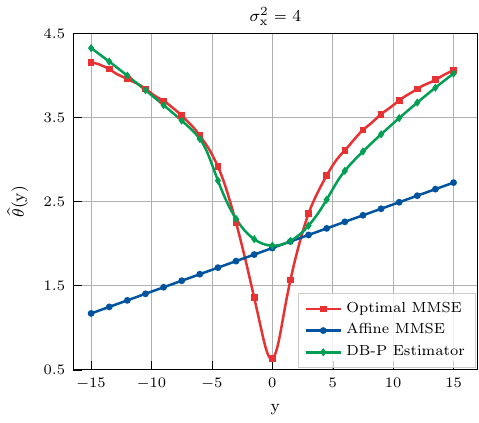}
\vspace{-0.5\baselineskip}
\captionof{figure}{\ref{eq:Dual_basis_parameter_estimator},
optimal MMSE~\eqref{eq:Exact_Bayesian_MMSE},
and affine~\eqref{eq:Parameters_Affine_Estimation} estimators.}
\label{Fig:Various_Estimators}
\end{center}
\end{tcolorbox}

In the above example, the basis function was linear, so estimating basis-function evaluations is equivalent to estimating the latent state. However, for nonlinear basis functions, this equivalence no longer holds, and the dual basis-parameter estimator can become both statistically and computationally inefficient. These limitations motivate an alternative, state-based DBS construction, in which the latent states are estimated directly, and the basis representation is treated as a derived quantity. 
In particular, the following remark provides a more detailed discussion of the main limitations of the dual basis-parameter estimation approach.

\begin{remark}[\textbf{Limitations of the dual basis-parameter estimator}] \label{remark:Limitations_dual_basis_parameter}
Two notable limitations of the dual basis-parameter estimator are
\begin{enumerate}[label=(\roman*), itemsep = 2mm, topsep = 1mm, leftmargin = 7.5mm]
\item \textbf{Elliptical-distribution approximation:} 
The affine MMSE estimators rely only on mean and covariance statistics and are exact when the joint distributions are Gaussian or, more generally, elliptical. In non-Gaussian settings, the performance of such estimators is limited to what one can achieve by approximating the true prior and posterior distributions of the unknown quantities by elliptical models, which may not hold in practice. In particular, the distribution of the stacked basis-function vector $\phi = \vectorize(\matrixPhi(\vecXbar))$ is generally unknown (e.g., when using Fourier basis functions), making such elliptical approximations potentially poor.

\item \textbf{High-dimensional basis representation:} 
In practice, the number of unknown basis functions, $(\Ntheta+1)(\Trajectory+1)$, typically exceeds the dimensionality of the state sequence $\vecXbar$, which is $(\nX)(\Trajectory+1)$. This high-dimensional representation increases both the statistical and computational complexity of the estimation problem.
\end{enumerate}
\end{remark}

Given these limitations, an alternative approach exploits the system dynamics in~\eqref{eq:Dynamic_Observation_Models} by first estimating the latent states and then computing DBS from these estimates statistics. This state-based DBS approach, developed in the next subsection, estimates the system states directly rather than the basis-function evaluations.
\subsection{Dual state-parameter estimation} \label{subsec:Dual_state_parameter_estimation}

In contrast to the previous approach, which directly estimates the basis collection as in~\eqref{eq:BasisCollection_Affine_Estimation}, this approach constructs DBS estimates from the state estimate sequence $\hat{\vecXbar}(\vecYbar) = [\hat{\vecX}_0^{\tr}(\vecYbar), \ldots, \hat{\vecX}_{\Trajectory}^{\tr}(\vecYbar)]^{\tr}\!$ and its corresponding covariance $\hatSigmaXbar{}(\vecYbar)$. This approach is particularly effective when the process noise is Gaussian, addressing a key limitation of the previous method discussed in Remark~\ref{remark:Limitations_dual_basis_parameter}.

\begin{assumption}[\textbf{Gaussian process noise}] \label{assumption:Gaussian_process_noise}
Throughout this subsection, the initial state and the process noise are drawn from a Gaussian distribution, i.e., $\vecX_{0} \sim \NormalDist(\MuX{0}, \SigmaX{0})$ and $\vecW_{t+1} \sim \NormalDist(\zero, \SigmaW{t+1})$.
\end{assumption}

Under this assumption, the state trajectory admits a Gaussian prior distribution. We adopt the \emph{lifted} representation of the process model~\eqref{eq:Dynamic_Observation_Models} as
\begin{equation} \label{eq:Lifted_Dynamic}
\vecXbar = \Abar (\Bbar \vecUbar + \vecWbar),
\end{equation}
where the input sequence $\vecUbar = [\MuX{0}^{\tr}, \vecU_0^{\tr}, \ldots, \vecU_{\Trajectory-1}^{\tr}]^{\tr}\!$ starts with the initial state mean, and the process noise sequence $\vecWbar = [\vecW_{0}^{\tr}, \vecW_1^{\tr}, \ldots, \vecW_{\Trajectory}^{\tr}]^{\tr}\!$ includes the initial state uncertainty, with $\vecW_{0} \sim \Pdist(\zero, \SigmaX{0})$. The lifted matrices $\Abar$ and $\Bbar$ have the following lower-triangular and block-diagonal structure, respectively:
\begin{equation} \label{eq:matrixAandB}
\Abar = \begin{bmatrix} \eye & \zeromx & \zeromx & \ldots & \zeromx \\ 
\A_0 & \eye & \zeromx & \ldots & \vdots \\ 
\A_1 \A_0 & \A_1 & \eye & \ddots & \vdots \\ 
\vdots & \vdots &  \vdots & \ddots & \zeromx \\ 
\prodOp_{i=0}^{\Trajectory-1} \A_{i} & \prodOp_{i=1}^{\Trajectory-1} \A_{i} &  \ldots & \A_{\Trajectory-1} & \eye
\end{bmatrix}, \qquad 
\Bbar = \, \diag(\eye, \B_{0}, \ldots, \B_{\Trajectory-1}),
\end{equation}
where $\prodOp_{i=t}^{t'} \A_{i} = \A_{t'}\ldots\A_{t}$. The process noise vector $\vecWbar \sim \NormalDist(\zero, \SigmaWbar{})$ follows a block-diagonal covariance structure under temporal independence, $\SigmaWbar{} = \diag(\SigmaX{0}, \SigmaW{1}, \ldots, \SigmaW{\Trajectory})$. Temporal correlations may also be incorporated if the corresponding covariance matrix is known. With the lifted representation~\eqref{eq:Lifted_Dynamic}, the state trajectory prior becomes $\vecXbar \sim \NormalDist(\Abar\Bbar\vecUbar, \Abar\SigmaWbar{}\Abar^{\!\tr}\!)$, which is Gaussian and enables a systematic computation of the DBS. Since a Gaussian distribution is fully characterized by its mean and covariance, the DBS computation in Definition~\ref{definition:Dynamic_Basis_Statistics_Collection} can be specialized to an operator requiring only these moments rather than the full distribution or higher moments. The following Gaussian DBS operator formalizes this specialization.

\begin{definition}[\textbf{Gaussian DBS operator}]
\label{definition:Gaussian_DBS_operator}
For a Gaussian state trajectory $\nu \sim \NormalDist(\mu_{\nu}, \Sigma_{\nu})$, the DBS can be computed from only the mean and covariance as
\begin{equation} \label{eq:Gaussian_DBS_operator}
(\matrixPhibarG, \SigmaPhiG) = \operatorGaussianDBS(\mu_{\nu}, \Sigma_{\nu}), 
\end{equation}
where $\operatorGaussianDBS \colon \Real^{\nX(\Trajectory+1)} \times \Real^{(\nX(\Trajectory+1))^{2}} \!\! \to \! \Real^{(\Ntheta+1)(\Trajectory+1)} \times \Real^{((\Ntheta+1)(\Trajectory+1))^{2}}$ maps the first two moments of the state trajectory to the basis statistics as in Definition~\ref{definition:Dynamic_Basis_Statistics_Collection}.
\end{definition}

Given the linearity of the basis function in Example~\ref{ex:1d_simple}, the Gaussian DBS operator reduces to the identity map. To highlight its nontrivial behaviour, we now consider another one-dimensional dynamical system in which the output function~\eqref{eq:Function_Space} contains a single sinusoidal basis function. In this setting, the Gaussian DBS operator transforms the mean and covariance of a single state into the exact statistics of the corresponding basis-function evaluation, illustrated in the following example.

\begin{tcolorbox}[colback=darkblue2!7, colframe=darkblue2!35, boxrule=0.5pt, arc=2pt, breakable]
\begin{examplepar}[(1-D Wiener model with sinusoidal basis function)] \label{ex:1D_nonlinear}
Consider the following sinusoidal nonlinearity in the measurement model:
\begin{equation}
\label{eq:1d_sin}
\vecY = \vecTheta \sin(f \vecX) + \vecV.
\end{equation}
where $\vecX \in \Real$ is a scalar random variable representing the initial state, with $\vecX \sim \Pdist(\MuX{}, \sigma_{\vecX}^{2})$, and $f$ is a known frequency. The Gaussian DBS operator $\operatorGaussianDBS(\cdot, \cdot)$ in Definition~\ref{definition:Gaussian_DBS_operator} maps the statistics of $\vecX$ to the following DBS, evaluated according to~\cite[Lem.~4.1, Ex.~1]{vakili2025optimal}:
\begin{equation}
\label{eq:GaussianDBS_Sin}
\begin{aligned}
\matrixPhibar & = \expS{-\dfrac{1}{2}f^{2} \sigma_{\vecX}^{2}}\sin(f \MuX{}), \\
\SigmaPhi & = \dfrac{1}{2} - \dfrac{1}{2}\expS{-2f^{2} \sigma_{\vecX}^{2}} \cos(2f\MuX{}) - \expS{-f^{2} \sigma_{\vecX}^{2}} \sin^{2}(f \MuX{}).
\end{aligned}
\end{equation}
\end{examplepar}
\end{tcolorbox}

For the lifted prior in~\eqref{eq:Lifted_Dynamic}, the Gaussian DBS operator provides an equivalent computation of the basis statistics from Definition~\ref{definition:Dynamic_Basis_Statistics_Collection} via $(\matrixPhibar, \SigmaPhi) = \operatorGaussianDBS(\Abar\Bbar\vecUbar, \Abar\SigmaWbar{}\Abar^{\!\tr}\!)$. Once state estimates $\hat{\vecXbar}(\vecYbar)$ and estimation-error covariance $\hatSigmaXbar{}(\vecYbar)$ are obtained, we apply a plug-in interpretation analogous to Remark~\ref{remark:Plugin_interpretation}: the estimated mean and covariance are treated as fixed values of a Gaussian prior distribution for the subsequent DBS computation, i.e., $\vecXbar \sim \NormalDist(\hat{\vecXbar}(\vecYbar), \hatSigmaXbar{}(\vecYbar))$, analogous to an empirical Bayes procedure applied to the state trajectory. This procedure naturally leads to data-driven DBS estimates via the Gaussian DBS operator,
\begin{equation} \label{eq:state_estimate_DBS}
\big(\hat{\matrixPhi}(\vecYbar), \hatSigmaPhi{}(\vecYbar)\big) = \operatorGaussianDBS(\hat{\vecXbar}(\vecYbar), \hatSigmaXbar{}(\vecYbar)).
\end{equation}
The affine MMSE state estimator is a natural choice for providing $\hat{\vecXbar}(\vecYbar)$ and $\hatSigmaXbar{}(\vecYbar)$, as it computes the posterior mean and covariance under a Gaussian approximation of the state trajectory posterior. The derivation of the affine MMSE state estimator as a direct consequence of Lemma~\ref{lemma:Generic_affine_MMSE_estimator} requires the cross-covariance $\Sigma_{\vecXbar\vecYbar}$ between the state trajectory and observation sequence. Under Assumption~\ref{assumption:Gaussian_process_noise}, this cross-covariance can be explicitly derived using Stein's identity~\cite{Stein1981}.

\begin{remark}[\textbf{Cross-covariance via Stein's identity}]
For a random vector $\nu \sim \NormalDist(\mu, \Sigma)$ and a differentiable function $\trueFunc$ with appropriate integrability conditions, Stein's identity states
\begin{equation} \label{eq:Stein_identity}
\Expectation{(\nu-\mu)\trueFunc(\nu)} = \Sigma \Expectation{\gradient_{\!\nu} \trueFunc(\nu)}.
\end{equation}
While this study applies Stein's identity under Gaussian noise, the identity extends to elliptical distributions~\cite{landsman2008stein}, enabling broader applicability to process noise models. Additionally, extensions to other distributions such as Poisson and $t$-distributions exist~\cite{liu1994siegel}, suggesting potential avenues for non-elliptical process noise distributions.
\end{remark}

The following lemma leverages this identity to derive a closed-form affine MMSE state estimator, which requires computing the mean of the Jacobian matrix of the basis-function evaluations with respect to the state distribution.

\begin{lemma} [\textbf{Affine MMSE state estimator}]
\label{lemma:MMSE_Affine_State_Estimator}
Under Assumption~\ref{assumption:Gaussian_process_noise}, suppose that for each basis function $\phi_{n}(\cdot)$ defined in~\eqref{eq:Function_Space}, the gradient vector $\gradient_{\!\vecX_{t}} \phi_{n}(\vecX_{t})$, for $t \in \{0, \ldots, \Trajectory\}$, has components that are continuous almost everywhere and satisfy the bounded expectation condition
\begin{equation} \label{eq:Bounded_expectations_conditions}
\gradient_{\!\vecX_{t}} \phi_{n}(\vecX_{t}) = \begin{bmatrix} \dfrac{\partial \phi_{n}(\vecX_{t})}{\partial \vecX_{t, 1}} \fracComma \cdots \fracComma \dfrac{\partial \phi_{n}(\vecX_{t})}{\partial \vecX_{t, \nX}} \end{bmatrix}^{\tr}\!\! \fracComma \quad
\BigExpectation{\big\lvert \dfrac{\partial \phi_{n}(\vecX_{t})}{\partial \vecX_{t, i}} \big\rvert} < \infty, \quad n = 0,\ldots,\Ntheta, \quad i = 1,\ldots,\nX,
\end{equation}
where $\vecX_{t, i}$ denotes the $i^{\nth}$ element of the state vector $\vecX_{t} \in \Real^{\nX}$. Then, the affine MMSE estimator for the state trajectory~\eqref{eq:Lifted_Dynamic}, its estimation-error covariance, and the optimal MSE, all as functions of $(\MuTheta{},\SigmaTheta{})$, are given by
\begin{subequations}\label{eq:opt_Bayes_States}
\begin{equation} \label{eq:MMSE_StateEstimation}
\begin{cases}
\hat{\vecXbar}_{\mathrm{af}}(\vecYbar ; \MuTheta{}, \SigmaTheta{}) =  \PsiX^{\star}(\MuTheta{},\SigmaTheta{})\vecYbar+\psiX^{\star}(\MuTheta{},\SigmaTheta{}),\\[0.2em]
\hatSigmaXbar{\mathrm{af}}\!(\MuTheta{}, \SigmaTheta{}) =  \Abar\SigmaWbar{}\Abar^{\!\tr}\! - \Abar\SigmaWbar{}\Abar^{\!\tr}\! \Cbar^{\tr}\!\mu_{\Theta}(\mu_{\Theta}^{\tr} \SigmaPhi \mu_{\Theta} + \operatorParam(\SigmaTheta{}) + \SigmaVbar{})^{\inv}\! \mu_{\Theta}^{\tr}\Cbar \Abar\SigmaWbar{}\Abar^{\!\tr}\!,\\[0.2em]
\Cost^{\star}_{\vecXbar}(\MuTheta{}, \SigmaTheta{}) = \trace(\hatSigmaXbar{\mathrm{af}}\!(\MuTheta{}, \SigmaTheta{})),
\end{cases}
\end{equation}
with the optimal coefficient matrices
\begin{equation} \label{eq:MMSE_State_LinearEstimator}
\PsiX^{\star}(\MuTheta{},\SigmaTheta{})=\Abar\SigmaWbar{}\Abar^{\!\tr}\! \Cbar^{\tr}\! \mu_{\Theta} (\mu_{\Theta}^{\tr} \SigmaPhi \mu_{\Theta} + \operatorParam(\SigmaTheta{}) + \SigmaVbar{})^{\inv}\!, \quad \psiX^{\star}(\MuTheta{},\SigmaTheta{}) = \Abar\Bbar\vecUbar - \PsiX^{\star}(\MuTheta{},\SigmaTheta{})\matrixPhibar^{\tr} \MuTheta{}, 
\end{equation}
where $(\matrixPhibar, \SigmaPhi) = \operatorGaussianDBS(\Abar\Bbar\vecUbar, \Abar\SigmaWbar{}\Abar^{\!\tr}\!)$ is the DBS for the Gaussian state trajectory prior from Definition~\ref{definition:Gaussian_DBS_operator}, $\mu_{\Theta} = \diag(\timedMuTheta{0}, \ldots, \timedMuTheta{\Trajectory})$ with $\timedMuTheta{t} = \MuTheta{}$ for $t \in \{0, \ldots, \Trajectory\}$, the linear operator $\operatorParam(\cdot)$ is defined in~\eqref{eq:MMSE_Basis_Covariance_Operator}, and
\begin{equation} \label{eq:Cbar_matrix}
\Cbar = \, \diag(\matrixC_{0}, \ldots, \matrixC_{\Trajectory}), \qquad \matrixC_{t} = \Expectation{\Jacobian_{\phi}(\vecX_{t})},
\end{equation}
with $\Jacobian_{\phi}(\vecX_{t}) = \big[ \gradient_{\!\vecX_{t}} \phi_{0}(\vecX_{t}), \ldots,\, \gradient_{\!\vecX_{t}} \phi_{\Ntheta}(\vecX_{t}) \big]^{\tr}\!$ being the Jacobian matrix of basis functions evaluated at $\vecX_{t}$.
\end{subequations}
\end{lemma}

A detailed proof is presented in Appendix~\ref{proof:MMSE_Affine_State_Estimator_proof}. The explicit computation and tractability of the affine MMSE state estimator depend on the choice of basis functions, since the closed-form expression~\eqref{eq:MMSE_StateEstimation} requires both the DBS $(\matrixPhibar, \SigmaPhi)$ and the expected Jacobian of the basis functions for constructing the matrix $\Cbar$ in~\eqref{eq:Cbar_matrix}. The sinusoidal basis of Example~\ref{ex:1D_nonlinear} satisfies the bounded expectation condition in~\eqref{eq:Bounded_expectations_conditions}, and the corresponding matrix $\matrixC$ admits a closed-form expression, as discussed below.

\begin{tcolorbox}[colback=darkblue2!7, colframe=darkblue2!35, boxrule=0.5pt, arc=2pt, breakable]
\paragraph{\textbf{Example~\ref{ex:1D_nonlinear} (continued)}}
In the context of~\eqref{eq:1d_sin}, the matrix $\matrixC$ is given by
\begin{equation}
\label{eq:matrixC_Sin}
\matrixC = \Expectation{\dfrac{\partial \phi(\vecX)}{\partial \vecX}} = \Expectation{f\cos(f\vecX)} = f \expS{-\dfrac{1}{2}f^{2} \sigma_{\vecX}^{2}}\cos(f \MuX{}),
\end{equation}
where $\Jacobian_{\phi}(\vecX)$ in~\eqref{eq:Cbar_matrix} reduces to a scalar derivative since $\vecX$ is a scalar variable. Given the explicit expressions for the DBS quantities in~\eqref{eq:GaussianDBS_Sin} and for $\matrixC$ in~\eqref{eq:matrixC_Sin}, the affine MMSE state estimator~\eqref{eq:MMSE_StateEstimation} yields
\begin{equation}
\label{eq:affine_state_1D}
\begin{aligned}
\hat{\vecXbar}_{\mathrm{af}}(\vecYbar ; \MuTheta{}, \SigmaTheta{}) & = \MuX{} + \dfrac{\sigma_{\vecX}^{2} \matrixC \MuTheta{}}{\MuTheta{}^{2} \SigmaPhi + \sigmaTheta{}^{2} \SigmaPhi + \sigmaTheta{}^{2} \matrixPhibar^{2} + \sigmaV{}^{2}}(\vecY - \matrixPhibar\MuTheta{}) \fracComma \\
\hatSigmaXbar{\mathrm{af}}\!(\MuTheta{}, \SigmaTheta{}) & = \sigma_{\vecX}^{2} - \dfrac{\sigma_{\vecX}^{4} \matrixC^{2} \MuTheta{}^{2}}{\MuTheta{}^{2} \SigmaPhi + \sigmaTheta{}^{2} \SigmaPhi + \sigmaTheta{}^{2} \matrixPhibar^{2} + \sigmaV{}^{2}} \fracDot
\end{aligned}
\end{equation}
Furthermore, the affine MMSE parameter and basis estimators for the setting of~\eqref{eq:1d_sin} follow their descriptions in Example~\ref{ex:1d_simple}, given in~\eqref{eq:affine_param_1D} and~\eqref{eq:affine_basis_1D}, respectively, after replacing the pair $(\MuX{}, \sigma_{\vecX}^{2})$ in those expressions with the DBS quantities obtained in~\eqref{eq:GaussianDBS_Sin}.
\end{tcolorbox}

In the more general Fourier-basis setting, both the DBS, $(\matrixPhibar, \SigmaPhi)$, and the matrix $\Cbar$ admit explicit expressions via the characteristic function of the Gaussian distribution~\cite[Lem.~4.1,~Ex.~1]{vakili2025optimal}, relying on the lifted representation~\eqref{eq:Lifted_Dynamic}. The following remark shows that the affine MMSE state estimates preserve this lifted structure, enabling the same explicit formulas to be applied in the plug-in construction~\eqref{eq:state_estimate_DBS}.

\begin{remark}[\textbf{Lifted plug-in prior for DBS estimates}]
\label{remark:Lifted_Plugin_Prior_DBS}
An equivalent representation of the affine MMSE state estimator in Lemma~\ref{lemma:MMSE_Affine_State_Estimator} is
$$
\hat{\vecXbar}_{\mathrm{af}}(\vecYbar ; \MuTheta{}, \SigmaTheta{}) = \Abar(\Bbar\vecUbar+\zeta), \qquad 
\hatSigmaXbar{\mathrm{af}}\!(\MuTheta{}, \SigmaTheta{}) = \Abar(\SigmaWbar{}-Z)\Abar^{\!\tr}\!,
$$
with measurement-driven corrections
$$
\begin{cases}
\zeta = \SigmaWbar{}\Abar^{\!\tr}\! \Cbar^{\tr}\! \mu_{\Theta} (\mu_{\Theta}^{\tr} \SigmaPhi \mu_{\Theta} + \operatorParam(\SigmaTheta{}) + \SigmaVbar{})^{\inv}\! (\vecYbar- \matrixPhibar^{\tr} \MuTheta{}),\\[0.2em]
Z = \SigmaWbar{}\Abar^{\!\tr}\! \Cbar^{\tr}\!\mu_{\Theta}(\mu_{\Theta}^{\tr} \SigmaPhi \mu_{\Theta} + \operatorParam(\SigmaTheta{}) + \SigmaVbar{})^{\inv}\! \mu_{\Theta}^{\tr}\Cbar \Abar\SigmaWbar{}.
\end{cases}
$$
Thus, the plug-in empirical Bayes Gaussian prior on $\vecXbar$ induced by the affine estimator can be written in lifted form as $\vecXbar \sim \NormalDist(\Abar (\Bbar \vecUbar + \zeta),\Abar(\SigmaWbar{}-Z)\Abar^{\!\tr}\!)$. For Fourier basis functions, this structure results in closed-form DBS estimates in~\eqref{eq:state_estimate_DBS} via the explicit characteristic-function formulas of~\cite[Ex.~1]{vakili2025optimal}.
\end{remark}

The affine state estimator in~\eqref{eq:MMSE_StateEstimation} is parameterized by the prior statistics $(\MuTheta{}, \SigmaTheta{})$. Since knowledge of $\vecTheta$ would improve these estimates, we adopt the empirical Bayes perspective of Remark~\ref{remark:Plugin_interpretation} and construct a nonlinear state estimator by inserting the parameter estimates $\big(\hat{\vecTheta}(\vecYbar), \hatSigmaTheta{}(\vecYbar)\big)$ as if they were the true prior moments,
\begin{equation} \label{eq:State_Nonlinear_Estimation}
\begin{cases}
\hat{\vecXbar}_{\mathrm{nl}}(\vecYbar) = \hat{\vecXbar}_{\mathrm{af}}(\vecYbar ; \hat{\vecTheta}(\vecYbar), \hatSigmaTheta{}(\vecYbar)), \\[0.2em]
\hatSigmaXbar{\mathrm{nl}}\!(\vecYbar) = \hatSigmaXbar{\mathrm{af}}\!(\hat{\vecTheta}(\vecYbar), \hatSigmaTheta{}(\vecYbar)).
\end{cases}
\end{equation}
Using~\eqref{eq:State_Nonlinear_Estimation} within the DBS estimates construction~\eqref{eq:state_estimate_DBS} for the nonlinear parameter estimator~\eqref{eq:Nonlinear_Parameter_Estimation}, and simultaneously feeding the resulting parameter estimates back into~\eqref{eq:State_Nonlinear_Estimation}, induces a circular dependence: the state estimates require the parameter estimates, while the parameter estimates depend on DBS computed from those same state estimates. This coupling gives rise to a nonlinear dual state-parameter estimator, defined as a solution of the resulting system of algebraic equations.

\begin{tcolorbox}[colback=gray!10, colframe=gray!50, boxrule=0.5pt, arc=2pt, breakable]
\textbf{Dual state-parameter estimator:} For any measurements vector~$\vecYbar$, there exists a tuple of state-parameter estimators $\big(\hat{\vecTheta}_{\mathrm{nl}}^{\star}, \hatSigmaTheta{\mathrm{nl}}^{\star}, \hat{\vecXbar}_{\mathrm{nl}}^{\star}, \hatSigmaXbar{\mathrm{nl}}^{\star}\big)$ such that
\begin{equation*} \label{eq:Dual_state_parameter_estimator}\tag{DS-P}
\begin{cases}
\hat{\vecTheta}_{\mathrm{nl}}^{\star} = \hat{\vecTheta}_{\mathrm{af}}\big(\vecYbar ; \operatorGaussianDBS(\hat{\vecXbar}_{\mathrm{nl}}^{\star}, \hatSigmaXbar{\mathrm{nl}}^{\star})\big), \\[0.2em]
\hatSigmaTheta{\mathrm{nl}}^{\star} = \hatSigmaTheta{\mathrm{af}}\!\big(\operatorGaussianDBS(\hat{\vecXbar}_{\mathrm{nl}}^{\star}, \hatSigmaXbar{\mathrm{nl}}^{\star})\big),
\end{cases}
\quad
\begin{cases}
\hat{\vecXbar}_{\mathrm{nl}}^{\star} = \hat{\vecXbar}_{\mathrm{af}}\big(\vecYbar ; \hat{\vecTheta}_{\mathrm{nl}}^{\star}, \hatSigmaTheta{\mathrm{nl}}^{\star}\big) \\[0.2em]
\hatSigmaXbar{\mathrm{nl}}^{\star} = \hatSigmaXbar{\mathrm{af}}\!\big(\hat{\vecTheta}_{\mathrm{nl}}^{\star}, \hatSigmaTheta{\mathrm{nl}}^{\star}\big)
\end{cases}
\end{equation*}
where $\operatorGaussianDBS(\hat{\vecXbar}_{\mathrm{nl}}^{\star}, \hatSigmaXbar{\mathrm{nl}}^{\star})$ returns the DBS estimates $\big(\hat{\matrixPhi}(\vecYbar), \hatSigmaPhi{}(\vecYbar)\big)$ according to~\eqref{eq:state_estimate_DBS} and Definition~\ref{definition:Gaussian_DBS_operator}, and the functions~$\big(\hat{\vecTheta}_{\mathrm{af}}(\cdot),\hatSigmaTheta{\mathrm{af}}\!(\cdot)\big)$ and $\big(\hat{\vecXbar}_{\mathrm{af}}(\cdot),\hatSigmaXbar{\mathrm{af}}\!(\cdot)\big)$ are the closed form solutions of the affine Bayesian MMSE defined in \eqref{eq:Parameters_Affine_Estimation} and \eqref{eq:MMSE_StateEstimation}, respectively.
\end{tcolorbox}

The dual state-parameter~\eqref{eq:Dual_state_parameter_estimator} estimation framework likewise couples two affine MMSE estimators, one for the parameters and one for the states, each depending on the estimates produced by the other. Although this approach addresses the limitations of \ref{eq:Dual_basis_parameter_estimator} estimator raised in Remark~\ref{remark:Limitations_dual_basis_parameter}, the asymptotic performance of the parameter estimates with infinitely many samples depends critically on the properties of the underlying dynamical system~\cite[Prop.~4.2]{vakili2025optimal}.

\begin{remark}[\textbf{Asymptotic performance}] \label{remark:asymptotic_performance}
While the state-estimation step in~\ref{eq:Dual_state_parameter_estimator} estimator reduces the effective process-noise uncertainty, consistency of the nonlinear parameter estimator requires that the state-estimation error remain suitably bounded over the trajectory, playing a role analogous to the multiple-trajectory condition in~\cite[Prop.~4.3]{vakili2025optimal}. Characterizing conditions under which the state-estimation error is bounded and the nonlinear parameter estimator is consistent is an interesting direction for future research. Nevertheless, for finite data, the numerical experiments in Section~\ref{sec:Numerical_Experiments} demonstrate that the proposed dual state-parameter estimator can substantially outperform existing affine and nonlinear alternatives when Fourier bases are employed.
\end{remark}

The mutual dependence between the estimator characterizations in~\ref{eq:Dual_basis_parameter_estimator} and~\ref{eq:Dual_state_parameter_estimator} naturally suggests viewing the nonlinear parameter estimator architecture (cf.~Figure~\ref{Fig:Nonlinear_Parameter_Estimator_Architecture}) as a fixed point of suitable mappings. This perspective motivates computing the estimator via an iterative fixed-point scheme, in which the coupled estimators are applied repeatedly until convergence. From this viewpoint, each iteration can be interpreted as propagating information between the DBS and parameter estimator blocks, gradually reconciling their respective posterior summaries. Once the iterative scheme has converged, the resulting parameter estimates generally exhibit a highly nonlinear dependence on the measurement value. Whether the proposed estimators produce estimates close to the optimal affine MMSE parameter estimator or deviate significantly from it depends sensitively on the measurement region and on the prior distribution of the latent system states. To illustrate this nonlinear behaviour in a concrete setting, we now revisit Example~\ref{ex:1D_nonlinear} and examine the~\ref{eq:Dual_basis_parameter_estimator} and~\ref{eq:Dual_state_parameter_estimator} estimators for different measurement values, while the discussion of the algorithmic aspects of the fixed-point characterization is deferred to the next subsection.

\begin{tcolorbox}[colback=darkblue2!7, colframe=darkblue2!35, boxrule=0.5pt, arc=2pt, breakable]
\paragraph{\textbf{Example~\ref{ex:1D_nonlinear} (continued)}}
For the parameter-estimation simulations, consider the observation model~\eqref{eq:1d_sin} with sinusoidal frequency $f = \frac{\pi}{6}$, a scalar latent state $\vecX \sim \NormalDist(\MuX{}, \sigma_{\vecX}^{2})$ (interpreted as the initial state of a dynamical system) with $\MuX{} = 0.5$ and two choices of variance, $\sigma_{\vecX}^{2} = 0.05$ (left plot) and $\sigma_{\vecX}^{2} = 4$ (right plot), an unknown scalar parameter $\vecTheta \in \Real$ with prior $\vecTheta \sim \UniformDist(-1, 5)$, and measurement noise $\vecV \sim \NormalDist(0, \sigmaV{}^{2})$ with $\sigmaV{}^{2} = 0.16$. Figure~\ref{Fig:Various_Estimators_Example2} shows the estimates of the unknown parameter $\vecTheta$ as a function of the measurement value, i.e., $\hat{\vecTheta}(\vecY)$, for four estimators: the \ref{eq:Dual_state_parameter_estimator} estimator (orange), the \ref{eq:Dual_basis_parameter_estimator} estimator (green), the empirical MMSE estimator computed via MCMC methods (red), and the affine MMSE estimator obtained from~\eqref{eq:Parameters_Affine_Estimation} (blue). As anticipated from the fixed-point perspective, the dual estimators exhibit notable nonlinearities relative to the affine MMSE estimator. The algorithms used to construct these dual estimators are described in the next subsection.

\medskip
\begin{center}
\includegraphics[width=0.452\textwidth]{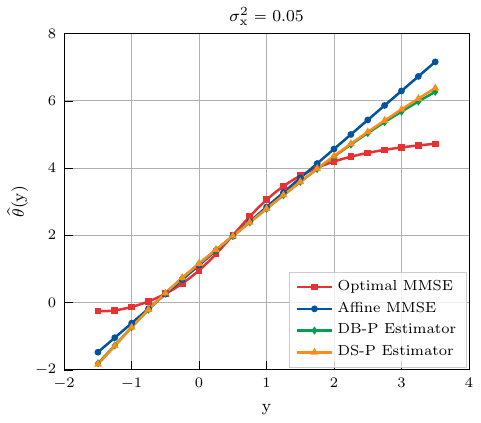}
\quad
\includegraphics[width=0.45\textwidth]{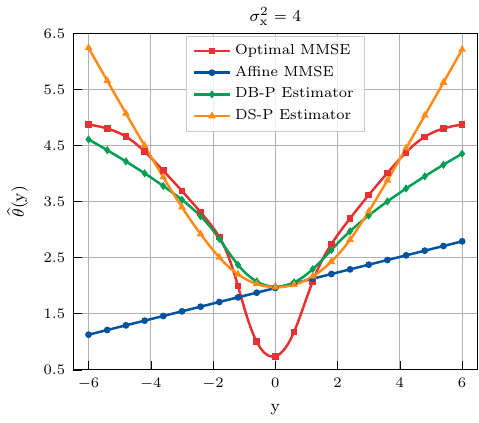}
\vspace{-0.5\baselineskip}
\captionof{figure}{\ref{eq:Dual_state_parameter_estimator}, \ref{eq:Dual_basis_parameter_estimator}, optimal MMSE~\eqref{eq:Exact_Bayesian_MMSE},
and affine MMSE~\eqref{eq:Parameters_Affine_Estimation} estimators
}
\label{Fig:Various_Estimators_Example2}
\end{center}
\end{tcolorbox}
\subsection{Fixed-point algorithm}
\label{subsec:Algorithm_Description}
The dual basis-parameter~\eqref{eq:Dual_basis_parameter_estimator} and dual state-parameter~\eqref{eq:Dual_state_parameter_estimator} estimators can be equivalently characterized by a fixed-point solution $\zeta^{\star}$ of an operator $\mathfrak{F}$, satisfying $\zeta^{\star} = \mathfrak{F}(\zeta^{\star})$. We collect the estimates and their covariances into the tuple: $\zeta \!=\! \big(\hat{\vecTheta}_{\mathrm{nl}}, \hatSigmaTheta{\mathrm{nl}}, \hat{\matrixPhi}_{\mathrm{nl}}, \hatSigmaPhi{\mathrm{nl}}\big)$ for~\ref{eq:Dual_basis_parameter_estimator} and $\zeta \!=\!  \big(\hat{\vecTheta}_{\mathrm{nl}}, \hatSigmaTheta{\mathrm{nl}}, \hat{\vecXbar}_{\mathrm{nl}}, \hatSigmaXbar{\mathrm{nl}}\big)$ for~\ref{eq:Dual_state_parameter_estimator}. For a given measurement sequence $\vecYbar$, the corresponding fixed-point operators are defined as
$$
\begin{cases}
\cmdb{\textbf{\ref{eq:Dual_basis_parameter_estimator}}:} \quad & \mathfrak{F}(\zeta) \coloneqq \big(
\hat{\vecTheta}_{\mathrm{af}}(\vecYbar ; \hat{\matrixPhi}_{\mathrm{nl}}, \hatSigmaPhi{\mathrm{nl}}),\, 
\hatSigmaTheta{\mathrm{af}}\!(\hat{\matrixPhi}_{\mathrm{nl}}, \hatSigmaPhi{\mathrm{nl}}),\, 
\hat{\matrixPhi}_{\mathrm{af}}(\vecYbar ; \hat{\vecTheta}_{\mathrm{nl}}, \hatSigmaTheta{\mathrm{nl}}),\, 
\hatSigmaPhi{\mathrm{af}}\!(\hat{\vecTheta}_{\mathrm{nl}}, \hatSigmaTheta{\mathrm{nl}})
\big), \\[0.2em]
\cmdb{\textbf{\ref{eq:Dual_state_parameter_estimator}}\mkern3.2mu:} \quad & \mathfrak{F}(\zeta) \coloneqq \big(
 \hat{\vecTheta}_{\mathrm{af}}\big(\vecYbar ; \operatorGaussianDBS(\hat{\vecXbar}_{\mathrm{nl}}, \hatSigmaXbar{\mathrm{nl}})\big),\, 
\hatSigmaTheta{\mathrm{af}}\!\big(\operatorGaussianDBS(\hat{\vecXbar}_{\mathrm{nl}}, \hatSigmaXbar{\mathrm{nl}})\big),\, 
\hat{\vecXbar}_{\mathrm{af}}(\vecYbar ; \hat{\vecTheta}_{\mathrm{nl}}, \hatSigmaTheta{\mathrm{nl}}),\, 
\hatSigmaXbar{\mathrm{af}}\!(\hat{\vecTheta}_{\mathrm{nl}}, \hatSigmaTheta{\mathrm{nl}})
\big),
\end{cases}
$$
respectively. A natural way to compute such a fixed point is via the simultaneous (Jacobi-style) iteration $\zeta^{(k+1)} = \mathfrak{F}(\zeta^{(k)})$, whose pseudo-code for both variants is summarized in Algorithm~\ref{algo:Dual_Estimators}. To improve the stability of the iterations, it is sometimes recommended to use an averaged update with $0 < \alpha \leq 1$, at the cost of a slower convergence rate,
\begin{equation}
\label{eq:averaged_update}
\zeta^{(k+1)} = \zeta^{(k)} + \alpha \big( \mathfrak{F}(\zeta^{(k)}) - \zeta^{(k)} \big),
\end{equation}
where $\alpha = 1$ corresponds to the Jacobi-style iteration, used in Algorithm~\ref{algo:Dual_Estimators}. Alternating updates (Gauss–Seidel-style), where the components of $\zeta$ are updated sequentially rather than simultaneously, are also common for coupled fixed-point problems and may improve convergence~\cite[Ch.~7]{Burden2010}.

\begin{remark}[\textbf{Convergence condition}]
\label{remark:Convergence_Conditions}

The convergence of the sequence of iterates $\{\zeta^{(k)}\}_{k\geq 0}$ in~\eqref{eq:averaged_update} depends on the properties of the operator $\mathfrak{F}$. A sufficient condition is that $\mathfrak{F}$ is nonexpansive~\cite[Thm.~5.14]{BauschkeCombettes2017}. A rigorous analysis of the convergence of these iterations and the properties of $\mathfrak{F}$ for both dual estimators is left for future work; however, the empirical results provide encouraging evidence supporting this behaviour.
\end{remark}

Building on this fixed-point characterization, the iterative scheme in Algorithm~\ref{algo:Dual_Estimators} for computing the dual estimators initializes the means and covariances of the parameters and DBS-related quantities. At iteration $k=1$, the estimates in the tuple $\zeta^{(1)}$ are obtained from a single pass of the Bayesian MMSE affine estimators for the unknown parameters together with the DBS or system states, depending on the chosen variant. In the \ref{eq:Dual_basis_parameter_estimator} variant, the associated DBS is computed once from the prior of the system states $\vecXbar$ according to Definition~\ref{definition:Dynamic_Basis_Statistics_Collection}. In contrast, in the \ref{eq:Dual_state_parameter_estimator} variant, the DBS is recomputed at every iteration by applying the Gaussian DBS operator $\operatorGaussianDBS(\cdot, \cdot)$ to the current state mean and covariance estimates. In this variant, however, the matrix $\Cbar$ in~\eqref{eq:Cbar_matrix} is evaluated only once, since it depends solely on the prior of the state $\vecXbar$, which does not change across iterations. The stopping criterion can be chosen based on the maximum number of iterations or on the deviation between subsequent iterates. In our implementation, we use a threshold $\epsilon$ on the improvement in the parameter MSE cost $\Cost_{\vecTheta}(\cdot, \cdot)$ in~\eqref{eq:Parameters_Affine_Estimation} between two iterations as the exit condition, i.e., $\big| \Cost^{(k)}_{\vecTheta} - \Cost^{(k-1)}_{\vecTheta} \big| < \epsilon$. Both algorithm variants, \ref{eq:Dual_basis_parameter_estimator} and \ref{eq:Dual_state_parameter_estimator}, converged to a fixed-point solution in almost all extensive numerical experiments reported in Section~\ref{sec:Numerical_Experiments}, with the \ref{eq:Dual_state_parameter_estimator} variant exhibiting rapid convergence. Moreover, both variants have a per-iteration computational complexity of $\mathcal{O}(\Trajectory^{3})$ when $\Ntheta, \nX \ll \Trajectory$, dominated by the covariance matrix inversions in the affine MMSE updates.

\begin{algorithm}[htbp]
\caption{Dual estimator: basis-parameter (\cmdb{\textbf{\ref{eq:Dual_basis_parameter_estimator}}}) or state-parameter (\cmdb{\textbf{\ref{eq:Dual_state_parameter_estimator}}})}
\label{algo:Dual_Estimators}
\begin{algorithmic}[1]
\Variant Choose either \cmdb{\textbf{\ref{eq:Dual_basis_parameter_estimator}}} or \cmdb{\textbf{\ref{eq:Dual_state_parameter_estimator}}}

\Require Tolerance $\epsilon$, maximum iterations $K$ 

\Ensure $(\hat{\vecTheta}_{\mathrm{nl}}, \hatSigmaTheta{\mathrm{nl}})$ and \cmdb{\textbf{\ref{eq:Dual_basis_parameter_estimator}}:} $(\hat{\matrixPhi}_{\mathrm{nl}}, \hatSigmaPhi{\mathrm{nl}})$ or \cmdb{\textbf{\ref{eq:Dual_state_parameter_estimator}}:} $(\hat{\vecXbar}_{\mathrm{nl}}, \hatSigmaXbar{\mathrm{nl}})$ 

\vspace{0.2em}
\State $(\hat{\vecTheta}_{\mathrm{nl}}^{\mkern2mu (0)}, \hatSigmaTheta{\mathrm{nl}}^{(0)}) \gets (\MuTheta{}, \SigmaTheta{})$ 
\Comment{\small $\vecTheta$ prior information} 

\vspace{0.2em}
\State $\Cost^{(0)}_{\vecTheta} \gets 0$
\Comment{\small initial cost}

\vspace{0.2em}
\State $(\hat{\matrixPhi}_{\mathrm{nl}}^{(0)}, \hatSigmaPhi{\mathrm{nl}}^{(0)}) \gets (\matrixPhibar, \SigmaPhi)$ 
\Comment{\small DBS, Definition~\ref{definition:Dynamic_Basis_Statistics_Collection}} 

\vspace{0.2em}
\For{$k = 1, 2, \ldots, K$} 
    
    \vspace{0.2em}
    \State $(\hat{\vecTheta}_{\mathrm{nl}}^{\mkern2mu (k)}, \hatSigmaTheta{\mathrm{nl}}^{(k)}) \gets \big(\hat{\vecTheta}_{\mathrm{af}}(\vecYbar ; \hat{\matrixPhi}_{\mathrm{nl}}^{(k-1)}, \hatSigmaPhi{\mathrm{nl}}^{(k-1)}), \hatSigmaTheta{\mathrm{af}}\!(\hat{\matrixPhi}_{\mathrm{nl}}^{(k-1)}, \hatSigmaPhi{\mathrm{nl}}^{(k-1)})\big)$ 
    \Comment{\small affine MMSE,~\eqref{eq:Parameters_Affine_Estimation}} 
    
    \vspace{0.2em}
    \State $\Cost^{(k)}_{\vecTheta} \gets \Cost^{\star}_{\vecTheta}(\hat{\matrixPhi}_{\mathrm{nl}}^{(k-1)}, \hatSigmaPhi{\mathrm{nl}}^{(k-1)})$ 
    \Comment{\small associated cost}
    
    \vspace{0.2em}
    \State \cmdb{\textbf{\ref{eq:Dual_basis_parameter_estimator}}:} $(\hat{\matrixPhi}_{\mathrm{nl}}^{(k)}, \hatSigmaPhi{\mathrm{nl}}^{(k)}) \gets \big(\hat{\matrixPhi}_{\mathrm{af}}(\vecYbar ; \hat{\vecTheta}_{\mathrm{nl}}^{\mkern2mu (k-1)}, \hatSigmaTheta{\mathrm{nl}}^{(k-1)}), \hatSigmaPhi{\mathrm{af}}\!(\hat{\vecTheta}_{\mathrm{nl}}^{\mkern2mu (k-1)}, \hatSigmaTheta{\mathrm{nl}}^{(k-1)})\big)$ 
    \Comment{\small affine MMSE,~\eqref{eq:BasisCollection_Affine_Estimation}} 
    
    \vspace{0.2em}    
    \State \cmdb{\textbf{\ref{eq:Dual_state_parameter_estimator}}:} $(\hat{\vecXbar}_{\mathrm{nl}}^{(k)}, \hatSigmaXbar{\mathrm{nl}}^{(k)}) \gets \big(\hat{\vecXbar}_{\mathrm{af}}(\vecYbar ; \hat{\vecTheta}_{\mathrm{nl}}^{\mkern2mu (k-1)}, \hatSigmaTheta{\mathrm{nl}}^{(k-1)}), \hatSigmaXbar{\mathrm{af}}\!(\hat{\vecTheta}_{\mathrm{nl}}^{\mkern2mu (k-1)}, \hatSigmaTheta{\mathrm{nl}}^{(k-1)})\big)$ 
    \Comment{\small affine MMSE,~\eqref{eq:MMSE_StateEstimation}} 
    
    \vspace{0.2em}
    \State \cmdb{\textbf{\ref{eq:Dual_state_parameter_estimator}}:} $(\hat{\matrixPhi}_{\mathrm{nl}}^{(k)}, \hatSigmaPhi{\mathrm{nl}}^{(k)}) \gets \operatorGaussianDBS(\hat{\vecXbar}_{\mathrm{nl}}^{(k)}, \hatSigmaXbar{\mathrm{nl}}^{(k)})$ 
    \Comment{\small Gaussian DBS, Definition~\ref{definition:Gaussian_DBS_operator}} 
    
    \vspace{0.2em}    
    \If{$\big| \Cost^{(k)}_{\vecTheta} - \Cost^{(k-1)}_{\vecTheta} \big| < \epsilon$} 
    \Comment{\small stopping criterion}
        \State \textbf{break} 
    \EndIf

\EndFor
\end{algorithmic}
\end{algorithm}
\section{Numerical Experiments} \label{sec:Numerical_Experiments}
In this section, we evaluate the proposed nonlinear estimators, dual basis-parameter (\ref{eq:Dual_basis_parameter_estimator}) and dual state-parameter (\ref{eq:Dual_state_parameter_estimator}), against the affine MMSE parameter estimator (Theorem~\ref{theorem:MMSE_Affine_Parameter_Estimator}), hereafter called A-MMSE, across two setups: $(\Ntheta = 2, \nX = 10)$ with more states than parameters, and $(\nX = 2, \Ntheta = 10)$ with more parameters than states. We then benchmark the dual state-parameter estimator (\ref{eq:Dual_state_parameter_estimator}) against two established sampling-based methods: Particle Gibbs with Ancestor Sampling (PGAS)~\cite{lindsten2014particle} and SMC-EM~\cite{schon_system_2011}. Both methods perform SMC-based smoothing at each iteration to draw from the state posterior. SMC methods suffer from particle degeneracy, which necessitates an exponential growth in the number of particles with state dimension~\cite{bengtsson2008curse}. While modern methods such as Nested Sequential Monte Carlo (NSMC)~\cite{naesseth2015nested} can improve SMC scalability in high dimensions, we restrict the benchmark comparisons to the second setup $(\nX = 2, \Ntheta = 10)$ for two reasons: 
(i)~The first setup primarily illustrates a regime where the dual basis-parameter estimator underperforms the affine MMSE estimator, demonstrating the limitations of the elliptical-distribution approximation discussed in Remark~\ref{remark:Limitations_dual_basis_parameter}; 
(ii)~Adapting PGAS and SMC-EM to use NSMC-based smoothing would require substantial additional development and tuning, beyond the scope of a fair performance comparison. 
Focusing on the lower-dimensional state case $(\nX = 2, \Ntheta = 10)$ mitigates the impact of state dimensionality on SMC-based methods and enables a more meaningful and interpretable comparison. We use a marginally stable linear time-invariant (LTI) dynamical system ($\A_{t} = \eye$) with a true function in the Fourier subspace
\begin{equation} \label{eq:Fourier_Basis}
\begin{cases}
\phi_{0}(\vecX) = 1 & n = 0 \\
\phi_{n}(\vecX) = \binarySum{\ell\in\{-1,1\}}{\expS{j\innerS{\ell f_{n}}{\vecX}}} & n \geq 1,
\end{cases}
\end{equation}
where $f_{n} \in \Real^{\nX}$ represents a \emph{known} frequency. 
This Fourier basis choice enables explicit computation of the DBS and cross-covariance terms, and it allows us to examine how increasing state uncertainty affects estimator performance as the number of measurements grows. The experiments use time-invariant Gaussian noise: $\vecV_{t} \sim \NormalDist(\zero, \sigma^{2}_{\vecV}\eye)$, $\vecW_{t+1}\sim \NormalDist(\zero, \sigma^{2}_{\vecW}\eye)$, and $\vecX_0 \sim \NormalDist(\MuX{0}, \sigma^{2}_{\vecX_{0}}\eye)$, with $\sigma^{2}_{\vecX_{0}} = \sigma^{2}_{\vecW}$. To assess the impact of process noise, we vary $\sigma^{2}_{\vecW} \in \{0.001,\ 0.01\}$ while holding the measurement-noise variance fixed at $\sigmaV{}^{2} = 0.01$ across all experiments. The input trajectories are generated using the active-learning procedure from~\cite[Sec.~5]{vakili2025optimal}, conducted prior to the experiments and optimized for the affine MMSE estimator. These same inputs are used for all methods in the study to ensure a fair comparison. For reproducibility, a MATLAB library implementing the affine MMSE estimator with optimal input design, as well as the dual basis-parameter and dual state-parameter estimators, is available at \href{https://github.com/sasanvakili/nonlinearBayesian4Wiener}{https://github.com/sasanvakili/nonlinearBayesian4Wiener}.

For each configuration, we generate $100$ independent samples of $\vecTheta$ and $100$ independent samples of $(\vecWbar, \vecVbar)$ for each trajectory length $\Trajectory$, resulting in $10,\!000$ experiments in total. We evaluate estimators using the squared-error criterion $||\vecTheta - \hat{\vecTheta}(\vecYbar)||^{2}$ and ensure fair comparison by using identical realizations of $\vecWbar$ and $\vecVbar$ across all methods. In the provided plots, dashed lines show the empirically computed MSE, $\Expectation{||\vecTheta - \hat{\vecTheta}(\vecYbar)||^{2}}$, for each method. Non-histogram plots include shaded regions indicating the $15^{\nth}$-$85^{\nth}$ percentile range of the squared error. Histograms display the probability density of $\log||\vecTheta - \hat{\vecTheta}(\vecYbar)||^{2}$ on a logarithmic vertical scale to resolve differences in the squared-error distributions. All histograms within each plot use the same bin width, and each histogram integrates to unity.
\paragraph{\textbf{Experiment setup~1} ($\boldsymbol{\Ntheta=2, \nX=10}$)}
This setup considers an LTI system with $\nX = 10$ states and $\nU = 2$ control inputs. The system dynamics are $\A_{t} = \eye$ and $\B_{t} = \Delta t [\eye, \ldots, \eye]^{\tr}\!$, where $\Delta t = 0.1$ is the sampling interval and the identity matrices have dimensions conformable with $\nX$ and $\nU$. The initial state is drawn from a standard normal distribution, $\MuX{0} \sim \NormalDist(0, \eye)$ and the input sequence $\vecU_{t} = 4.5\!\sum_{\upsilon \in \Upsilon} [\cos(\upsilon t\Delta t),\ \sin(\upsilon t\Delta t)]^{\tr}\!$, where $\Upsilon \in \{3, 5, 10, 20, 100\}$, is used to initialize the projected steepest descent algorithm in the active-learning procedure of~\cite[Sec.~5]{vakili2025optimal}, with each component constrained to $[-200, 200]$. The optimal input $\vecUbar^{\star}$ is precomputed once and used for all methods evaluated under this configuration. The unknown parameters follow the Fourier basis representation in~\eqref{eq:Fourier_Basis}. We use three unknown parameters $\vecTheta_{n}$ ($n = 0, 1, 2$) with known frequency vectors $f_n \in \Real^{10}$. Specifically, $f_{0} = 0$ (the constant basis), and the nonzero frequencies are drawn independently as $f_{n} \sim \NormalDist(0, \eye)$ for $n \in \{1, 2\}$. The prior over the unknown parameters is uniform, $\vecTheta_{n} \sim \UniformDist(2, 8)$ for each $n$, implying prior moments $\MuTheta{n} = 5$ and $\sigmaTheta{n}^{2} = 3$. True parameter values are generated from this same prior.
\paragraph{\textbf{Benchmark~1.1} (\textbf{High-dimensional states})}
This benchmark uses Experiment setup~1 with trajectory length $\Trajectory = 100$ (i.e., $101$ measurements) and two process-noise variances $\sigma^{2}_{\vecW} \in \{0.001,\, 0.01\}$, with $10,\!000$ simulations per configuration. We compute parameter estimates $\hat{\vecTheta}(\vecYbar)$ from A-MMSE~\eqref{eq:Parameters_Affine_Estimation} and the dual estimators \ref{eq:Dual_basis_parameter_estimator} and \ref{eq:Dual_state_parameter_estimator}. Figure~\ref{Fig:basisParam_Error_Setup1} shows histograms of $\log \|\vecTheta - \hat{\vecTheta}(\vecYbar)\|^{2}$ over all simulations, with dashed vertical lines indicating the empirical MSE: orange for \ref{eq:Dual_basis_parameter_estimator}, blue for A-MMSE, and red for \ref{eq:Dual_state_parameter_estimator}. Although one might expect \ref{eq:Dual_basis_parameter_estimator} to perform best, given that only two unknown basis variables and a scalar observation $\vecY_{t}$ are involved at each time, this setup provides a counterexample. In both noise regimes, \ref{eq:Dual_basis_parameter_estimator} leads to a higher empirical MSE than A-MMSE. This underperformance is consistent with the limitations of the elliptical prior approximation used in the basis estimation (cf.~Remark~\ref{remark:Limitations_dual_basis_parameter}). With a scalar measurement at each time and a high state dimension ($\nX = 10$), the latent states remain substantially uncertain, limiting the benefit of dual state-parameter estimation relative to the affine approach. Nevertheless, \ref{eq:Dual_state_parameter_estimator} consistently achieves the lowest empirical MSE (red), demonstrating that joint estimation of latent states and parameters improves parameter inference on average when the latent state process is Gaussian (cf.~Assumption~\ref{assumption:Gaussian_process_noise}).
\begin{figure}[htbp]
\centering 
\includegraphics[width=0.4523\linewidth]{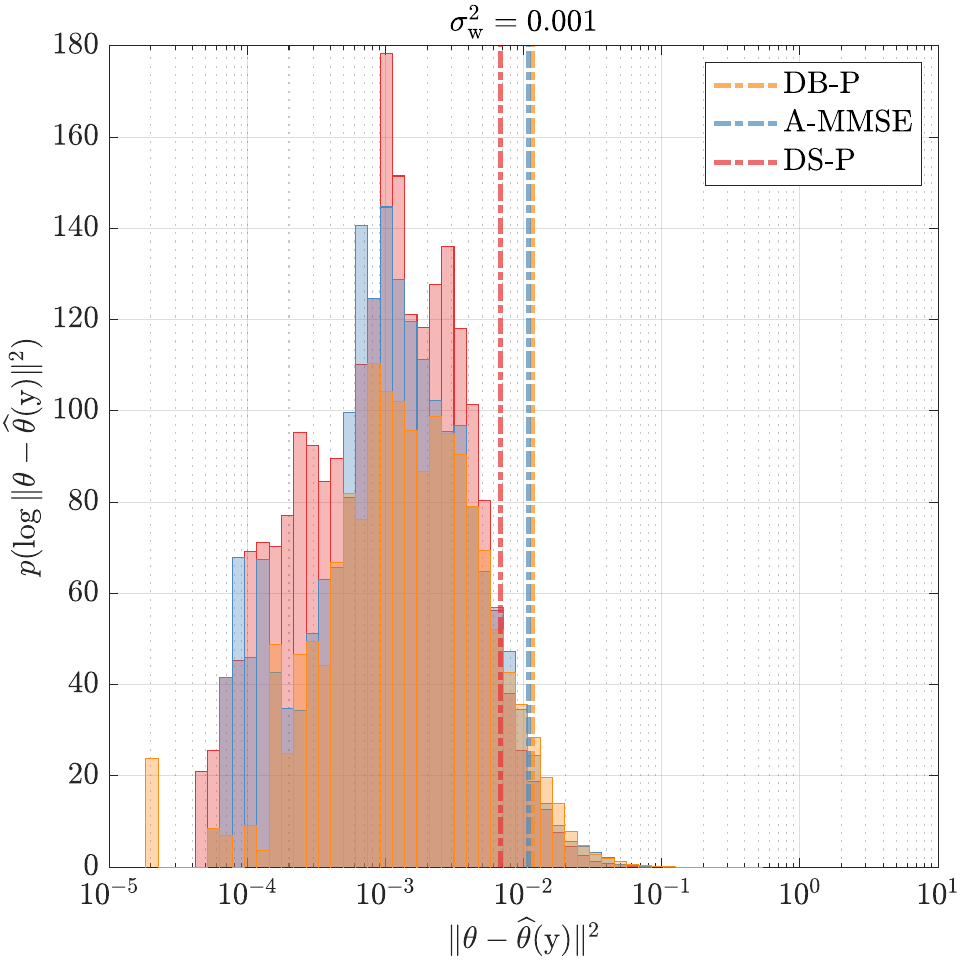} \qquad 
\includegraphics[width=0.4523\linewidth]{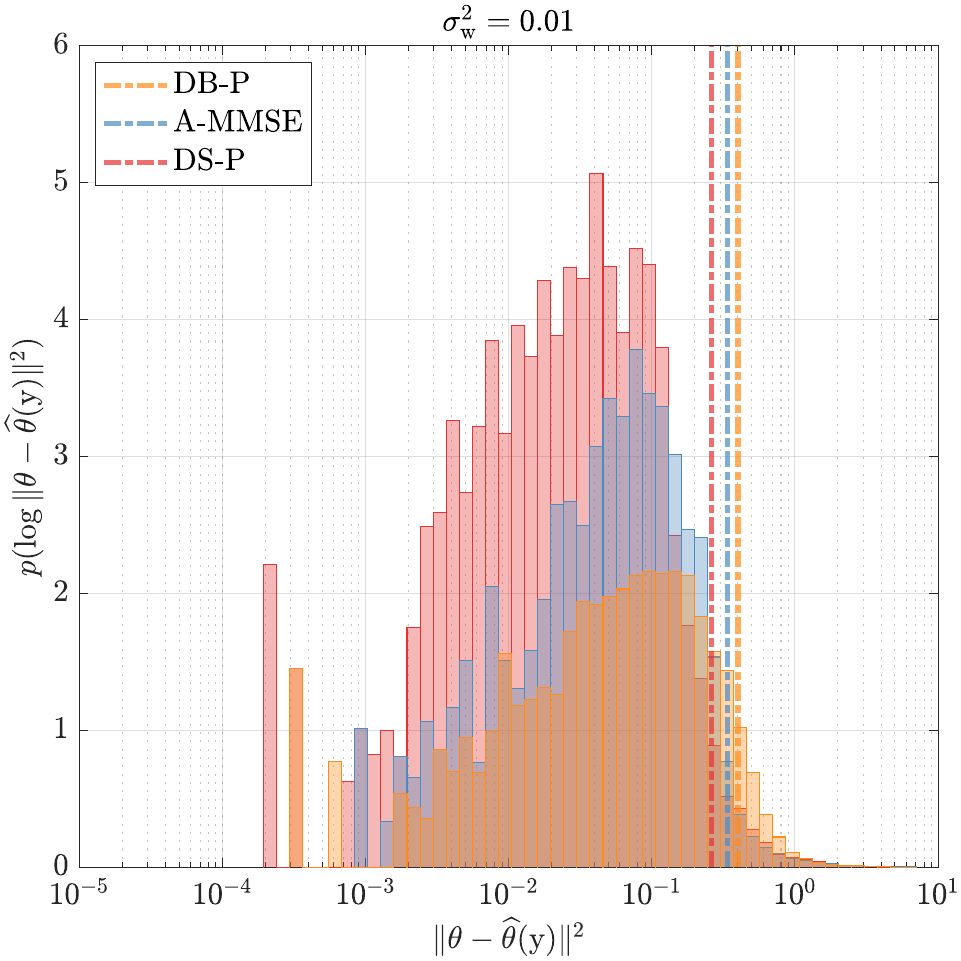}
\caption{Parameter estimation squared-error distributions for Experiment setup~1}
\label{Fig:basisParam_Error_Setup1}
\end{figure}
\paragraph{\bf Experiment setup~2 ($\boldsymbol{\Ntheta=10, \nX=2}$)}
This setup focuses on the complementary configuration of low-dimensional states and a larger number of parameters. The LTI system has two states, $\vecX_{t} \in \Real^{2}$, and shares the same structure as Experiment setup~1: $\vecU_{t} \in \Real^{2}$, $\A_{t} = \eye$, and $\B_{t} = \Delta t \eye$, with $\Delta t = 0.1$. The initial state is set to $\MuX{0} = [3.2, 2.8]^{\tr}\!$. The active-learning procedure uses the same input initialization structure $\vecU_{t} = 4.5\!\sum_{\upsilon \in \Upsilon} [\cos(\upsilon t\Delta t),\ \sin(\upsilon t\Delta t)]^{\tr}\!$, with $\Upsilon \in \{3, 5, 10, 20, 100\}$ and bounds $[-200, 200]$, and the resulting optimal trajectory $\vecUbar^{\star}$ is used for all methods in this setup. Under the Fourier basis representation in~\eqref{eq:Fourier_Basis}, we use eleven unknown parameters $\vecTheta_{n}$ ($n = 0, \ldots, 10$) with known frequency vectors $f_{n} \in \Real^{2}$ defined as: $f_{0} = [0, 0]^{\tr}\!$, $f_{n} = [n\frac{2\pi}{10}, 0]^{\tr}\!$ for $n \in \{1, 2, 3\}$, and $f_{n} = [(n-7)\frac{2\pi}{10},\, \frac{2\pi}{6}]^{\tr}\!$ for $n \in \{4, \ldots, 10\}$. The unknown parameters follow the prior $\UniformDist(2, 8)$, corresponding to $\MuTheta{n} = 5$ and $\sigmaTheta{n}^{2} = 3$, from which true parameter values are sampled.
\paragraph{\textbf{Benchmark~2.1} (\textbf{High-dimensional parameters})}
This benchmark uses Experiment setup~2 with trajectory length $\Trajectory=100$ (i.e., $101$ measurements) and two process-noise variances $\sigma^{2}_{\vecW} \in \{0.001,\, 0.01\}$, with $10,\!000$ simulations per configuration. Using identical realizations of $(\vecWbar, \vecVbar)$ across methods, we compute parameter estimates $\hat{\vecTheta}(\vecYbar)$ from A-MMSE~\eqref{eq:Parameters_Affine_Estimation} and the dual estimators \ref{eq:Dual_basis_parameter_estimator} and \ref{eq:Dual_state_parameter_estimator}. Figure~\ref{Fig:param_Error_distributions} shows histograms of $\log \|\vecTheta - \hat{\vecTheta}(\vecYbar)\|^{2}$ over all simulations, with dashed vertical lines indicating the empirical MSE for each estimator: orange for A-MMSE, blue for \ref{eq:Dual_basis_parameter_estimator}, and red for \ref{eq:Dual_state_parameter_estimator}. In contrast to Experiment setup~1, this configuration has only two latent states but ten unknown parameters, leading to a practically relevant identification problem with high-dimensional parameters and low-dimensional dynamics. In both noise regimes, A-MMSE (orange) exhibits a higher empirical MSE than both dual estimators. The low state dimension ($\nX = 2$) enables \ref{eq:Dual_state_parameter_estimator} to exploit latent-state dynamics more effectively than \ref{eq:Dual_basis_parameter_estimator}, achieving the lowest empirical MSE (red) and demonstrating that explicit state estimation can enhance parameter inference when Gaussian linear dynamics govern the latent states.
\begin{figure}[htbp]
\centering 
\includegraphics[width=0.4523\linewidth]{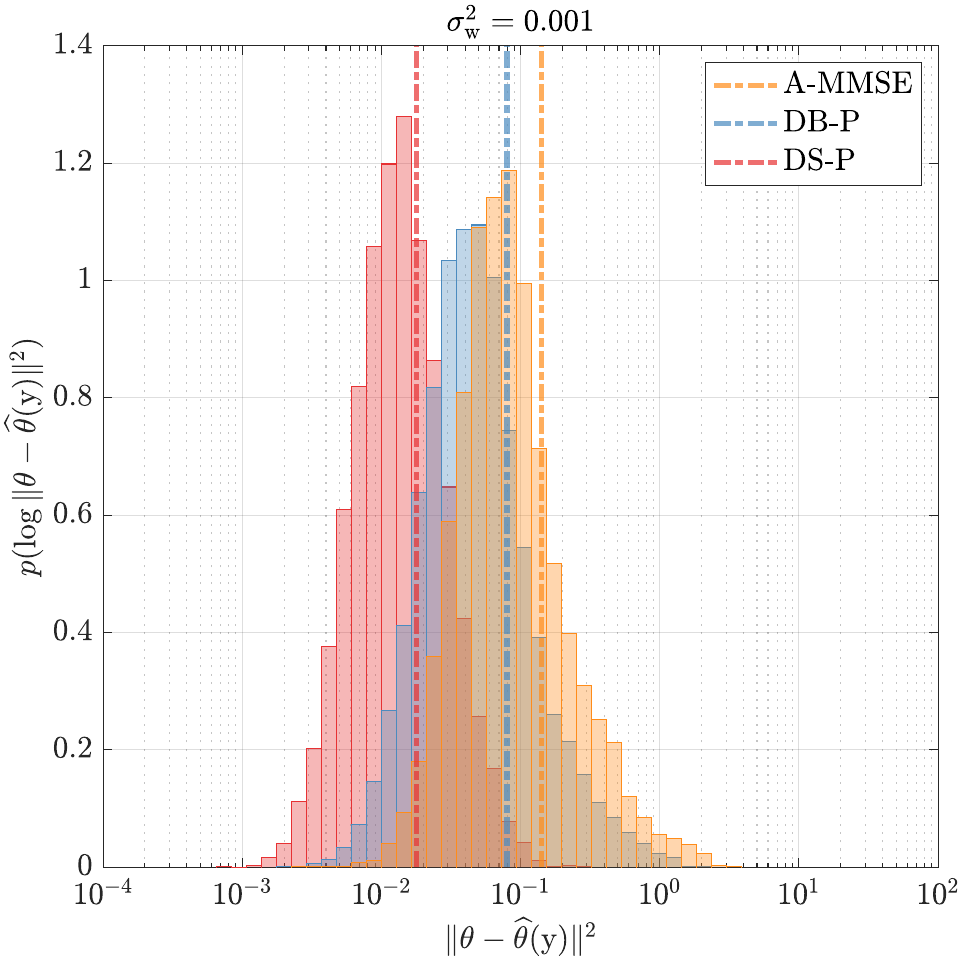} \qquad 
\includegraphics[width=0.4523\linewidth]{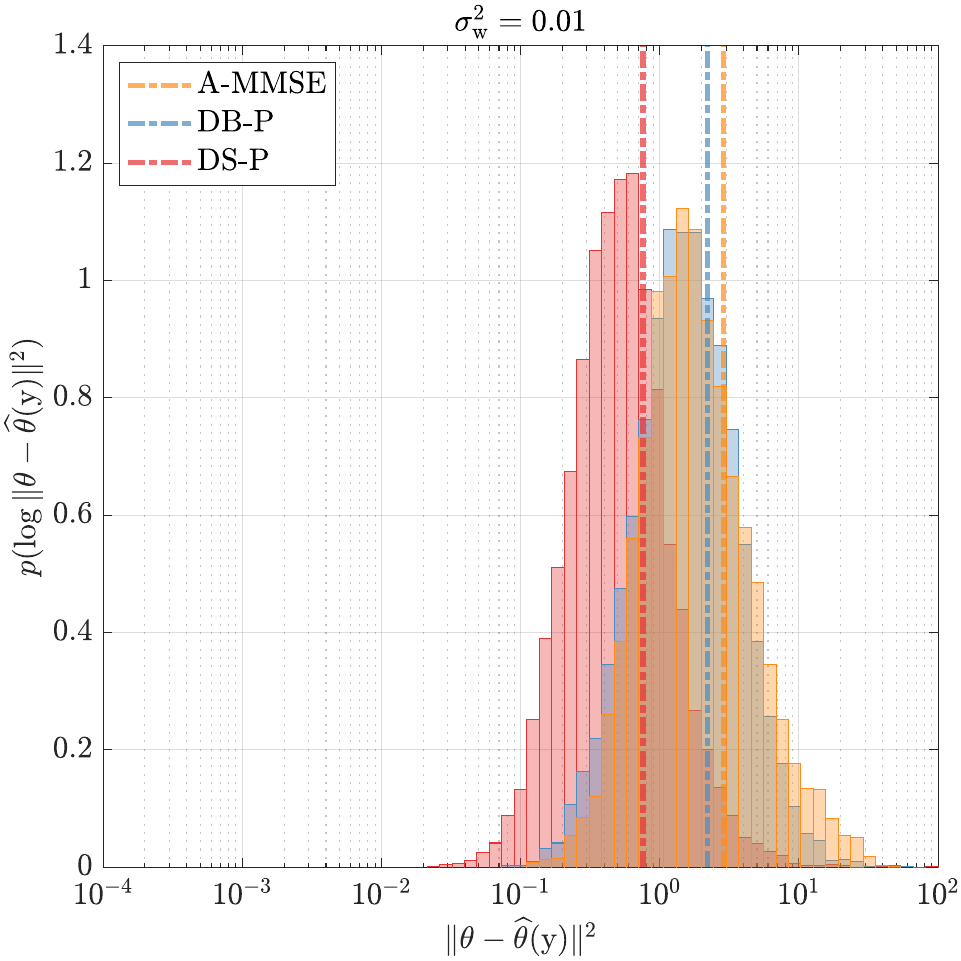}
\caption{Parameter estimation squared-error distributions for Experiment setup~2}
\label{Fig:param_Error_distributions}
\end{figure}
\paragraph{\textbf{Benchmark~2.2} (\textbf{Trajectory horizon})}
This experiment considers $10,\!000$ simulations with trajectories of varying lengths, $\Trajectory \in \{0,\, 4,\, 10,\, 13,\, 16,\, 20,\, 25,\, 32,\, 40,\, 50,\, 63,\, 79,\, 100\}$, to examine how parameter estimation error evolves as the number of measurements increases. Figure~\ref{Fig:param_Error_differentMeasurements} compares the squared errors for A-MMSE (orange), \ref{eq:Dual_basis_parameter_estimator} (blue), and \ref{eq:Dual_state_parameter_estimator} (red) under two process-noise variances $\sigma^{2}_{\vecW} \in \{0.001,\, 0.01\}$. Shaded regions show the $15^{\nth}$-$85^{\nth}$ percentile range of the squared error, while dashed lines indicate the empirical MSE. All estimators use an optimized input $\vecUbar^{\star}$ designed separately for each trajectory length. The results confirm that input optimization significantly reduces squared error for all three estimators when the number of measurements matches the number of unknown parameters. For both noise levels, the error reduction of A-MMSE and \ref{eq:Dual_basis_parameter_estimator} slows markedly as $\Trajectory$ increases and measurements become more strongly correlated. This behaviour aligns with the established limitations of A-MMSE for marginally stable systems~\cite[Prop.~4.2]{vakili2025optimal}, and the similar trend for \ref{eq:Dual_basis_parameter_estimator} suggests it may suffer from related limitations. In contrast, \ref{eq:Dual_state_parameter_estimator} exhibits a sustained downward trend in empirical MSE as $\Trajectory$ grows. In the low process-noise case ($\sigma^{2}_{\vecW} = 0.001$, left plot), the red dashed line decreases steadily with increasing $\Trajectory$. For higher process noise ($\sigma^{2}_{\vecW} = 0.01$, right plot), the red dashed line shows slight local increases between $\Trajectory = 32$ and $40$ and between $50$ and $63$, before decreasing again for larger $\Trajectory$. These small increases reflect heavier upper tails (above the $85^{\nth}$ percentile) of the squared-error distribution, as the central $15^{\nth}$-$85^{\nth}$ percentile range continues to decrease steadily. These results highlight the advantages of \ref{eq:Dual_state_parameter_estimator} in the high-dimensional parameter regime and motivate a comparison with SMC-based methods in Benchmark~2.3.
\begin{figure}[htbp]
\centering 
\includegraphics[width=0.4523\linewidth]{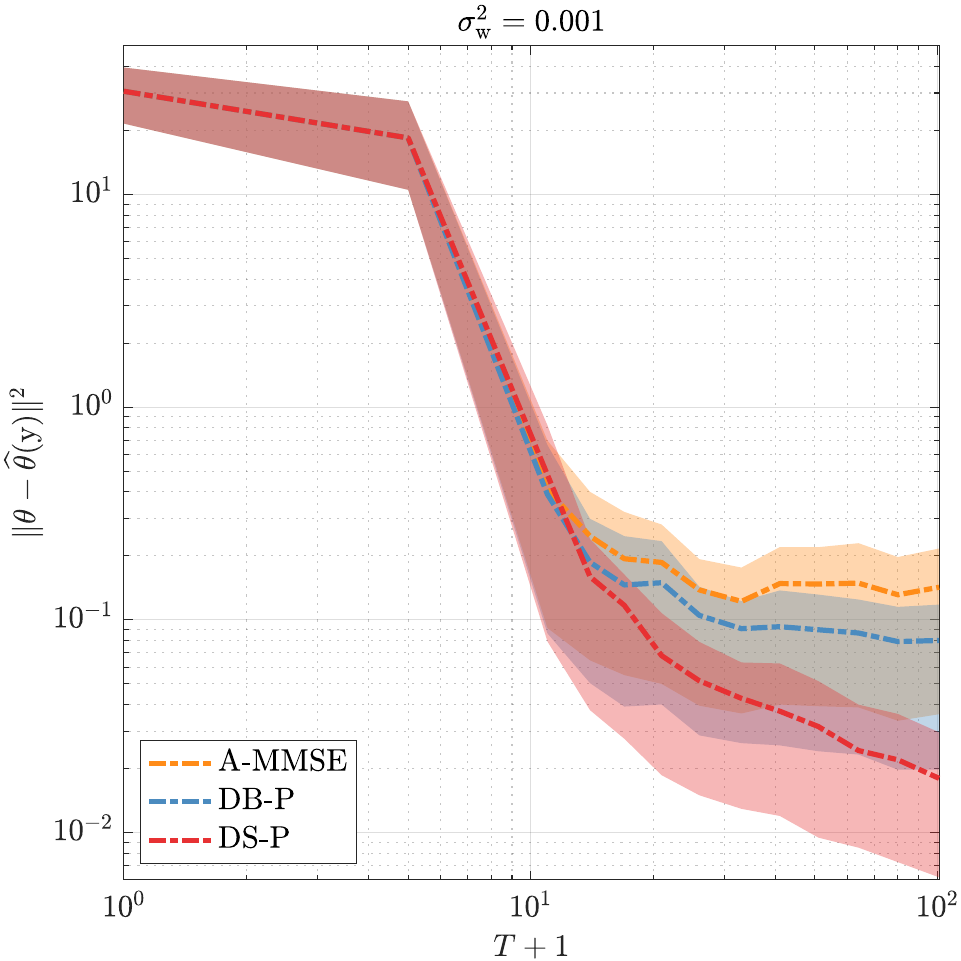} \qquad 
\includegraphics[width=0.4523\linewidth]{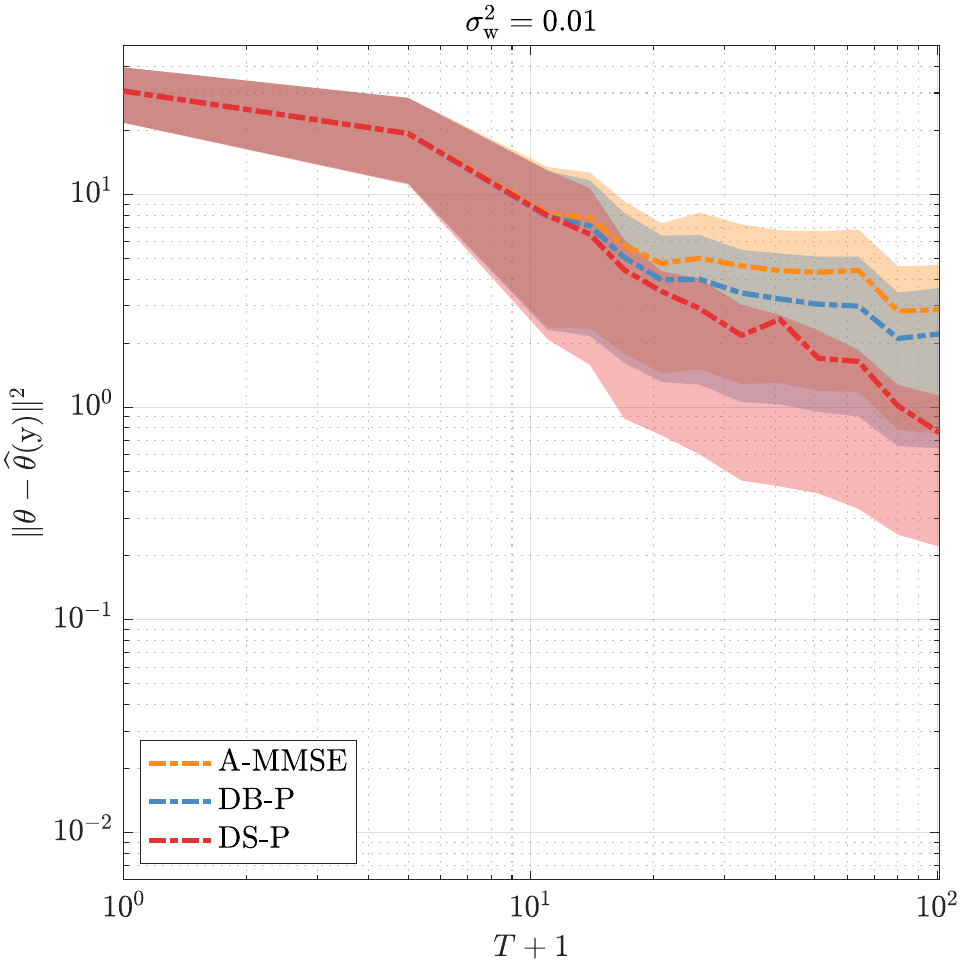}
\caption{Parameter estimation error vs. number of measurements for Experiment setup~2}
\label{Fig:param_Error_differentMeasurements}
\end{figure}
\paragraph{\textbf{Benchmark~2.3} (\textbf{Dual state-parameter vs. SMC methods})} 
In this final benchmark, \ref{eq:Dual_state_parameter_estimator} is compared with two well-known SMC-based approaches, PGAS and SMC-EM. Both methods alternate between updating the latent states given the current parameter estimates and updating the parameters given sampled state trajectories, whereas \ref{eq:Dual_state_parameter_estimator} alternates between parameter and state updates using summary statistics of the other block rather than full conditional sampling. The comparison uses Experiment setup~2 with $10,\!000$ simulations, trajectory length $\Trajectory=100$, and two process-noise levels $\sigma^{2}_{\vecW} \in \{0.001,\, 0.01\}$. Identical realizations of $(\vecWbar, \vecVbar)$ and the same optimized input $\vecUbar^{\star}$ from the active-learning procedure are used for all three methods. We briefly discuss each of the two SMC-based methods and refer interested readers to the cited references for further background and technical details. The following describes how each method is tuned for a fair benchmark comparison: 
\begin{enumerate} [label=(\roman*), itemsep = 2mm, topsep = 1mm, leftmargin = 7.5mm]

\item \textbf{PGAS}~\cite[Alg.~4]{lindsten2014particle}: PGAS is a PMCMC method that targets the joint posterior over the latent state trajectory and the parameters. At each iteration $k$, it alternates between: (i) drawing a new state trajectory $\vecXbar^{k}$ using the PGAS Markov kernel~\cite[Alg.~3]{lindsten2014particle} given $\vecTheta^{k-1}$, and (ii) sampling $\vecTheta^{k}$ from the conditional posterior $p(\vecTheta \!\mid\! \vecXbar^{k}, \vecYbar)$. In our configuration, we use $\Nparticle = 6000$ particles and run the algorithm for $K = 5000$ iterations. The initial state trajectory is obtained from a bootstrap particle filter, and the initial parameters are drawn from the prior $\UniformDist(2,8)$. At each iteration, the parameters are sampled from a multivariate truncated Gaussian distribution~\cite{li2015efficient}, corrected by an MH step to target the conditional posterior. The final estimate is the empirical mean of the retained samples, with the first $2500$ samples discarded as burn-in.

\item \textbf{SMC-EM}~\cite[Alg.~4]{schon_system_2011}: SMC-EM is a maximum-likelihood method that embeds particle smoothing within the EM framework. At each EM iteration $k$, it treats the full state trajectory as missing data and alternates between: (i) E-step: given $\vecTheta^{k-1}$, approximating the EM $Q$-function (the expected complete-data log-likelihood) by running a particle smoother to approximate the smoothing distribution of $\vecXbar$; and (ii) M-step: maximizing this approximate $Q$-function with respect to $\vecTheta$ to obtain $\vecTheta^{k}$. In our implementation, the original forward-backward smoother~\cite[Alg.~2-3]{schon_system_2011} is replaced by the MH-FFBP joint smoother~\cite[Fig.~5]{bunch2012improved}, which is an MCMC-based refinement of forward filtering backward smoothing methods~\cite{doucet2000sequential, used_smoothing}. We use a standard particle filter with $N_{F} = 6000$ particles in the forward pass~\cite[Fig.~1]{bunch2012improved} and, in the backward pass, $N_{S} = 700$ smoothing particles with MH chain length $M = 300$. This choice follows the complexity trade-off analysis in~\cite{bunch2012improved} and aims to improve computational efficiency while maintaining particle diversity. With a uniform prior on $\vecTheta$, the M-step reduces to a constrained maximization of the approximate $Q$-function over the prior support, equivalent to a MAP update. The algorithm is initialized at $\vecTheta^{0} = \MuTheta{}$ and run for a fixed $K = 1000$ iterations, since the particle-based approximation of the $Q$-function precludes reliable likelihood-based convergence diagnostics. The final estimate represents a local optimum of the Monte Carlo-based approximation of the posterior objective rather than a guaranteed stationary point of the true posterior.
\end{enumerate}

From a practical standpoint, the three methods differ notably in tuning requirements and computational effort. PGAS and SMC-EM require careful selection of multiple hyperparameters (particle numbers $\Nparticle$ for PGAS and $(N_{F}, N_{S})$ for SMC-EM, plus the chain length $M$ in the MH-FFBP smoother), which must be set conservatively large to mitigate Monte Carlo error, thereby increasing computational cost. In contrast, \ref{eq:Dual_state_parameter_estimator} requires no such Monte Carlo tuning beyond specifying the prior mean and covariance. A further practical advantage of \ref{eq:Dual_state_parameter_estimator} is its rapid convergence: we set the stopping tolerance $\epsilon = 10^{-6}$ in Algorithm~\ref{algo:Dual_Estimators} with a maximum of $K = 10,\!000$ iterations. For $\sigma^{2}_{\vecW} = 0.001$, the state-parameter variant (\ref{eq:Dual_state_parameter_estimator}) of Algorithm~\ref{algo:Dual_Estimators} requires at most $56$ iterations to satisfy the stopping criterion across all experiments. For $\sigma^{2}_{\vecW} = 0.01$, only $3$ out of $10,\!000$ experiments reach the maximum iteration limit, while the remaining runs converge in at most $229$ iterations. In contrast, PGAS required $5000$ iterations for sufficient posterior samples, and SMC-EM was run for $1000$ iterations with no reliable convergence diagnostics.

\begin{figure}[htbp]
\centering 
\includegraphics[width=0.4523\linewidth]{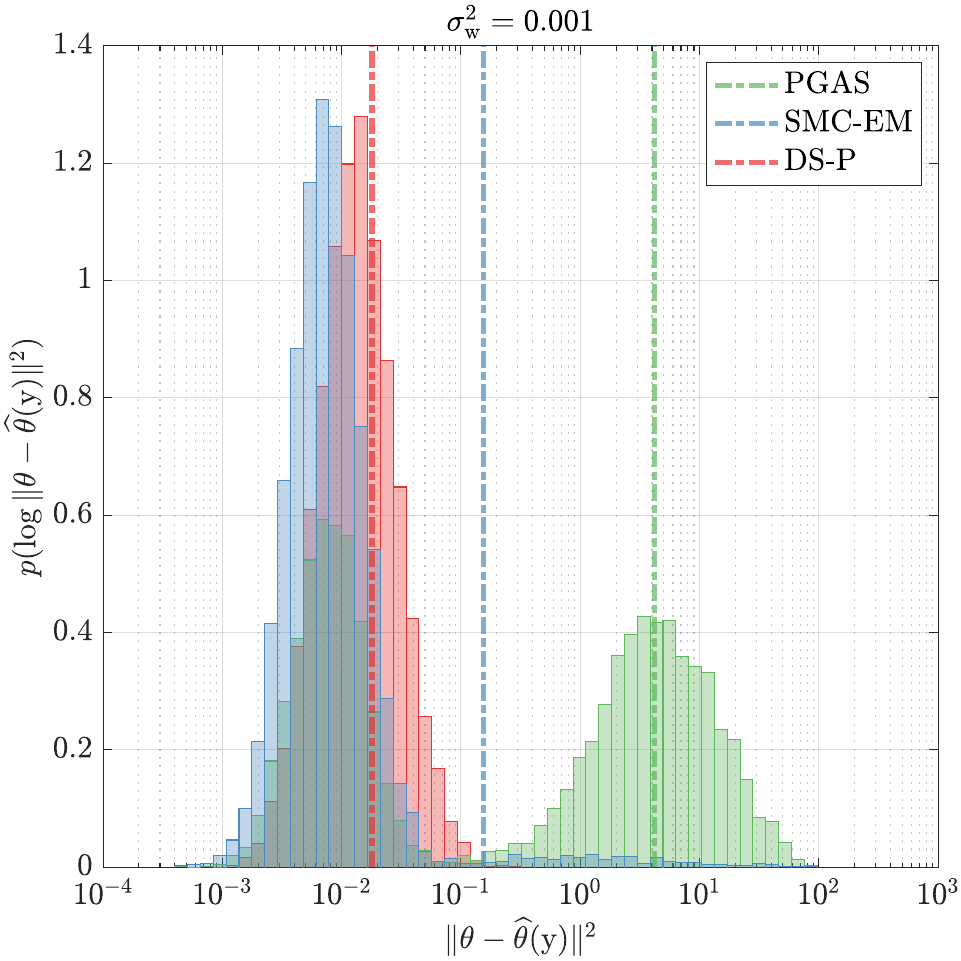} \qquad 
\includegraphics[width=0.4523\linewidth]{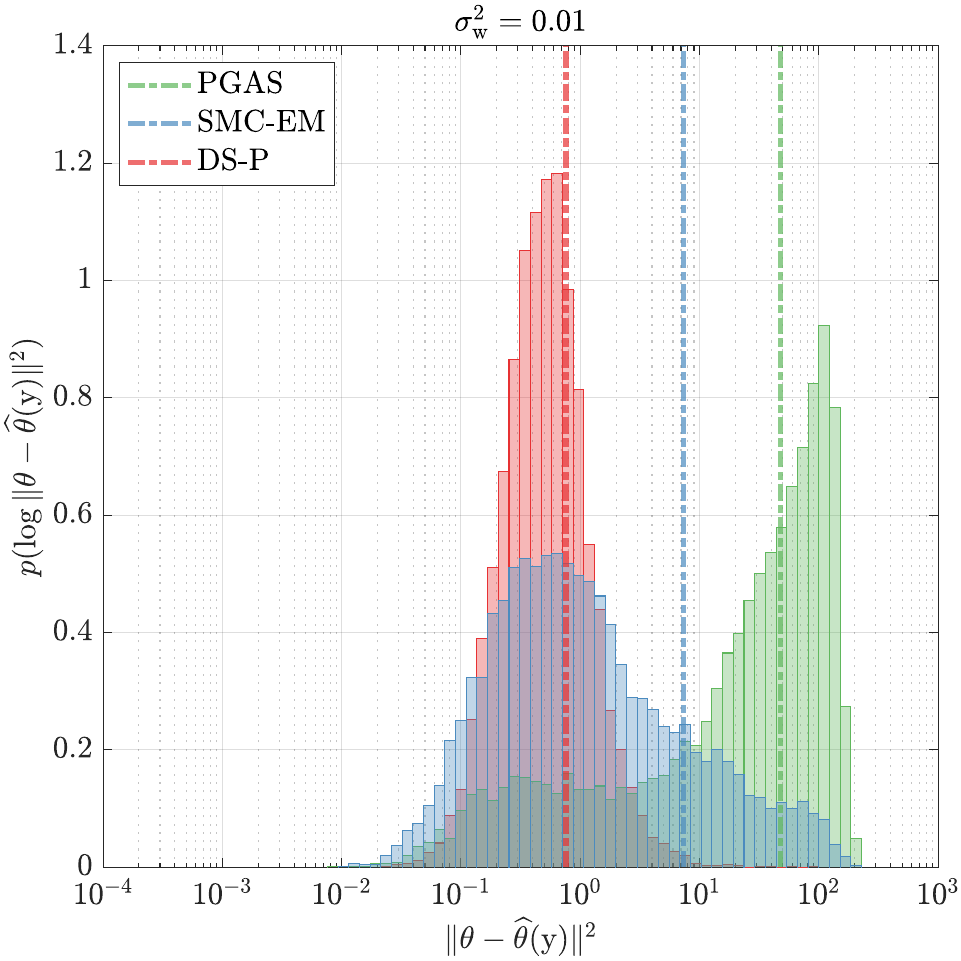}
\caption{Squared-error distributions: dual state-parameter vs. SMC-based methods}
\label{Fig:stateParam_Error_differentMethods}
\end{figure}

Figure~\ref{Fig:stateParam_Error_differentMethods} compares the distributions of $\log \|\vecTheta - \hat{\vecTheta}(\vecYbar)\|^{2}$ over all $10,\!000$ simulations for PGAS (green), SMC-EM (blue), and \ref{eq:Dual_state_parameter_estimator} (red), with dashed vertical lines indicating the empirical MSEs. In the low process-noise case ($\sigma^{2}_{\vecW} = 0.001$), SMC-EM produces an error distribution that overlaps partially with that of \ref{eq:Dual_state_parameter_estimator} but also exhibits a substantial number of simulations with noticeably larger errors, raising its overall MSE. PGAS shows a similar pattern: one concentration of errors near \ref{eq:Dual_state_parameter_estimator} and another at much larger values, again leading to a higher MSE. When the process noise increases to $\sigma^{2}_{\vecW} = 0.01$, the squared-error distributions of both SMC-EM and PGAS shift further away from the main concentration region of \ref{eq:Dual_state_parameter_estimator}, and their empirical MSEs increase correspondingly. In this higher-noise regime, the SMC-based methods produce a small fraction of simulations with errors comparable to \ref{eq:Dual_state_parameter_estimator}, while most realizations fall into higher-error regions. In both noise scenarios, \ref{eq:Dual_state_parameter_estimator} attains the lowest empirical MSE and the most tightly concentrated error distribution, demonstrating more robust parameter-estimation performance as the process-noise level increases.
\section{Conclusions} \label{sec:Conclusions}
This work presents a nonlinear parameter estimator as a fixed-point characterization of two affine MMSE estimators, one for the unknown parameters and one for latent variables, which use the summary statistics of each other to iteratively refine their estimates. Two variants are developed, both requiring no tuning beyond specifying the prior: the dual basis-parameter (\ref{eq:Dual_basis_parameter_estimator}) estimator, which directly estimates the DBS via an affine MMSE basis estimator, and the dual state-parameter (\ref{eq:Dual_state_parameter_estimator}) estimator, which uses affine MMSE state estimates as inputs for DBS estimates construction. Extensive numerical experiments validate that \ref{eq:Dual_state_parameter_estimator} consistently achieves the lowest empirical MSE, outperforming \ref{eq:Dual_basis_parameter_estimator}, the affine MMSE parameter estimator, and two established SMC-based methods (PGAS and SMC-EM). The superior performance of \ref{eq:Dual_state_parameter_estimator} stems from two mechanisms: the state-estimation step reduces the effective process-noise uncertainty, and the joint fixed-point iteration mutually refines both parameter and state estimates.

These results motivate several directions for future research. Although both algorithm variants converged to a fixed point in almost all experiments, a formal convergence analysis of the dual fixed-point iteration would provide valuable theoretical guarantees. Additionally, when the dual state-parameter variant (\ref{eq:Dual_state_parameter_estimator}) is used, the consistency of the nonlinear parameter estimator depends on the state estimation error remaining suitably bounded along the trajectory. Characterizing conditions that guarantee such boundedness would establish asymptotic performance guarantees for the proposed estimator. For finite data, however, extensive simulations indicate that the dual state-parameter estimator can substantially reduce parameter estimation error as the trajectory length increases (cf.~Figure~\ref{Fig:param_Error_differentMeasurements}). Other promising extensions include generalizing the framework to nonparametric models, accommodating nonlinear state dynamics, and developing recursive implementations that enable scalability to large datasets by reducing computational complexity.
\appendix
\section{Technical proofs} \label{appendix:Technical_proofs}

\subsection{Lemma~\ref{lemma:Generic_affine_MMSE_estimator}}
\label{proof:Generic_affine_MMSE_estimator_proof}
\begin{proof}
The proof follows the first-order optimality over the constants of the affine estimator~\eqref{eq:Affine_Class}. To this end, we apply the decompositions $\xi_{1} = \mu_{1} + \Delta\xi_{1}$ and $\xi_{2} = \mu_{2} + \Delta\xi_{2}$, where $\Delta\xi_{1}$ and $\Delta\xi_{2}$ are zero-mean random variables describing the uncertainty in $\xi_{1}$ and $\xi_{2}$, and then expand the objective as
\begin{align*}
\minOp_{\Psi,\, \psi} \,
\Expectation{
(\Delta\xi_{1} -\Psi\Delta\xi_{2})^{\!\tr}\!(\Delta\xi_{1} -\Psi\Delta\xi_{2})} & - 2\Expectation{(\Delta\xi_{1} -\Psi\Delta\xi_{2})^{\!\tr}\!(\psi + \Psi\mu_{2} - \mu_{1})} \\ 
& + \Expectation{(\psi + \Psi\mu_{2} - \mu_{1})^{\!\tr}\!(\psi + \Psi\mu_{2} - \mu_{1})}. \notag 
\end{align*}
The second term of the first line vanishes because $\psi + \Psi\mu_{2} - \mu_{1}$ is deterministic, and the random variables~$(\Delta\xi_{1}, \Delta\xi_{2})$ are zero mean, i.e., $\Expectation{\Delta\xi_{1}} = 0$ and $\Expectation{\Delta\xi_{2}} = 0$. The last term is minimized and its contribution set to zero by choosing $\psi^{\star} = \mu_{1} - \Psi\mu_{2}$. With this choice, applying the trace operator, and defining the covariances $\Sigma_{11} \coloneqq \Expectation{\Delta\xi_{1}\Delta\xi_{1}^{\tr}}$, $\Sigma_{12} \coloneqq \Expectation{\Delta\xi_{1}\Delta\xi_{2}^{\tr}}$, and $\Sigma_{22} \coloneqq \Expectation{\Delta\xi_{2}\Delta\xi_{2}^{\tr}}$, where $\Delta\xi_{1} = \xi_{1} - \mu_{1} $ and $\Delta\xi_{2} = \xi_{2} - \mu_{2}$, the optimization reduces to minimizing
$$
\Cost(\Psi) = \trace(\Sigma_{11} - 2\Psi\Sigma_{21} + \Psi\Sigma_{22}\Psi^{\tr}),
$$
with $\Sigma_{21} = \Sigma_{12}^{\tr}$. The optimizer $\Psi^{\star}$ is found by setting the partial derivative of $\Cost(\Psi)$ with respect to $\Psi$ to zero, i.e.,
\begin{align*}
\dfrac{\partial \Cost}{\partial \Psi} \Big|_{\Psi=\Psi^{\star}} & \!\!= -2\Sigma_{12} + 2\Psi^{\star} \Sigma_{22} = 0,
\end{align*}
which yields $\Psi^{\star} = \Sigma_{12}\Sigma_{22}^{\inv}$. Since $\Sigma_{22}$ is invertible, $\Cost(\Psi)$ is a strictly convex quadratic in $\Psi$, so this stationary point is the unique global minimizer. Substituting $\Psi^{\star}$ and $\psi^{\star} = \mu_{1} - \Psi^{\star}\mu_{2}$ into the estimation-error covariance $\Sigma_{\hat{\xi}_{1}} \! \coloneqq \Expectation{(\xi_{1} - \Psi\xi_{2} - \psi) (\xi_{1} - \Psi\xi_{2} - \psi)^{\!\tr}\!}$,
yields the expressions in~\eqref{eq:Affine_MMSE_closed_form}, which completes the proof.
\end{proof}

\subsection{Lemma~\ref{lemma:MMSE_Lowerbound}}
\label{proof:MMSE_Lowerbound_proof}
\begin{proof}
By the law of total expectation~\cite[Ch.~4]{carlton2017probability},
\begin{equation} \label{eq:MSE_Total_Expectation}
\Cost^{\star}_{\vecTheta}(\matrixPhibar, \SigmaPhi) = \bigExpectation{\|\vecTheta - \hat{\vecTheta}_{\mathrm{af}}(\vecYbar ; \matrixPhibar, \SigmaPhi)\|^{2}} = \CondExpectation{\vecXbar}{\bigExpectation{\|\vecTheta - \hat{\vecTheta}_{\mathrm{af}}(\vecYbar ; \matrixPhibar, \SigmaPhi)\|^{2} \,\big|\, \vecXbar}},
\end{equation}
where $\hat{\vecTheta}_{\mathrm{af}}(\vecYbar ; \matrixPhibar, \SigmaPhi)$ is the affine MMSE estimator with its optimal MSE $\Cost^{\star}_{\vecTheta}(\matrixPhibar, \SigmaPhi)$ as defined in~\eqref{eq:Parameters_Affine_Estimation}. Let $\estimatorAffine{\vecTheta}$ denote the class of affine estimators of $\vecTheta$ from $\vecYbar$, analogous to~\eqref{eq:Affine_Class}. Since $\hat{\vecTheta}_{\mathrm{af}}(\vecYbar ; \matrixPhibar, \SigmaPhi) \in \estimatorAffine{\vecTheta}$, the conditional MSE of this estimator cannot be smaller than the minimum conditional MSE over all affine estimators in $\estimatorAffine{\vecTheta}$ for each fixed $\vecXbar$, i.e.,
\begin{equation} \label{eq:MMSE_inequality}
\CondExpectation{\vecXbar}{\bigExpectation{\|\vecTheta - \hat{\vecTheta}_{\mathrm{af}}(\vecYbar ; \matrixPhibar, \SigmaPhi)\|^{2} \,\big|\, \vecXbar}} \geq 
\CondExpectation{\vecXbar}{\min_{\hat{\vecTheta}(\cdot) \in \estimatorAffine{\vecTheta}} \, \bigExpectation{\big\|\vecTheta - \hat{\vecTheta}(\vecYbar)\big\|^2 \,\big|\, \vecXbar}}.
\end{equation}
Moreover, for each fixed state trajectory $\vecXbar$, the optimizer of the minimum conditional MSE and its corresponding affine MMSE value are given by
\begin{equation} \label{eq:Affine_MMSE_complete_knowledge}
\begin{cases}
\hat{\vecTheta}_{\mathrm{af}}(\vecYbar ; \matrixPhi(\vecXbar), 0) = \arg\!\min_{\hat{\vecTheta}(\cdot) \in \estimatorAffine{\vecTheta}} \, \bigExpectation{\big\|\vecTheta - \hat{\vecTheta}(\vecYbar)\big\|^2 \,\big|\, \vecXbar}, \\[0.2em]
\Cost^{\star}_{\vecTheta}(\matrixPhi(\vecXbar), 0) = \bigExpectation{\| \vecTheta - \hat{\vecTheta}_{\mathrm{af}}(\vecYbar ; \matrixPhi(\vecXbar), 0) \|^{2} \,\big|\, \vecXbar}.
\end{cases}
\end{equation}
Combining~\eqref{eq:MSE_Total_Expectation},~\eqref{eq:MMSE_inequality}, and~\eqref{eq:Affine_MMSE_complete_knowledge} arrives at the desired bound in~\eqref{eq:MSE_lower_bound}.
\end{proof}

\subsection{Lemma~\ref{lemma:MMSE_Affine_Basis_Estimator}}
\label{proof:MMSE_Affine_Basis_Estimator_proof}
\begin{proof}
Define the stacked basis-function vector 
\begin{equation} \label{eq:stacked_basis}
\phi \coloneqq \vectorize(\matrixPhi(\vecXbar)) = [\phi^{\tr}\!\!(\vecX_{0}),\ldots,\phi^{\tr}\!\!(\vecX_{\Trajectory})]^{\tr}\!, 
\end{equation}
where $\phi(\vecX_{t})$ is defined in~\eqref{eq:Function_Space}. By Definition~\ref{definition:Dynamic_Basis_Statistics_Collection}, the prior DBS $(\matrixPhibar,\SigmaPhi)$ induce the prior mean and covariance of $\phi$ as
\begin{equation} \label{eq:basis_mean}
\phiBar \coloneqq \Expectation{\phi}
= \Expectation{\vectorize(\matrixPhi(\vecXbar))}
= \vectorize(\matrixPhibar),
\qquad
\Expectation{(\phi-\phiBar)(\phi-\phiBar)^{\tr}\!} = \SigmaPhi.
\end{equation}
Next, introduce the lifted observation model over the trajectory. Stacking the sequence of measurements as $\vecYbar = [\vecY_{0}, \ldots, \vecY_{\Trajectory}]^{\tr}\!$ and using~\eqref{eq:Function_Space} yields
\begin{equation} \label{eq:Equivalent_Lifted_ObservationModel}
\vecYbar = \matrixG_{\theta}^{\tr} \phi + \vecVbar,
\end{equation}
where $\matrixG_{\theta} = \diag(\timedTheta{0}, \ldots, \timedTheta{\Trajectory})$ with $\timedTheta{t} = \vecTheta$ for all $t \in \{ 0, \ldots, \Trajectory \}$, and $\vecVbar = [\vecV_{0}, \ldots, \vecV_{\Trajectory}]^{\tr}\!$ has covariance matrix $\SigmaVbar{}$. The affine MMSE estimator of $\phi$ given $\vecYbar$ follows from the generic affine MMSE estimator 
of Lemma~\ref{lemma:Generic_affine_MMSE_estimator} with $\xi_{1}=\phi$ and $\xi_{2}=\vecYbar$, 
resulting in the form $\hat{\phi}_{\mathrm{af}}(\vecYbar) = \PsiPhi^{\star} \vecYbar + \psiPhi^{\star}$, 
with coefficients
\begin{equation} \label{eq:MMSE_basis_coefficients}
\PsiPhi^{\star} = \Sigma_{\phi\vecYbar}\Sigma_{\vecYbar}^{\inv}, \qquad \psiPhi^{\star} = \phiBar - \PsiPhi^{\star}\mu_{\vecYbar},
\end{equation}
where $\phiBar$ is the mean of $\phi$ defined in~\eqref{eq:basis_mean}, $\mu_{\vecYbar}$ is the observation mean, $\Sigma_{\phi\vecYbar}$ is the cross-covariance between $\phi$ and $\vecYbar$, and $\Sigma_{\vecYbar}$ is the observation covariance. To express these quantities, apply the decompositions: $\phi = \phiBar + \deltaPhi$ and $\matrixG_{\theta} = \mu_{\Theta} + \Delta\matrixG_{\theta}$, where $\deltaPhi$ and $\Delta\matrixG_{\theta}$ are zero-mean random quantities, and $\mu_{\Theta} \coloneqq \Expectation{\matrixG_{\theta}} = \diag(\timedMuTheta{0}, \ldots, \timedMuTheta{\Trajectory})$ with $\timedMuTheta{t} = \MuTheta{}$ for all $t \in \{0,\ldots,\Trajectory\}$. By construction, $\Expectation{\deltaPhi\deltaPhi^{\tr}\!} = \SigmaPhi$ and $\Expectation{(\timedTheta{t} - \timedMuTheta{t})(\timedTheta{t} - \timedMuTheta{t})^{\tr}\!} = \SigmaTheta{}$. From the independence assumptions between $\phi$, $\vecTheta$ and $\vecVbar$ it follows that $\Expectation{\Delta\matrixG_{\theta}^{\tr}\deltaPhi} = 0$ and $\Expectation{\Delta\matrixG_{\theta}^{\tr}\deltaPhi\vecVbar^{\tr}\!} = 0$. Using~\eqref{eq:Equivalent_Lifted_ObservationModel} together with these decompositions, and exploiting the zero-mean and independence properties, yields
\begin{equation} \label{eq:Basis_estimator_mean_covariances}
\begin{aligned}
\mu_{\vecYbar} & = \Expectation{(\mu_{\Theta} + \Delta\matrixG_{\theta})^{\tr}\!(\phiBar + \deltaPhi) + \vecVbar}
= \mu_{\Theta}^{\tr}\phiBar = \matrixPhibar^{\tr} \MuTheta{}, \\
\Sigma_{\phi\vecYbar} & = \bigExpectation{\deltaPhi \big((\mu_{\Theta} + \Delta\matrixG_{\theta})^{\tr}\!(\phiBar + \deltaPhi) + \vecVbar - \mu_{\Theta}^{\tr}\phiBar\big)^{\!\!\tr}\!}
= \Expectation{\deltaPhi\deltaPhi^{\tr}\!} \mu_{\Theta}
= \SigmaPhi \mu_{\Theta}, \\
\Sigma_{\vecYbar} & = \Expectation{(\vecYbar-\mu_{\Theta}^{\tr}\phiBar)(\vecYbar-\mu_{\Theta}^{\tr}\phiBar)^{\tr}\!} = \mu_{\Theta}^{\tr}\SigmaPhi\mu_{\Theta} + \matrixP + \matrixQ + \matrixS + \matrixS^{\tr}\! + \SigmaVbar{},
\end{aligned}
\end{equation}
where $\matrixP = \bigExpectation{\Delta\matrixG_{\theta}^{\tr} \deltaPhi \deltaPhi^{\tr}\! \Delta\matrixG_{\theta}}$, $\matrixQ = \bigExpectation{\Delta\matrixG_{\theta}^{\tr} \phiBar \phiBar^{\tr}\! \Delta\matrixG_{\theta}}$ and $\matrixS = \Expectation{\Delta\matrixG_{\theta}^{\tr}\phiBar\deltaPhi^{\tr}\! \Delta\matrixG_{\theta}}$. Consider the $(t,t')^{\nth}$ components of these matrices, denoted by $\matrixP_{tt'}$, $\matrixQ_{tt'}$, and $\matrixS_{tt'}$.  By explicitly exploiting the block-diagonal structure of $\matrixG_{\theta}$ and $\mu_{\Theta}$ and denoting by $\DeltaTheta$ the random parameter vector appearing in each individual block, we derive the following explicit closed-form expressions for $\matrixP_{tt'}$, $\matrixQ_{tt'}$, and $\matrixS_{tt'}$, which in turn implies that $\matrixS=0$ by applying the trace operator and noting that $\Expectation{\deltaPhi(\vecX_{t'})} = 0$ by construction:
$$
\begin{aligned}
\matrixP_{\!tt'} & = \Expectation{\DeltaTheta^{\tr}\! \deltaPhi(\vecX_{t}) \deltaPhi^{\tr}\!(\vecX_{t'}) \DeltaTheta} = \Expectation{\trace(\DeltaTheta^{\tr}\! \deltaPhi(\vecX_{t}) \deltaPhi^{\tr}\!(\vecX_{t'}) \DeltaTheta)} = 
\trace( \SigmaTheta{} \TimedSigmaPhi{^{tt'}}), \\
\matrixQ_{tt'} & = \Expectation{\DeltaTheta^{\tr}\! \MuPhi^{t} \MuPhi^{t'^{\tr}}\! \DeltaTheta} = \Expectation{\trace(\DeltaTheta^{\tr}\! \MuPhi^{t} \MuPhi^{t'^{\tr}}\! \DeltaTheta)} = 
\trace(\SigmaTheta{}\MuPhi^{t}\MuPhi^{t'^{\tr}}), \\
\matrixS_{tt'} & = \Expectation{\DeltaTheta^{\tr}\! \MuPhi^{t} \deltaPhi^{\tr}\!(\vecX_{t'}) \DeltaTheta} = \Expectation{\trace( \DeltaTheta^{\tr}\! \MuPhi^{t} \deltaPhi^{\tr}\!(\vecX_{t'}) \DeltaTheta)} = 
\trace \big(\Expectation{\DeltaTheta \DeltaTheta^{\tr}\!} \MuPhi^{t} \Expectation{\deltaPhi^{\tr}\!(\vecX_{t'})}\big) = 0,
\end{aligned}
$$
Collecting these expressions, define the linear operator $\operatorParam \colon \Real^{(\Ntheta+1)^{2}} \!\! \to \! \Real^{(\Trajectory+1)^2}$ by its $(t,t')^{\nth}$ element
$$
\operatorParam^{tt'}\!(\SigmaTheta{}) = \trace \big( \SigmaTheta{} ( \TimedSigmaPhi{^{tt'}} + \MuPhi^{t}\MuPhi^{t'^{\tr}\!} ) \big).
$$
Consequently, $\Sigma_{\vecYbar} = \mu_{\Theta}^{\tr}\SigmaPhi\mu_{\Theta} + \operatorParam(\SigmaTheta{}) + \SigmaVbar{}$, which is invertible because $\SigmaVbar{} \succ 0$. Using this $\Sigma_{\vecYbar}$ together with $\mu_{\vecYbar}$ and $\Sigma_{\phi\vecYbar}$ from~\eqref{eq:Basis_estimator_mean_covariances} to substitute into~\eqref{eq:MMSE_basis_coefficients} yields the closed-form expressions for $\PsiPhi^{\star}(\MuTheta{}, \SigmaTheta{})$  and $\psiPhi^{\star}(\MuTheta{}, \SigmaTheta{})$ in~\eqref{eq:MMSE_BasisEstimator}, both as functions of $(\MuTheta{}, \SigmaTheta{})$. Hence the affine MMSE basis estimator can be written as $\hat{\phi}_\mathrm{af}(\vecYbar ; \MuTheta{}, \SigmaTheta{}) = \PsiPhi^{\star}(\MuTheta{}, \SigmaTheta{}) \vecYbar + \psiPhi^{\star}(\MuTheta{}, \SigmaTheta{})$, emphasizing the dependence of the estimator on the prior mean and covariance of the unknown parameters. Furthermore, the corresponding estimation-error covariance and MMSE of this estimator follow from~\eqref{eq:Affine_MMSE_closed_form}. Since $\hat{\phi}_\mathrm{af}(\vecYbar ; \MuTheta{}, \SigmaTheta{})$ is stacked as in~\eqref{eq:stacked_basis}, its components $\hat{\phi}_\mathrm{af}^{\mkern1mu t}(\vecYbar ; \MuTheta{}, \SigmaTheta{})$ directly define the column-wise DBS estimate $\hat{\matrixPhi}_{\mathrm{af}}(\vecYbar ; \MuTheta{}, \SigmaTheta{})$, together with $\hatSigmaPhi{\mathrm{af}}\!(\MuTheta{}, \SigmaTheta{})$ in~\eqref{eq:BasisCollection_Affine_Estimation}, with the corresponding optimal MMSE cost $\Cost^{\star}_{\phi}(\MuTheta{}, \SigmaTheta{}) = \trace(\hatSigmaPhi{\mathrm{af}}\!(\MuTheta{}, \SigmaTheta{}))$.
\end{proof}

\subsection{Lemma~\ref{lemma:MMSE_Affine_State_Estimator}}
\label{proof:MMSE_Affine_State_Estimator_proof}
\begin{proof}
The affine MMSE estimator of $\vecXbar$ given $\vecYbar$ follows from the generic affine MMSE estimator of Lemma~\ref{lemma:Generic_affine_MMSE_estimator} with $\xi_{1}=\vecXbar$ and $\xi_{2}=\vecYbar$, resulting in the form $\hat{\vecXbar}_{\mathrm{af}}(\vecYbar) = \PsiX^{\star} \vecYbar + \psiX^{\star}$, where the optimal coefficients based on~\eqref{eq:Affine_optimal_coefficients} are $\PsiX^{\star} = \Sigma_{\vecXbar\vecYbar}\Sigma_{\vecYbar}^{\inv}$ and $\psiX^{\star} = \mu_{\vecXbar} - \PsiX^{\star}\mu_{\vecYbar}$. Moreover, the estimation-error covariance and optimal MSE following~\eqref{eq:Affine_MMSE_closed_form} are $\hatSigmaXbar{\mathrm{af}}\! = \Sigma_{\vecXbar} - \Sigma_{\vecXbar\vecYbar}\Sigma_{\vecYbar}^{\inv}\Sigma_{\vecXbar\vecYbar}^{\tr}$ and $\Cost^{\star} = \trace(\hatSigmaXbar{\mathrm{af}})$, respectively. From~\eqref{eq:Lifted_Dynamic}, we have $\mu_{\vecXbar} = \Abar\Bbar\vecUbar$ and $\Sigma_{\vecXbar} = \Abar\SigmaWbar{}\Abar^{\!\tr}\!$. Furthermore, from~\eqref{eq:Basis_estimator_mean_covariances} in the proof of Lemma~\ref{lemma:MMSE_Affine_Basis_Estimator} (Appendix~\ref{proof:MMSE_Affine_Basis_Estimator_proof}), we have $\mu_{\vecYbar} = \matrixPhibar^{\tr} \MuTheta{}$ and $\Sigma_{\vecYbar} = \mu_{\Theta}^{\tr}\SigmaPhi\mu_{\Theta} + \operatorParam(\SigmaTheta{}) + \SigmaVbar{}$. Thus, it remains to compute the cross-covariance matrix $\Sigma_{\vecXbar\vecYbar}$ between the state trajectory $\vecXbar$ and the observation sequence $\vecYbar$. In particular,
\begin{equation}\label{eq:cross_cov_expansion}
\Sigma_{\vecXbar\vecYbar} = \Expectation{(\vecXbar-\Abar\Bbar\vecUbar)(\matrixPhi^{\tr}\!\!(\vecXbar)\vecTheta+\vecVbar - \matrixPhibar^{\tr}\MuTheta{})^{\!\tr}\!} = \Expectation{(\vecXbar-\Abar\Bbar\vecUbar)\MuTheta{}^{\tr}\matrixPhi(\vecXbar)},
\end{equation}
where the observation model $\vecYbar = \matrixPhi^{\tr}\!\!(\vecXbar)\vecTheta+\vecVbar$ is used as in~\cite{vakili2025optimal}, and the equality follows from algebraic expansion using the independence of $\vecTheta$ from $\vecXbar$ (and hence from $\matrixPhi(\vecXbar)$), which yields $\Expectation{\vecXbar\vecTheta^{\tr}\!\matrixPhi(\vecXbar)} = \Expectation{\vecXbar\MuTheta{}^{\tr}\matrixPhi(\vecXbar)}$, together with $\Expectation{\vecXbar\vecVbar^{\tr}\!} = \zero$ by independence of $\vecXbar$ and $\vecVbar$. Consider the $(t, t')^\nth$ block of this cross-covariance matrix, $\Sigma_{\vecXbar\vecYbar}^{tt'}$, which captures the cross-covariance between $\vecX_{t}$ and $\vecY_{t'}$. The state $\vecX_{t}$ can be written as $\vecX_{t} = \Abar_{t}(\Bbar \vecUbar + \vecWbar)$, where $\Abar_{t}$ is the $(t+1)^\nth$ block-row of $\Abar$ as in~\eqref{eq:matrixAandB}, i.e.,
\begin{equation} \label{eq:matrixA_t}
\Abar_{t} = \begin{bmatrix}
\prodOp_{i=0}^{t-1} \A_{i} & \prodOp_{i=1}^{t-1} \A_{i} &  \ldots & \eye & \zeromx & \ldots & \zeromx
\end{bmatrix}.
\end{equation}
The $(t, t')^\nth$ block of the cross-covariance matrix can then be expressed as
\begin{equation} \label{eq:State_Basis_Covariance_Matrix}
\Sigma_{\vecXbar\vecYbar}^{tt'} = \Expectation{(\vecX_{t} - \Abar_{t}\Bbar\vecUbar)\MuTheta{}^{\tr}\phi(\vecX_{t'})} = \sumOp_{n=0}^{\Ntheta} \Expectation{(\vecX_{t} - \Abar_{t}\Bbar\vecUbar)\phi_{n}(\vecX_{t'})} \MuTheta{n},
\end{equation}
where $\phi(\vecX_{t'})$ is the $(t'+1)^{\nth}$ column of the basis collection matrix $\matrixPhi(\vecXbar)$, according to Definition~\ref{definition:Dynamic_Basis_Statistics_Collection}, and $\MuTheta{n}$ denotes the $n^{\nth}$ component of the prior mean $\MuTheta{}$ for the unknown parameter vector $\vecTheta$. Under Assumption~\ref{assumption:Gaussian_process_noise}, $\vecWbar \sim \NormalDist(\zero,\SigmaWbar{})$, hence $\vecX_{t}$ and $\vecX_{t'}$ are both multivariate Gaussian random variables with $\vecX_{t} \sim \NormalDist(\Abar_{t}\Bbar\vecUbar, \Abar_{t}\SigmaWbar{}\Abar_{t}^{\!\tr})$ and $\vecX_{t'} \sim \NormalDist(\Abar_{t'}\Bbar\vecUbar, \Abar_{t'}\SigmaWbar{}\Abar_{t'}^{\!\tr})$, where $\Abar_{t'}$ is the $(t'+1)^\nth$ block-row of $\Abar$ analogous to~\eqref{eq:matrixA_t}. These variables can be represented as affine transformations of a standard Gaussian random vector $\nu \sim \NormalDist( \zero, \eye )$: $\vecX_{t} = \Abar_{t}\Bbar\vecUbar + \Abar_{t}\matrixSqrt{\SigmaWbar{}}\nu$ and $\vecX_{t'} = \Abar_{t'}\Bbar\vecUbar + \Abar_{t'}\matrixSqrt{\SigmaWbar{}}\nu$. Under the regularity conditions~\eqref{eq:Bounded_expectations_conditions} on $\phi_{n}(\cdot)$ (continuity almost everywhere and bounded expectation of partial derivatives) for all $n \in \{0, \ldots, \Ntheta\}$, define $\funcG(\nu) \coloneqq \phi_{n}(\Abar_{t'}\Bbar\vecUbar + \Abar_{t'}\matrixSqrt{\SigmaWbar{}}\nu)$. By Stein's identity~\cite{Stein1981},
\begin{equation}\label{eq:stein_application}
\Expectation{(\vecX_{t} - \Abar_{t}\Bbar\vecUbar)\phi_{n}(\vecX_{t'})} = \Expectation{\Abar_{t}\matrixSqrt{\SigmaWbar{}}\nu\funcG(\nu)} = \Abar_{t}\matrixSqrt{\SigmaWbar{}} \Expectation{\gradient_{\!\nu} \funcG(\nu)}.
\end{equation}
Using the chain rule, $\Expectation{\gradient_{\!\nu} \funcG(\nu)} = \matrixSqrt{\SigmaWbar{}}\Abar_{t'}^{\tr} \Expectation{\gradient_{\!\vecX} \phi_{n}(\vecX_{t'})}$, where $\gradient_{\!\vecX} \phi_{n}(\vecX_{t'})$ is the gradient of $\phi_{n}(\cdot)$ evaluated at $\vecX_{t'}$. Substituting this into~\eqref{eq:stein_application} gives
\begin{equation}\label{eq:cross_cov_component}
\Expectation{(\vecX_{t} - \Abar_{t}\Bbar\vecUbar)\phi_{n}(\vecX_{t'})} = \Abar_{t}\SigmaWbar{}\Abar_{t'}^{\tr} \Expectation{\gradient_{\!\vecX} \phi_{n}(\vecX_{t'})}.
\end{equation}
Finally, inserting~\eqref{eq:cross_cov_component} into~\eqref{eq:State_Basis_Covariance_Matrix}, we obtain
\begin{equation}\label{eq:block_cross_cov}
\Sigma_{\vecXbar\vecYbar}^{tt'} = \sumOp_{n=0}^{\Ntheta} \Abar_{t}\SigmaWbar{}\Abar_{t'}^{\!\tr} \Expectation{\gradient_{\!\vecX} \phi_{n}(\vecX_{t'})} \MuTheta{n} = \Abar_{t}\SigmaWbar{}\Abar_{t'}^{\!\tr} \matrixC_{t'}^{\tr}\MuTheta{},
\end{equation}
where $\matrixC_{t'} = \Expectation{\Jacobian_{\phi}(\vecX_{t'})}$, and $\Jacobian_{\phi}(\vecX_{t'}) = \big[ \gradient_{\!\vecX} \phi_{0}(\vecX_{t'}), \ldots,\, \gradient_{\!\vecX} \phi_{\Ntheta}(\vecX_{t'}) \big]^{\tr}\!$ is the Jacobian matrix evaluated at $\vecX_{t'}$. Assembling these blocks for all pairs $(t, t')$ with $t,t' \in \{0, \ldots, \Trajectory\}$ results in the full form of the cross-covariance matrix~\eqref{eq:cross_cov_expansion},
\begin{equation}\label{eq:full_cross_cov}
\Sigma_{\vecXbar\vecYbar} = \Abar\SigmaWbar{}\Abar^{\!\tr}\! \Cbar^{\tr}\!\mu_{\Theta}, \qquad \Cbar = \diag(\matrixC_{0}, \ldots, \matrixC_{\Trajectory}),
\end{equation}
with $\mu_{\Theta} = \diag(\timedMuTheta{0}, \ldots, \timedMuTheta{\Trajectory})$ and $\timedMuTheta{t} = \MuTheta{}$ for all $t \in \{0,\ldots,\Trajectory\}$. Using $\Sigma_{\vecXbar\vecYbar}$ from~\eqref{eq:full_cross_cov}, along with the previously derived means and covariances, in the generic affine MMSE formulas provides the closed-form expressions for $\PsiX^{\star}(\MuTheta{}, \SigmaTheta{})$ and $\psiX^{\star}(\MuTheta{}, \SigmaTheta{})$ in~\eqref{eq:MMSE_State_LinearEstimator}, both as functions of $(\MuTheta{}, \SigmaTheta{})$. Hence, the affine MMSE state estimator $\hat{\vecXbar}_\mathrm{af}(\vecYbar ; \MuTheta{}, \SigmaTheta{}) = \PsiX^{\star}(\MuTheta{}, \SigmaTheta{}) \vecYbar + \psiX^{\star}(\MuTheta{}, \SigmaTheta{})$ explicitly emphasizes the dependence on the prior parameter statistics $(\MuTheta{}, \SigmaTheta{})$. The corresponding estimation-error covariance $\hatSigmaXbar{\mathrm{af}}\!(\MuTheta{}, \SigmaTheta{})$ and optimal MSE $\Cost^{\star}_{\vecXbar}(\MuTheta{}, \SigmaTheta{})$ follow directly from~\eqref{eq:MMSE_StateEstimation}, completing the proof.
\end{proof}
\bibliographystyle{plain}
\bibliography{references}
\end{document}